%% file: exampleFeb2015.tex
\documentclass[bsl]{asl} 
\newcommand{\hide}[1]{}    
\usepackage{rotate}   
\usepackage{setspace}  
\usepackage[utf8]{inputenc} 
\usepackage{listings}
\usepackage{textcomp}  
\usepackage{multicol}
\usepackage{enumitem}
\usepackage{bussproofs}
\usepackage{color} 
\usepackage{times} 
\usepackage{tikz}   
\usepackage[framemethod=TikZ]{mdframed}
\usepackage{marvosym} 
\usepackage{listings}
\usepackage{caption}   
\usepackage{multirow} 
\usepackage{wrapfig}  
\usepackage[sans]{dsfont}
\usepackage{verbatim}
\usepackage[lofdepth,lotdepth]{subfig} 
\newtheorem{postulate}{Postulate}    
\newtheorem{corollary}{Corollary}
 
\newtheorem{proposition}{Proposition}  
\newtheorem{theorem}{Theorem} 
\newtheorem{lemma}{Lemma}  
\newtheorem{definition}{Definition}

\newcommand{\formulalength}{\ensuremath{\textsf{f}\_\textsf{len}}}
  
\newcommand{\extract}{\textsf{Extract}} 
\newcommand{\translate}{\textsf{Translate}}
\newcommand{\andMeta}{\ensuremath{\wedge^{\dagger}}}
\newcommand{\orMeta}{\ensuremath{\vee^{\dagger}}} 
\newcommand{\impliesMeta}{\ensuremath{\rightarrow^{\dagger}}}
\newcommand{\intFrame}{\ensuremath{\mathfrak{M}}}  
\newcommand{\tintFrame}{\ensuremath{\mathfrak{M}}} 

\newcommand{\lformulasize}{\ensuremath{\textsf{f}\_\textsf{size}}}
\newcommand{\negMax}{\ensuremath{\textsf{neg}\_\textsf{max}}}
\newcommand{\recurseReduce}{\textsf{recursiveReduce}}
\newcommand{\recurseReducet}{\textsf{recursiveReduce2}}
\newcommand{\abbR}{\textsf{R}}
\newcommand{\squash}{\textsf{EXTRACT}}

\newcommand{\exptime}{\textsf{EXPTIME}}
\newcommand{\sat}{\textsf{UNSAT}}  
\newcommand{\rewrite}{\textsf{REWRITE}}  
\newcommand{\unorders}[1]{\ensuremath{\langle#1\rangle}} 
\newcommand{\seqs}[1]{\ensuremath{\lfloor#1\rfloor}}

\newcommand{\struct}{\ensuremath{\textsf{FStruct}}}

\title{Gradual Classical Logic for Attributed Objects - Extended in Re-Presentation}
\author{Ryuta Arisaka\\
ryutaarisaka@gmail.com}

\begin{document}  
    \hide{ 
    Dear Journal of Philosophical Logic Editors, \\\\
    
    I have written a document which focuses on 
    a foundational property of the predicate logic, 
    namely the assumption of atomicity of entities. 
    The logic that this manuscript presents expresses 
    predication without the division of predicates from 
    terms. The logic can capture, 
    very intuitively, some properties of natural expressions 
    that the post-Fregean sentential logic tends to shun away, 
    such as contrarieties and pre-suppositions. 
    It, I believe, is of interest to areas of logic, philosophy,
    linguistics, artificial intelligence and computer science, 
    and I wish it to be considered for publication. Thank you 
    very much for handling this document. \\\\

    \noindent Yours sincerely,\\
    Author, 3.Mar.2015 
}
    \begin{abstract} 
      Our understanding about things is conceptual.  
    By stating that we reason about objects, 
    it is in fact not the objects but concepts 
   referring to them that we manipulate.     
    Now, so long just as we acknowledge 
    infinitely extending notions such as 
    space, time, size, colour, etc, - 
    in short, any reasonable quality - into  
    which an object is subjected, 
    it becomes 
    infeasible to affirm
    atomicity in the concept referring to the object. 
    However, 
    formal/symbolic 
    logics typically presume atomic entities 
    upon which other expressions are built.    
    Can we reflect our intuition about 
    the concept onto formal/symbolic logics at all? 
    I assure that we can, but the usual 
    perspective about the atomicity needs inspected. 
    In this work, I present {\it gradual
        logic} which materialises the observation 
    that we cannot 
    tell apart whether a so-regarded 
    atomic entity is atomic or is just atomic 
    enough not to be considered non-atomic. 
    The motivation 
    is to capture 
    certain phenomena that naturally occur 
    around concepts with 
    attributes, including presupposition and 
    contraries. I present logical particulars 
    of the logic, which is then 
    mapped onto formal semantics. Two 
    linguistically interesting semantics 
    will be considered. 
    Decidability is shown. 
    \end{abstract} 
\maketitle 
\section{Introduction}     
I present a 
logic  
that expresses gradual shifts in domain of discourse.   
The motivation is to capture certain peculiar phenomena 
about concepts/objects and other concepts/objects\footnote{
    These two terms will not be strongly 
    distinguished  in this work. An object may exist by itself, but 
       to reason about relation between objects or 
       just to speak about them, 
       it is, as we presume, 
       concepts referring to the objects that 
       we reason/speak about.} 
 as their attributes. The first such phenomenon is that 
extension of a concept alters when 
it becomes an attribute to other concepts. 
Also a concept that is specified another 
concept as its attribute becomes an intension \cite{Carnap47,Montague74,Church51} of 
the concept which itself is an extensional concept. 
Consider 
for instance `brooch'ed `hat' and `brooch'. 
The `brooch' in the former is an attribute to `hat'.   
By definition, if 
anything is an attribute 
to something at all, it must be found among all that can 
be an attribute to it. 
Whereas the extension of 
`brooch' in the latter is only 
delimited by our understanding about `brooch', 
that in the former as an attribute of 
`hat' is delimited also by our understanding 
about the concept `hat' (needless to say, only if 
the understanding 
of ours should permit `brooch' as an attribute at all). 
But this is not all. The attribute in turn 
specialises the `hat' to which it is an attribute: 
the `brooch'ed `hat' forms an intension of the `hat', 
and itself becomes an extensional concept 
``brooch'ed `hat''. \\
\indent The shift in extension is not typically 
observed in formal logics, be they 
temporal, epistemic, modal {\it etc.}  
Some exceptions that challenge the norm are 
a kind of spatio-temporal 
logics \cite{Gabelaia05,Muller98} and 
some kinds of context logics (\emph{Cf.} 
\cite{McCarthy93,Guha03}) in the line of \cite{Buvac93,Nayak94}. In \cite{Gabelaia05} for instance, 
Gabelaia {\it et al.} consider the definition 
of EU at a point of time and at another point of 
time. If some countries are merged into the current  
EU, then the term EU will remain EU at the 
future time reference as it is now, but the 
spaces that the two EU occupy are not the same. 
Similar phenomena are occurring 
in the relation between concepts/objects and their 
attributes. However, unlike the case of 
the spatio-temporal logics, 
there is no external and global space for them: 
there are only those spaces generated by the 
(extension of) concepts 
themselves. The stated (re-)action of 
intension/extension within an attributed concept/object 
is another intriguing feature that has not been 
formalised before.  \\
\indent Another point about the concept is that 
a concept in itself, which is to say, 
an atomic concept which 
does not itself possess 
any other concepts as its attributes, 
is almost certainly imperceptible,\footnote{Let 
us arbitrarily suppose the concept hat, and let 
us conduct an experiment as to if we could 
perceive the hat in itself as something 
differing from nothingness for which any scenario where it comes with 
an attribute is inconceivable. 
To begin with, if the 
word hat evoked in our mind any specific hat with 
specific colour and shape, we first remove the colour 
out of it. If the process should make it transparent, 
we then remove the transparentness away from it. 
And if there should be still some things that 
are by some means perceivable as having originated 
from it, then because they are an attribute of the hat,
we again remove any one of them. If {\it the humanly 
    no longer detectable 
    something is not nothingness} is not itself 
contradictory, then there must be still some quality
originating in the hat that makes the something 
differ from nothingness. Now the question is whether 
the something can be perceived at all to be different 
from nothingness. Intuition speaks otherwise.} and hence almost certainly cannot be 
reasoned about. Typically, however, formal/symbolic logics assume smallest entities 
upon which other expressions are formed.   
In this work I challenge the assumption, and 
materialise the observation that we cannot tell apart 
- that is, 
{\it we cannot know} - whether 
a so-regarded atomic entity is atomic or is just 
atomic enough not to be considered non-atomic. 
I present a logic in which every entity is non-atomic, 
reflecting our intuition about the concept.\footnote{The utility
    of logical 
    non-atomicity
    is noted in a recent work 
    \cite{Jung15} in programming community. 
    The idea of logical non-atomicity 
    in formal/symbolic logic itself, however, 
    appears previously 
    in 
    the immediately preceding 
    work 
    to the current paper, namely, 
    in \cite{Arisaka14tech1}; as well as, and more bluntly, in its variation as 
    a technical report.} 
    \\
\indent Strikingly we can represent both extensional
shifts and non-atomicity using the familiar classical 
logic only (but 
the results extend to other 
Boolean logics); with 
many domains of discourse.\footnote{The use of 
multiple domains of discourse is also notable 
in contextual logics. 
Connections to those will be mentioned 
at the end of this work. }
The idea is as follows. We shall   
define a binary connective {\small $\gtrdot$} over classical logic 
instances. As an example, {\small $\textsf{Hat} \gtrdot 
    \textsf{Brooch}$} reads as;  Hat is, 
and under the presupposition 
that Hat is, Brooch is as its attribute.
In this simple expression there are  
two domains of discourse: one in which 
{\small $\textsf{Hat}$} in {\small $\textsf{Hat} 
    \gtrdot \textsf{Brooch}$} is being discussed; 
and one in which {\small $\textsf{Brooch}$} in  
{\small $\textsf{Hat} \gtrdot \textsf{Brooch}$} 
is being 
discussed. 
The second domain of discourse 
as a whole is delimited by the (extension of) 
{\small $\textsf{Hat}$} 
that gives rise to it. {\small $\textsf{Hat} \gtrdot 
    \textsf{Brooch}$} is an intension of 
{\small $\textsf{Hat}$}, and itself forms 
an extensional concept. The non-atomicity of concepts 
is captured without breaking the properties 
of classical logic. The ideas are that 
every concept has attributes, 
but that the attributes are not discussed 
in the same domain of discourse as the concept is 
discussed in. From within the domain 
of discourse discussing {\small $\textsf{Hat}$} 
in {\small $\textsf{Hat} \gtrdot \textsf{Brooch}$},  
it cannot be perceived whether it has or has not 
attributes, \emph{i.e.} whether it is atomic or is 
just atomic enough not to be considered non-atomic.   \\
\indent We can also explain some reasonably common 
every-day linguistic 
phenomenon with this connective. Let us turn to 
an example.\\\\
{\it Episode}.
There is 
a tiny hat shop in our town, having the following in stock: 
\begin{enumerate}[leftmargin=0.5cm]
  \item 3 types of hats: orange hats, green hats 
    ornamented with 
    some brooch, and blue hats decorated with 
    some white hat-shaped accessory. Only the green 
    and the blue hats are displayed in the shop. 
  \item 2 types of shirts: yellow and blue shirts. 
    Only the blue shirts are displayed in the shop. 
\end{enumerate}
A young man has come to the hat shop.  
After a while he asks the shop owner, a lady of many a year of experience 
in hat-making; 
``Have you got a yellow hat?'' 
Knowing that it is not in her shop, she answers; 
``No, I do not have it in stock,'' negating 
the possibility that there is one in stock at her shop at the present
point of time.  {\it Period}. \\\\
\indent  
``What is she actually denying about?'' 
is our question, however. 
It is plausible that, in delivering the answer, 
the question posed may have allowed her to infer that the young man was 
looking for 
a hat, a yellow hat in particular. Then the answer
may be followed by she saying; ``\dots but I do have hats with 
different colours including ones not currently displayed.'' 
That is, while she denies the presence of a yellow hat, 
she still presumes the availability of hats of which she 
reckons he would like to learn. 
It does not appear so unrealistic to suppose such a thought 
of hers that he may be ready to 
compromise his preference for a yellow hat with 
some non-yellow one, 
possibly an orange one in stock, given its comparative closeness 
in hue to yellow. \\
\indent Now, what if the young man turned out to be a 
town-famous collector of yellow articles? Then it may be 
that from his question she had divined instead that 
he was looking for something yellow, a yellow hat in 
particular, in which case her answer could have 
been a contracted form of ``No, I do not have it 
in stock, but I do have a yellow shirt nonetheless (as 
you are looking after, I suppose?)"\\
\indent Either way, these somewhat-appearing-to-be partial 
negations contrast with 
the classical negation with which her 
answer can be interpreted only as that she does not have a yellow 
hat, nothing less, nothing more, with no restriction in the range 
of possibilities outside it. \\  
\indent The explanation that I wish to provide  is that in the 
first case she actually means {\small $\textsf{Hat} 
    \gtrdot \neg \textsf{Yellow}$}, presuming 
the main concept {\small $\textsf{Hat}$} 
but negating {\small $\textsf{Yellow}$} as its attribute in a different 
domain of discourse in which its attributes 
can be discussed; and in the second case 
she actually means 
{\small $\textsf{Yellow} \gtrdot \neg \textsf{Hat}$} with the main 
concept {\small $\textsf{Yellow}$} presumed but {\small $
    \textsf{Hat}$} denied as its attribute. 
Like this manner, gradual classical logic 
that I propose here 
can capture partial negation, which is known as 
contrariety in the pre-Fregean term logic 
from the Aristotle's era \cite{Horn89}, 
as well as nowadays more 
orthodox contradictory negation. 
Here we illustrated attribute 
negation. Complementary, we may also consider 
object negation of the kind 
{\small $\neg \textsf{Hat} \gtrdot \textsf{Yellow}$}, 
as well as more orthodox negation of the sort 
{\small $\neg (\textsf{Hat} \gtrdot \textsf{Yellow})$} 
which I call attributed-object negation. \\
\indent My purpose 
is to assume attributed concepts/objects\footnote{In the rest 
I simply write attributed objects, assuming 
that it is clear that we do not strongly distinguish 
the two terms in the context of this work.} as primitive 
entities and to analyse the logical behaviour of 
{\small $\gtrdot$} in interaction with 
the other familiar connectives in classical logic. In 
this logic, 
the sense of `truth', a very fundamental property of classical 
logic, gradually shifts\footnote{This should not 
    be confused with the idea of many 
    truths in a single domain of discourse \cite{Hajek10,Gottwald09}. } 
by domains of discourse moving deeper into 
attributes of (attributed) objects. 
As for 
inconsistency, if there is an inconsistent argument within  
a discourse on attributed objects, wherever it may be that 
it is occurring, 
the reasoning part of which is inconsistent cannot be 
said to be consistent. For this reason it remains in gradual classical logic 
just as strong as is in standard classical logic.   
\subsection{Structure of this work}    
Shown below is the organisation of this work. The basic 
conceptual 
core is formed in Section 1, Section 2, which 
is put into formal semantics in Section 3. 
Decidability of the logic is proved in Section 4. 
After the foundation is laid down, 
more advanced observations will be made 
about the object-attribute relation. They will be 
found in Section 5. 
Section 6 concludes with 
prospects.  
\begin{itemize}[leftmargin=0.3cm]
  \item Development of gradual classical logic 
      (Sections 1 and Section 2). 
  \item A formal semantics 
    of gradual classical logic and a proof that it is 
    not para-consistent/inconsistent 
    (Section 3). 
  \item Decidability of gradual classical logic 
      (Section 4).  
  \item Advanced materials: the notion 
      of recognition cut-offs, and an alternative presentation 
      of 
      gradual classical logic (Section 5). 
  \item Conclusion (Section 6). 
\end{itemize} 
\section{Gradual Classical Logic: Logical Particulars} 
In this section we shall look into logical particulars of 
gradual classical logic. Some familiarity 
with 
propositional classical logic, in particular with 
how the logical connectives behave, is presumed. 
Mathematical transcriptions of 
gradual classical logic are found in the next section.  
\subsection{Logical connective for object/attribute and 
    interactions with negation} 
We shall dedicate the symbol {\small $\gtrdot$} to represent the object-attribute relation.
The usage 
of the new connective is fixed to take the form 
{\small $\textsf{Object}_1 \gtrdot \textsf{Object}_2$}. 
It denotes an attributed object. 
{\small $\textsf{Object}_1$} is more generic 
an object than {\small $\textsf{Object}_1 \gtrdot \textsf{Object}_2$} 
({\small $\textsf{Object}_2$} acting as an attribute to 
{\small $\textsf{Object}_1$}  
makes {\small $\textsf{Object}_1$} more specific).  
The schematic reading is as follows:  
``It is true that {\small $\textsf{Object}_1$} is, 
and it is true that it has {\small $\textsf{Object}_2$} 
as its attribute." Now, this really is a short-form 
of the following expression: 
``It is true by some sense of truth X reigning over 
the domain of discourse discussing {\small $\textsf{Object}_1$}
that {\small $\textsf{Object}_1$} is judged existing in 
the domain of discourse,\footnote{As must be 
    the case, 
    a domain of discourse defines what can be talked about, 
    which itself does not dictate that all the elements 
    that are found in the domain are judged existing.} 
and it is true by some sense of truth Y reigning over 
the domain of discourse discussing {\small $\textsf{Object}_2$}
as an attribute to {\small $\textsf{Object}_1$} that 
{\small $\textsf{Object}_1$} is judged having 
{\small $\textsf{Object}_2$} as its attribute."
Also, this reading is what is meant when we say 
that ``It is true that $\textsf{Object}_1 \gtrdot 
\textsf{Object}_2$ is," where the sense of 
the truth Z
judging this statement has relation to X and Y, 
in order for compatibility. I take 
these side-remarks for granted in the rest without explicit 
stipulation. 
Given 
an attributed object {\small $\textsf{Object}_1 \gtrdot 
\textsf{Ojbect}_2$},  
{\small $\neg (\textsf{Object}_1 \gtrdot \textsf{Object}_2)$} 
expresses its attributed object negation, 
{\small $\neg \textsf{Object}_1 \gtrdot \textsf{Object}_2$} 
its object negation and 
{\small $\textsf{Object}_1 \gtrdot \neg \textsf{Object}_2$} 
its attribute negation. Again the schematic readings for them 
are, respectively;   
\begin{itemize}[leftmargin=0.3cm] 
  \item
It is false 
that 
{\small $\textsf{Object}_1 \gtrdot \textsf{Object}_2$} is. 
\item 
It is false that {\small $\textsf{Object}_1$} is, but 
it is true that some non-{\small $\textsf{Object}_1$} is  
which has an attribute of {\small $\textsf{Object}_2$}. 
\item 
It is true that {\small $\textsf{Object}_1$} is, 
but it is false that it has an attribute 
of {\small $\textsf{Object}_2$}.
\end{itemize}
The presence 
of negation flips ``It is true that \ldots'' into 
``It is false that \ldots'' and vice versa. 
But it should be also 
noted how negation acts in attribute negations and 
object/attribute negations. Several specific examples 
constructed parodically from the items in the hat shop episode are;  
\begin{enumerate}[leftmargin=0.5cm]
  \item {\small $\text{Hat} \gtrdot \text{Yellow}$}:        
    It is true that hat is, and it is true that it has 
    the attribute of being yellow (that is, it is yellow). 
\item {\small $\text{Yellow} \gtrdot \text{Hat}$}:     
  It is true that yellow is, and it is true that it 
  has hat as its attribute. 
\item {\small $\text{Hat} \gtrdot \neg \text{Yellow}$}:       
    It is true that hat is, but it is false that it is yellow.
\item {\small $\neg \text{Hat} \gtrdot \text{Yellow}$}:  
    It is false that hat is, but it is true that yellow object (which 
    is not hat) is.  
      \item {\small $\neg (\text{Hat} \gtrdot \text{Yellow})$}:          
  Either it is false that hat is, or if it is true that hat is,
  then it is false that it is 
  yellow.    
\end{enumerate} 
\subsection{Object/attribute relation and 
    conjunction} 
We examine specific examples first involving 
{\small $\gtrdot$} and {\small $\wedge$} (conjunction), 
and then observe what the readings imply.  
\begin{enumerate}[leftmargin=0.5cm]
  \item {\small $\text{Hat} \gtrdot (\text{Green} \wedge 
          \text{Brooch})$}: It is true that 
    hat is, and it is true that it is green and brooched. 
  \item {\small $(\text{Hat} \gtrdot \text{Green}) 
    \wedge (\text{Hat} \gtrdot \text{Brooch})$}: 
    for one, it is true that 
    hat is, and it is true that it is 
    green; for one, it is true that 
    hat is, and it is true that it is 
    brooched.  
  \item {\small $(\text{Hat} \wedge \text{Shirt}) 
    \gtrdot \text{Yellow}$}: 
    It is true that 
    hat and shirt are, and it is true that 
    they are yellow.  
  \item {\small $(\text{Hat} \gtrdot \text{Yellow}) 
    \wedge (\text{Shirt} \gtrdot \text{Yellow})$}:  
    for one, it is true that 
    hat is, and it is true that 
    it is yellow; for one, it is true that 
    shirt is, and it is true that 
    it is yellow.  
\end{enumerate} 
By now it has hopefully become clear that by 
{\it existential facts as truths} 
I do not mean how many of a given 
(attributed) object exist: in 
gradual classical logic, cardinality of 
objects (Cf. Linear Logic \cite{DBLP:journals/tcs/Girard87})
is not what it must be responsible for, but 
only the facts themselves of whether any of them 
exist in a given domain of discourse, which 
is in line with classical logic.\footnote{
That proposition A is true and that proposition A is true 
mean that proposition A is true; the subject of this sentence 
is equivalent to the object of its.}
Hence they univocally assume a singular 
rather than a plural form, as in the examples inscribed so far. 
The first and the second, and the third and the fourth, 
then equate.\footnote{I will also touch upon an 
    alternative
    interpretation in Section 5, as an advanced material: 
    just as there are many 
modal logics with a varying degree of strength of modalities, 
so does it seem 
that more than one interpretations about {\small $\gtrdot$} in 
interaction with the other logical connectives can be studied.} Nevertheless, 
it is still important that we analyse them with a sufficient 
precision. In the third and the fourth where 
the same attribute is shared among several objects, the attribute 
of being yellow ascribes to all of them. 
Therefore those expressions are a true statement only if  
(1) there is an existential fact that both hat and shirt are  
and (2) being yellow is true for the existential fact (formed 
by existence of hat and that of shirt).  
Another example is in Figure \ref{first_figure}.    
\input{first_figure.tex}   
\subsection{Object/attribute relation and disjunction} 
We look at examples first.   
\begin{enumerate}[leftmargin=0.5cm]
  \item {\small $\text{Hat} \gtrdot (\text{Hat} \vee 
    \text{Brooch})$}: 
    It is true that hat is, and it is true that 
    it is either hatted or brooched.  
  \item {\small $(\text{Hat} \gtrdot \text{Hat}) 
    \vee (\text{Hat} \gtrdot \text{Brooch})$}: 
    At least either that it is true that hat is and 
    it is true that it is hatted, or that
    it is true that hat is and it is true that 
    it is brooched.  
  \item {\small $(\text{Hat} \vee \text{Shirt}) \gtrdot 
    \text{Yellow}$}: It is true that 
    at least either hat or shirt is, and 
    it is true that whichever is existing (or both) is (or are) yellow. 
  \item {\small $(\text{Hat} \gtrdot \text{Yellow}) \vee 
    (\text{Shirt} \gtrdot \text{Yellow})$}:  
    At least either it is true that hat is and it is 
    true that it is yellow, or 
    it is true that shirt is and it is true that 
    it is yellow.  
\end{enumerate}
Just as in the previous sub-section, here again 1) and 2), 
and 3) and 4) are equivalent. 
However, in the cases of 3) and 4), we 
have that the existential fact of the attribute yellow 
depends on that of hat or shirt, whichever is existing, or 
that of both if they both exist.\footnote{In classical logic, that proposition 
A or proposition B is true means 
that at least one of the proposition A or the proposition B is true though 
both can be true. 
Same goes here.}
\subsection{Nestings of object/attribute relations}
An expression of the kind\linebreak {\small $(\textsf{Object}_1 \gtrdot 
\textsf{Object}_2) \gtrdot \textsf{Object}_3$} 
is ambiguous. But we begin by listing examples 
and then move onto analysis of the readings of the nesting of 
the relations. 
\begin{enumerate}[leftmargin=0.5cm]
\item  {\small $(\text{Hat} \gtrdot \text{Brooch}) \gtrdot 
    \text{Green}$}: It is true that hat is, and it is true that 
    it is brooched. It is true that the object thus 
    described is green.  
  \item {\small $\text{Hat} \gtrdot (\text{Hat} \gtrdot \text{White})$}:
    It is true that hat is, and it is true that 
    it has the attribute of which 
    it is true that hat is and that it is 
    white. (More simply, it is true that hat is, 
    and it is true that it is white-hatted.)
  \item {\small $\neg (\text{Hat} \gtrdot \text{Yellow}) 
    \gtrdot \text{Brooch}$}:   
    Either it is false that hat is, or else it is true that 
    hat is but it is false that it is yellow.\footnote{
    This is the reading of 
    {\small $\neg (\text{Hat} \gtrdot \text{Yellow})$}.}
    If it is false that hat is, then it is true that 
    brooched object (which obviously cannot be hat) is. 
    If it is true that hat is but it is false that it is yellow,  
    then it is true that the object thus described is brooched. 
\end{enumerate}     
Note that to say that Hat {\small $\gtrdot$} Brooch (brooched hat) 
is being green, we must mean to say that the object to 
the attribute of being green, \emph{i.e.} hat, is green. It 
is on the other hand 
unclear if green brooched hat should or should not 
mean that the brooch, an accessory to hat, is also green. 
But common sense about adjectives dictates that 
such be simply indeterminate. It is reasonable
for (Hat {\small $\gtrdot$} Brooch) {\small $\gtrdot$} 
Green, while if we have 
(Hat {\small $\gtrdot$} Large) {\small $\gtrdot$} 
Green, ordinarily speaking it cannot be the case that the attribute of being 
large is green. Therefore 
we enforce that {\small $(\textsf{Object}_1 \gtrdot \textsf{Object}_2) 
\gtrdot \textsf{Object}_3$} amounts to 
{\small $(\textsf{Object}_1 \gtrdot \textsf{Object}_3) 
\wedge ((\textsf{Object}_1 \gtrdot \textsf{Object}_2)
\vee 
(\textsf{Object}_1 \gtrdot (\textsf{Object}_2 \gtrdot \textsf{Object}_3)))$} in which disjunction as usual captures the 
indeterminacy. No ambiguity is posed in 2), 
and 3) is understood in the same 
way as 1). 
\subsection{Two nullary logical connectives} 
Now we examine the nullary logical connectives  
{\small $\top$} and {\small $\bot$} which 
denote, in classical logic, the concept of the truth and that of 
the inconsistency. In gradual classical logic 
{\small $\top$} denotes the concept of the presence 
and {\small $\bot$} denotes that of the absence. 
Several examples for
the readings are;
\begin{enumerate}[leftmargin=0.5cm]
  \item {\small $\top \gtrdot \text{Yellow}$}:   
    It is true that yellow object is. 
  \item {\small $\text{Hat} \gtrdot (\top \gtrdot \text{Yellow})$}:   
    It is true that hat is, and it is true that it has the 
    following attribute of which it is true that 
    it is yellow object. 
  \item {\small $\bot \gtrdot \text{Yellow}$}:   
    It is true that nothingness is, and it is true that 
    it is yellow. 
  \item {\small $\text{Hat} \gtrdot \top$}:  
    It is true that hat is. 
  \item {\small $\text{Hat} \gtrdot \bot$}:    
    It is true that hat is, and it is true that it has no attributes. 
  \item {\small $\bot \gtrdot \bot$}: It is  
    true that nothingness is, and it is true that it has no attributes. 
\end{enumerate}      
It is illustrated in 1) and 2) how the sense of the `truth'
is delimited by the object to which it acts as 
an attribute. For the rest, however, 
there is a point  
which is not so vacuous as not to merit 
a consideration, and to which I in fact append the following postulate. 
\begin{postulate} 
  That which does not have any attribute cannot be  
  distinguished from nothingness for which
  any scenario where it comes with an attribute is 
  inconceivable. 
  Conversely, anything that remains once all the attributes 
  have been removed from a given object is  
  nothingness. 
  \label{axiomZero}
\end{postulate}   
With it, the item 3) which asserts the existence of 
nothingness is contradictory. The item 4) then behaves as expected in 
that Hat which is asserted with the presence of attribute(s)  
is just as generic a term as Hat itself is. 
The item 5) which asserts the existence of an object with 
no attributes again contradicts Postulate \ref{axiomZero}.  
The item 6) illustrates that any attributed object in some part of which 
has turned out to be contradictory remains contradictory no matter 
how it is to be extended: a {\small $\bot$}  
cannot negate another {\small $\bot$}. Cf. 
the footnote 2 for the plausibility of the postulate. 
\section{Mathematical mappings: syntax and semantics}   
In this section a semantics 
of gradual classical logic is formalised. 
We assume in the rest of this document;
\begin{itemize}[leftmargin=0.3cm]
  \item {\small $\mathbb{N}$} denotes the set of natural numbers 
    including 0.   
  \item {\small $\wedge^{\dagger}$} and {\small $\vee^{\dagger}$} 
    are two binary 
    operators on Boolean arithmetic. 
    The following laws hold; 
    {\small $1 \vee^{\dagger} 1 = 1 \vee^{\dagger} 0 = 0 
    \vee^{\dagger} 1 = 1$}, 
    {\small $0 \wedge^{\dagger} 0 
    = 0 \wedge^{\dagger} 1 = 1 \wedge^{\dagger} 0 = 0$}, 
    {\small $1 \wedge^{\dagger} 1 = 1$}, 
    and {\small $0 \vee^{\dagger} 0 = 0$}. 
  \item {\small $\wedge^{\dagger}$}, {\small $\vee^{\dagger}$} 
      {\small $\rightarrow^{\dagger}$}, {\small $\neg^{\dagger}$}, {\small $\exists$} 
    and {\small $\forall$}
    are meta-logical connectives: conjunction, disjunction,\footnote{
    These two symbols are overloaded. Save whether 
    truth values or  the ternary values are supplied as arguments, 
    however, the distinction is clear from the context 
    in which they are used. 
    }
    material implication, negation, existential quantification 
    and universal quantification, whose semantics  
    follow those of standard classical logic.    
    We abbreviate {\small $(A \rightarrow^{\dagger} B) 
    \wedge^{\dagger} (B \rightarrow^{\dagger} A)$} 
    by {\small $A \leftrightarrow^{\dagger} B$}.  
  \item Binding strength of logical or meta-logical connectives is  
        {\small $[\neg]\! \gg \!
    [\wedge \ \ \vee]\! \gg \!
    [\gtrdot] \gg [\forall \ \ \exists]\! \gg \!
    [\neg^{\dagger}]\! \gg\! [\wedge^{\dagger} \ \  \vee^{\dagger}]
    \! \gg\! [\rightarrow^{\dagger}] 
    \! \gg\! [\leftrightarrow^{\dagger}]$} 
in the order of 
    decreasing precedence. Those that belong to the same 
    group are assumed having the same precedence. 
  \item For any binary connectives {\small $?$}, 
    for any {\small $i, j \in \mathbb{N}$} and 
    for {\small $!_0, !_1, \cdots, !_j$} that are some recognisable 
    entities, 
    {\small $?_{i = 0}^j !_i$} is an abbreviation of\linebreak
    {\small $(!_0) ? (!_1) ? \cdots ? (!_j)$}.   
  \item 
    For the unary connective {\small $\neg$},   
    {\small $\neg \neg !$} for some recognisable entity 
    {\small $!$} is an abbreviation of 
    {\small $\neg (\neg !)$}. Further, 
    {\small $\neg^k !$} for some {\small $k \in \mathbb{N}$}  
    and some recognisable entity {\small $!$} 
    is an abbreviation of 
    {\small $\underbrace{\neg \cdots \neg}_k
    !$}.
  \item For the binary connective {\small $\gtrdot$}, 
    {\small $!_0 \gtrdot !_1 \gtrdot !_2$} 
    for some three recognisable entities 
    is an abbreviation of 
    {\small $!_0 \gtrdot (!_1 \gtrdot !_2)$}.  
\end{itemize}
On this preamble we shall begin. 
\subsection{Development of semantics}
The set of literals in gradual classical logic is denoted by {\small $\mathcal{A}$} 
whose elements are referred to by {\small $a$} with or without 
a sub-script. This set has a countably many 
number of literals. 
Given a literal {\small $a \in \mathcal{A}$}, 
its complement is denoted by {\small $a^c$} which is 
in {\small $\mathcal{A}$}. As usual, we have {\small $\forall 
a \in \mathcal{A}.(a^c)^c = a$}. 
The set {\small  $\mathcal{A} \cup \{\top\} \cup \{\bot\}$} where 
{\small $\top$} and {\small $\bot$} are the two nullary logical 
connectives is denoted by {\small $\mathcal{S}$}. Its elements 
are referred to by {\small $s$} with or without 
a sub-script. Given {\small $s \in \mathcal{S}$}, its 
complement is denoted by {\small $s^c$} which 
is in {\small $\mathcal{S}$}. Here we have    
{\small $\top^c = \bot$} and 
{\small $\bot^c = \top$}.  
The set of formulas is denoted by {\small $\mathfrak{F}$} 
whose elements, {\small $F$} with or without 
a sub-/super-script,
are finitely constructed from the following grammar; \\
\indent {\small $F := s \ | \ F \wedge F \ | \ 
F \vee F \ | \ \neg F \ | \  F \gtrdot F$}\\
We now develop semantics. This is done in two parts: we do not outright 
jump to the definition 
of valuation (which we could, but which we simply do not 
choose for succinctness of the
proofs of the main results). 
Instead, just as we only need 
consider negation normal form in classical logic because every 
classical logic formula definable 
has a reduction into a normal form, so 
shall we first define rules for formula reductions 
(for any {\small $F_1, F_2, F_3 \in 
\mathfrak{F}$}):  
\begin{itemize}[leftmargin=0.3cm]
  \item {\small $\forall s \in \mathcal{S}.\neg s \mapsto s^c$} 
    ({\small $\neg$} reduction 1). 
  \item {\small $\neg (F_1 \wedge F_2) \mapsto
    \neg F_1 \vee \neg F_2$} ({\small $\neg$} reduction 2). 
  \item {\small $\neg (F_1 \vee F_2) \mapsto
    \neg F_1 \wedge \neg F_2$} ({\small $\neg$} reduction 3). 
  \item {\small $\neg (s \gtrdot F_2) \mapsto
    s^c \vee (s \gtrdot \neg F_2)$} ({\small $\neg$} 
    reduction 4).    
  \item {\small $(F_1 \gtrdot F_2) \gtrdot F_3
     \mapsto (F_1 \gtrdot F_3) \wedge  
     ((F_1 \gtrdot F_2) \vee (F_1 \gtrdot F_2 \gtrdot F_3))$} 
     ({\small $\gtrdot$} reduction 1).
   \item {\small $
     (F_1 \wedge F_2) \gtrdot F_3 \mapsto
     (F_1 \gtrdot F_3) \wedge (F_2 \gtrdot F_3)$} 
     ({\small $\gtrdot$} reduction 2). 
   \item {\small $
     (F_1 \vee F_2) \gtrdot F_3 \mapsto 
     (F_1 \gtrdot F_3) \vee (F_2 \gtrdot F_3)$} 
     ({\small $\gtrdot$} reduction 3).  
   \item {\small $F_1 \gtrdot (F_2 \wedge F_3) 
     \mapsto (F_1 \gtrdot F_2) \wedge (F_1 \gtrdot F_3)$} 
     ({\small $\gtrdot$} reduction 4). 
   \item {\small $F_1 \gtrdot (F_2 \vee F_3) 
     \mapsto (F_1 \gtrdot F_2) \vee (F_1 \gtrdot F_3)$} 
     ({\small $\gtrdot$} reduction 5). 
\end{itemize}    
\begin{definition}[Domain function/valuation frame]        
  Let {\small $\mathcal{S}^*$} denote the set\linebreak union of (A) the set 
  of finite sequences of elements 
  of {\small $\mathcal{S}$}  and 
  (B) a singleton set {\small $\{\epsilon\}$} denoting 
  an empty sequence.  
  We define a domain function
  {\small $D: \mathcal{S}^* \rightarrow  2^{\mathcal{S}}$}. 
    We define a valuation frame as a 2-tuple: 
  {\small $(\mathsf{I}, \mathsf{J})$}, where 
  {\small $\mathsf{I}: \mathcal{S}^* \times \mathcal{S} \rightarrow 
      \{0,1\}$} is what we call  
  local interpretation 
  and 
  {\small $\mathsf{J}: \mathcal{S}^* \backslash 
      \{\epsilon\} \rightarrow \{0,1\}  
       $} is what we call gloal interpretation. 
       The following are defined to satisfy 
       for all {\small $k \in \mathbb{N}$} 
       and for all {\small $s_0, \ldots, s_{k} \in 
           \mathcal{S}$}.    
       \begin{description}  
           \item[Regarding domains of discourse]{\ } 
               \begin{itemize}[leftmargin=-0.3cm] 
                   \item For 
                       all {\small $s^* \in \mathcal{S}^*$},
                       {\small $D(s^*)$} is closed 
                       under complementation,  
                       and has at least 
                       {\small $\top$} and 
                       {\small $\bot$}. 
                \end{itemize}
           \item[Regarding local interpretations]{\ }
      \begin{itemize}[leftmargin=-0.3cm] 
          \item {\small $[\mathsf{I}(s_0.s_1.\dots.s_{k-1}, \top) = 1]$}\footnote{Simply 
  for a presentation purpose, 
  we use a dot such as {\small $s_1^*.s_2^*$} for 
  {\small $s_1^*, s_2^* \in \mathcal{S}^*$} to show that 
  {\small $s_1^*.s_2^*$} is an element of 
  {\small $\mathcal{S}^*$} in which 
  {\small $s_1^*$} is the preceding constituent and {\small $s_2^*$} 
  the following constituent of 
  {\small $s_1^*.s_2^*$}. 
        When {\small $k = 0$}, 
        we assume that 
        {\small $s_0.s_1.\dots.s_{k-1} 
        = \epsilon$}. Same applies 
    in the rest.} 
          ({\small $\mathsf{I}$} valuation of $\top$).
      \item {\small $[\mathsf{I}(s_0.s_1.\dots.s_{k-1}, \bot) = 0]$} 
          (That  of $\bot$).  
      \item {\small $\forall a_k \in  
              D(s_0.s_1.\dots.s_{k-1}). 
              [\mathsf{I}(s_0.s_1.\dots.s_{k-1}, a_{k})
              = 0] \vee^{\dagger}$}\\
          {\small $
              [\mathsf{I}(s_0.s_1.\dots.s_{k-1}, a_{k})
= 1]$} 
(That of a literal).  
\item {\small $\forall a_k \in  
              D(s_0.s_1.\dots.s_{k-1}). 
[\mathsf{I}(s_0.s_1.\dots.s_{k-1},a_{k}) = 0]
        \leftrightarrow^{\dagger}$}\\
        {\small $[\mathsf{I}(s_0.s_1.\dots.s_{k-1},a^c_{k}) = 1]$}
    (That of a complement).   
\item {\small $
[{\mathsf{I}(s_0.s_1.\dots.s_{k-1}, s_{k}) 
            = \mathsf{I}(s'_0.s'_1.\dots.s'_{k-1}, s_{k})}]$}
     (Synchronization condition on {\small $\mathsf{I}$} 
     interpretation; this reflects  
     the dependency of the existential fact of an attribute to 
     the existential fact of objects to which 
     it is an attribute).   
 \end{itemize}
 \item[Regarding global interpretations]{\ }  
     \begin{itemize}[leftmargin=-0.3cm]
\item {\small $[\mathsf{J}(s_0.s_1.\dots.s_k) =  
        1] \leftrightarrow^{\dagger}$} {\small $\forall i \in \mathbb{N}.
        \bigwedge_{i=0}^{\dagger k} 
        [\mathsf{I}(s_0.s_1.\dots.s_{i-1},s_i) = 1]$}\\    (Non-contradictory {\small $\mathsf{J}$} 
    valuation). 
\item {\small $[\mathsf{J}(s_0.s_1.\dots.s_k) = 
        0] \leftrightarrow^{\dagger}$} 
    {\small $
        \exists i \in \mathbb{N}.[i \le k] \wedge^{\dagger}
        [\mathsf{I}(s_0.s_1.\dots.s_{i-1}, s_i) = 0]$} \\
    (Contradictory {\small $\mathsf{J}$} valuation).   
  \end{itemize} 
  \end{description}  
  \label{interpretations}
\end{definition}    
Note that the global interpretation is completely characterised by the local interpretation. 
What we will need in the end are global interpretations; local 
interpretations are for intermediate value calculations 
for the ease of presentation of the semantics and 
of proofs of the main results. 
In the rest, we assume that any literal that appears 
in a formula is in a domain of discourse. 
\begin{definition}[Valuation]    
  Suppose a valuation frame {\small $\intFrame = (\mathsf{I}, \mathsf{J})$}. The following are defined to hold 
  for all  
  {\small $F_1, F_2\in \mathfrak{F}$} and for all 
  {\small $k \in \mathbb{N}$}:    
\begin{itemize}[leftmargin=0.3cm]
    \item {\small $[\intFrame \models s_0 \gtrdot s_1 \gtrdot \dots \gtrdot s_k] =  
  \mathsf{J}(s_0.s_1.\dots.s_k)$}.    
\item {\small $[\intFrame \models F_1 \wedge F_2] = 
  [\intFrame \models F_1] \wedge^{\dagger} 
  [\intFrame \models F_2]$}. 
\item {\small $[\intFrame \models F_1 \vee F_2] = 
  [\intFrame \models F_1] 
  \vee^{\dagger} [\intFrame \models F_2]$}. 
       \end{itemize}  
   \label{model}
\end{definition}    
\hide{  Given 
  any formula {\small $F \in \mathfrak{F}$}, 
  we denote by {\small $[F]_p$}  
  of {\small $\mathfrak{F}$} such that 
  it contains only those formulas (A) that result 
  from applying rules of transformations 
  
  \indent {\small $\{F' (\in \mathfrak{F})\ | \ F' \text{ is 
  expanded in primary chains via \textbf{Transformations} in Definition 3.}\}$}. 
  Likewise, we denote by {\small $[F]_u$} the set;\\
  \indent {\small $\{F' (\in \mathfrak{F})\ | \ F' \text{ is 
  expanded in unit chains via Definition 3.}\}$}.\footnote{
  Given a formula {\small $F$}, it may be that {\small $[F]_u$} 
  is always a singleton set, 
  {\small $F$} always leading to a unique unit expansion 
  via Definition 3. But we do not verify this in this paper, as 
  no later results strictly the stronger statement.}
  }
\noindent The 
notions of validity and satisfiability are 
as usual. 
\begin{definition}[Validity/Satisfiability]
    A formula {\small $F \in \mathfrak{F}$} is said 
  to be satisfiable in a valuation frame {\small $\mathfrak{M}$} 
  iff 
  {\small $1 = [\intFrame \models F]$}; 
  it is said to be valid iff it is satisfiable 
  for all the valuation frames; 
  it is said to be invalid iff 
  {\small $0 = [\intFrame \models F]$} for 
  some valuation frame {\small $\intFrame$}; 
  it is said to be unsatisfiable 
  iff it is invalid for all the valuation frames.    
  \label{universal_validity}
\end{definition}         
\subsection{Study on the semantics} 
We have not yet formally verified some important points.
Are there, firstly, any formulas {\small $F \in 
    \mathfrak{F}$} that do not reduce into 
some value-assignable formula? Secondly, what if 
both {\small $1 = [\mathfrak{M} \models F]$} 
and {\small $1 = [\mathfrak{M} \models \neg F]$}, 
or both {\small $0 = [\mathfrak{M} \models F]$} 
and {\small $0 = [\mathfrak{M} \models \neg F]$} 
for some {\small $F \in \mathfrak{F}$} under 
some {\small $\mathfrak{M}$}? Thirdly, should it happen 
that {\small $[\mathfrak{M} \models F] = 0 = 1$} 
for any formula {\small $F$}, given a valuation frame? \\
\indent If the first should hold, the semantics - 
the reductions and valuations as were presented in the previous 
sub-section - would not assign a value (values) to every member 
of {\small $\mathfrak{F}$} even with the reduction rules 
made available. If the second should hold, 
we could gain {\small $1 = [\mathfrak{M} \models F \wedge 
    \neg F]$}, which would relegate this gradual logic 
to a family of para-consistent logics - quite out of keeping 
with my intention. And the third should never hold, clearly. \\
\indent Hence it must be shown that these unfavoured situations
do not arise. An outline to the completion of the proofs is; 
\begin{enumerate} 
    \item to establish that every formula has a reduction 
        through {\small $\neg$} and {\small $\gtrdot$} reductions 
        into some formula {\small $F$} for which 
        it holds that {\small $\forall \mathfrak{M}. 
        [\mathfrak{M} \models F] \in \{0,1\}$}, to 
    settle down the first inquiry. 
\item to prove that
      any formula {\small $F$} 
      to which a value 0/1
      is assignable {\it without the use of 
      the reduction 
      rules} satisfies for every valuation frame (a) 
      that {\small $[\mathfrak{M} \models F] \vee^{\dagger} 
    [\mathfrak{M}\models
    \neg F] = 1$} and 
{\small $[\mathfrak{M} \models F] \wedge^{\dagger} [\mathfrak{M} \models
    \neg F] = 0$}; and (b) either that 
    {\small $0 \not= 1 = [\intFrame \models F]$} 
    or that {\small $1 \not= 0 = [\intFrame \models F]$},
    to settle down the other inquiries partially. 
\item to prove that
      the reduction through 
    {\small $\neg$} reductions and {\small $\gtrdot$} reductions 
    on any formula {\small $F \in \mathfrak{F}$} 
    is normal in that, 
    in whatever order those reduction rules are 
    applied to {\small $F$}, any {\small $F_{\textsf{reduced}}$} in 
    the set of possible formulas it reduces into satisfies 
    for every valuation frame either 
    that {\small $[\mathfrak{M} \models 
        F_{\textsf{reduced}}] = 1$}, or 
    that {\small $[\mathfrak{M} \models 
        F_{\textsf{reduced}}] = 0$}, 
    for all such {\small $F_{\textsf{reduced}}$}, 
    to conclude. 
\end{enumerate} 
\subsubsection{Every formula is 0/1-assignable} 
\vspace{-0.1mm}
We state several definitions for the first objective of ours.  
\hide{ 
\begin{definition}[Group formula]
    Let {\small $\mathfrak{G}_F: \mathfrak{F} \rightarrow 
        \mathfrak{F}$} be a function, defined 
    for all the sub-formulas {\small $F_a$} of 
    {\small $F$} as follows.
    \begin{itemize} 
        \item {\small $\mathfrak{G}_F(F_a) = F_a \wedge F_b$}
            if there exists {\small $F_b$} 
            such that {\small $F_a \wedge F_b$} 
            is a sub-formula of {\small $F$}. 
        \item {\small $\mathfrak{G}_F(F_a) = F_b \wedge F_a$}   
if there exists 
{\small $F_b$} 
            such that {\small $F_a \wedge F_b$} 
            is a sub-formula of {\small $F$}. 
        \item {\small $\mathfrak{G}_F(F_a) = F_a \vee F_b$} 
            if there exists 
{\small $F_b$} 
            such that {\small $F_a \vee F_b$} 
            is a sub-formula of {\small $F$}.  
        \item {\small $\mathfrak{G}_F(F_a) = F_b \vee F_a$} 
if there exists 
{\small $F_b$} 
            such that {\small $F_a \vee F_b$} 
            is a sub-formula of {\small $F$}.   
        \item {\small $\mathfrak{G}_F(F_a) = F_a$}, otherwise. 
    \end{itemize}
    Then the group formula of {\small $F_a$} 
    in {\small $F$} is defined 
    to be {\small $\mathfrak{G}_F^{k}(F_a)$} for some 
    {\small $k \in \mathbb{N}$} such that 
    {\small $\mathfrak{G}_F^{k}(F_a) = \mathfrak{G}^{k+1}_F(F_a)$}.
    We assume that {\small $\mathfrak{G}^0(F_a) = F_a$}. 
\end{definition}  
}
\begin{definition}[Chains/Unit chains/Unit chain expansion]
  A chain is defined to be any formula 
  {\small $F \in \mathfrak{F}$} such 
  that 
  {\small $F = F_0 \gtrdot F_1 \gtrdot \dots \gtrdot F_{k+1}$} 
  for {\small $k \in \mathbb{N}$}. 
  A unit chain is defined to be a chain 
  for which {\small $F_i \in \mathcal{S}$} for 
  all {\small $0 \le i \le k+1$}. We denote 
  the set of 
  unit chains by {\small $\mathfrak{U}$}.   
  By the head of a chain {\small $F_a \gtrdot F_b \in 
      \mathfrak{F}$}, we mean {\small $F_a$}; 
  and by the tail {\small $F_b$}.  
  and by the  tail 
  some formula {\small $F_a \in 
  \mathfrak{F}$} satisfying
  (1) that {\small $F_a$} is not in the form 
  {\small $F_b \gtrdot F_c$} for some 
  {\small $F_b, F_c \in \mathfrak{F}$} and (2) that 
  {\small $F = F_a \gtrdot F_d$} 
  for some {\small $F_d \in \mathfrak{F}$}.   
  By the tail of a chain {\small $F \in 
  \mathfrak{F}$}, we then mean some formula 
  {\small $F_d \in \mathfrak{F}$} such that  
  {\small $F = F_a \gtrdot F_d$} for 
  some {\small $F_a$} as the head of {\small $F$}.  
Given any {\small $F \in \mathfrak{F}$}, we 
  say that {\small $F$} is expanded in 
  unit chains only if 
  any chain that occurs in {\small $F$} is a unit chain.  
\end{definition}    
\begin{definition}[Formua length] 
   Let us define a function as follows. 
   \begin{itemize} 
       \item {\small $\forall s \in \mathcal{S}.\formulalength(s) = 1$}. 
       \item {\small $\forall F_1, F_2 \in \mathfrak{F}. 
               \formulalength(F_1 \wedge F_2) 
               = \formulalength(F_1 \vee F_2) 
               = \formulalength(F_1 \gtrdot F_2) 
               = \formulalength(F_1) + \formulalength(F_2) 
               + 1$}. 
       \item {\small $\forall F_1 \in \mathfrak{F}. 
               \formulalength(\neg F_1) = 
               1 + \formulalength(F_1)$}. 
   \end{itemize}
   Then we define the length of {\small $F \in \mathfrak{F}$} 
   to be {\small $\formulalength(F)$}. 
\end{definition} 
\begin{definition}[Maximal number of {\small $\neg$} nesting]
    Let us define a function. 
  \begin{itemize}
      \item {\small $\forall s \in \mathcal{S}.\negMax(F_0) = 0$}.
      \item {\small $\forall F_1, F_2 \in \mathfrak{F}. 
              \negMax(F_1 \wedge F_2) 
              = \negMax(F_1 \vee F_2) 
              = \negMax(F_1 \gtrdot F_2)
              = 
              \textsf{max}(\negMax(F_1), \negMax(F_2))$}. 
      \item 
          {\small $\forall F_1 \in \mathfrak{F}.\negMax(\neg F_1) = 1 + \negMax(F_1)$}.   
    \end{itemize} 
    Then we define the maximal number of 
    {\small $\neg$} nesting for {\small $F \in \mathfrak{F}$} 
    to be {\small $\negMax(F)$}. 
\end{definition}  
\noindent We now work on the main results.  
\begin{lemma}[Linking principle]
  Let {\small $F_1$} and {\small $F_2$} be two formulas 
  in unit chain expansion. Then it holds that 
  {\small $F_1 \gtrdot F_2$} has a reduction into 
  a formula in unit chain expansion.
  \label{linking_principle}
\end{lemma}   
\begin{proof} 
    Apply {\small $\gtrdot$} reductions 2 and 3 on 
      {\small $F_1 \gtrdot F_2$} into a formula 
      in which the only occurrences of the chains are 
      {\small $f_0 \gtrdot 
      F_2$}, {\small $f_1 \gtrdot F_2$}, \dots, {\small $f_{k} 
      \gtrdot F_2$} for some {\small $k \in \mathbb{N}$} and 
      some {\small $f_0, f_1, \dots, f_k \in \mathfrak{U} 
      \cup \mathcal{S}$}. Then apply {\small $\gtrdot$} reductions 
      4 and 5 to each of those chains into a formula 
      in which the only occurrences of the chains are: 
      {\small $f_0 \gtrdot g_{0}, f_0 \gtrdot g_{1}, \dots, 
      f_0 \gtrdot g_{j}$},
      {\small $f_1 \gtrdot g_{0}$}, \dots, 
      {\small $f_1 \gtrdot g_{j}$}, \dots, 
      {\small $f_k \gtrdot g_{0}$}, \dots, {\small $f_k \gtrdot g_j$} 
      for some {\small $j \in \mathbb{N}$} and 
      some {\small $g_0, g_1, \dots, g_j \in \mathfrak{U}$}. 
      To each such chain, apply 
      {\small $\gtrdot$} reduction 1 as long as it is applicable. 
      This process cannot continue infinitely since any formula 
      is finitely constructed and finitely 
      branching by any reduction rule, and since, on the
      assumptions, we 
      can apply induction
      on the number of elements of {\small $\mathcal{S}$} 
      occurring in {\small $g_x$}, {\small $0 \le x 
      \le j$}. 
      The straightforward inductive proof is left to readers. The 
      result is a formula in unit chain expansion.  
\end{proof}  
\begin{lemma}[Reduction without negation]
  Any formula {\small $F_0 \in \mathfrak{F}$} in which 
  no {\small $\neg$} occurs reduces into 
  some formula in unit chain expansion.
  \label{normalisation_without_negation}
\end{lemma}   
\begin{proof}    
    By induction on the formula 
    length. 
  For inductive cases, consider what 
  {\small $F_0$} actually is:  
  \begin{enumerate}[leftmargin=0.5cm]
    \item {\small $F_0 = F_1 \wedge F_2$} or 
      {\small $F_0 = F_1 \vee F_2$}: Apply induction 
      hypothesis on {\small $F_1$} and {\small $F_2$}. 
    \item {\small $F_0 = F_1 \gtrdot F_2$}: 
        Apply  
      induction hypothesis on {\small $F_1$} and {\small $F_2$} to 
      get {\small $F'_1 \gtrdot F'_2$} where {\small $F'_1$} and 
      {\small $F'_2$} are formulas in unit chain expansion. 
      Then apply Lemma \ref{linking_principle}.
        \end{enumerate}  
\end{proof}     
  \begin{lemma}[Reduction] 
    Any formula {\small $F_0 \in \mathfrak{F}$} reduces 
    into some formula in unit chain expansion. 
    \label{reduction_result}
  \end{lemma}  
  \begin{proof}  
    By induction on the maximal number 
    of {\small $\neg$} nesting, and a\linebreak sub-induction 
    on the formula length.
    We quote Lemma \ref{normalisation_without_negation} for 
        the 
    base cases. For the inductive cases, assume that the current 
    lemma holds true for all the formulas with {\small $\negMax(F_0)$} 
    of up to {\small $k$}. Then we conclude by showing that 
    it still holds true 
    for all the formulas with {\small $\negMax(F_0)$} of {\small $k+1$}. 
    Now, because any formula is finitely constructed,  
    there exist sub-formulas in which occur no {\small $\neg$}. 
    By Lemma \ref{normalisation_without_negation}, those sub-formulas 
    have a reduction into a formula in unit chain expansion. Hence 
    it suffices to show that those formulas 
    {\small $\neg F'$} with {\small $F'$} already in unit chain 
    expansion reduce into a formula in unit chain expansion, upon which
    inductive hypothesis applies for a conclusion. 
    Consider what {\small $F'$} is: 
    \begin{enumerate}
      \item {\small $s$}: then apply {\small $\neg$} reduction 1 
	on {\small $\neg F'$} to remove the {\small $\neg$}
	occurrence. 
      \item {\small $F_a \wedge F_b$}: apply {\small $\neg$} 
	reduction 2. 
	Then apply induction hypothesis on 
	{\small $\neg F_a$} and {\small $\neg F_b$}.  
      \item {\small $F_a \vee F_b$}: apply {\small $\neg$} 
	reduction 3. Then apply induction hypothesis on 
	{\small $\neg F_a$} and {\small $\neg F_b$}. 
      \item {\small $s \gtrdot F \in \mathfrak{U}$}: apply 
	{\small $\neg$} reduction 4. Then apply 
	induction hypothesis on {\small $\neg F$}.
    \end{enumerate} 
\end{proof}  
\begin{lemma}
  For any {\small $F \in \mathfrak{F}$} in unit chain 
  expansion, there 
  exists 
  {\small $v \in \{0,1\}$} such that {\small $[\intFrame \models
      F] = v$} for  
  any valuation frame. 
  \label{simple_lemma2}
\end{lemma}   
\begin{proof}
  Since 
  a value 0/1 is assignable to any element of 
  {\small $\mathcal{S} \cup \mathfrak{U}$} by Definition 
  \ref{model}, 
  it is (or they are if more than 
  one in \{0, 1\}) assignable to {\small $[\mathfrak{M} \models
  F]$}. 
\end{proof} 
\noindent Hence we obtain the desired result for the first objective.
\begin{proposition}
  To any {\small $F \in \mathfrak{F}$} corresponds  
  at least one formula {\small $F_a$} in unit chain expansion 
  into which {\small $F$} reduces. 
  It holds for any such {\small $F_a$} that 
  {\small $[\mathfrak{M} \models F_a] \in 
      \{0, 1\}$} for 
  any valuation frame. 
  \label{simple_proposition} 
\end{proposition}    
\noindent For the next sub-section, 
the following observation about the negation on
a unit chain comes in handy. Let us state a procedure. 
\begin{definition}[Procedure \recurseReduce]{\ }\\
  The procedure given below 
  takes as an input a formula {\small $F$} in unit chain expansion.\footnote{Instead of stating in 
        lambda calculus, we aim to be more descriptive 
        in this work 
        for not-so-trivial a function or a procedure, 
        using a pseudo program.}  
  \textbf{Description of {\small $\recurseReduce(F: \mathfrak{F})$}}
\begin{enumerate}[leftmargin=0.5cm]
    \item Replace {\small $\wedge$} in {\small $F$} 
      with {\small $\vee$}, and 
      {\small $\vee$} with {\small $\wedge$}. These two 
      operations are simultaneous. 
    \item Replace all the non-chains {\small $s \in \mathcal{S}$} 
      in {\small $F$} simultaneously with 
      {\small $s^c\ (\in \mathcal{S})$}.   
    \item For every chain {\small $F_a$} in {\small $F$} with 
      its head {\small $s \in \mathcal{S}$} for some 
      {\small $s$} and its tail 
      {\small $F_{\textsf{tail}}$}, replace {\small $F_a$}  
      with {\small $(s^c \vee (s \gtrdot 
      (\recurseReduce(F_{\textsf{tail}}))))$}.    
    \item Reduce {\small $F$} via {\small $\gtrdot$} reductions 
      in unit chain expansion. 
  \end{enumerate}
\end{definition}
\noindent Then we have the following result. 
\begin{proposition}[Reduction of negated unit chain expansion] 
  Let {\small $F$} be a formula in unit chain expansion. Then 
  {\small $\neg F$} reduces 
  via the {\small $\neg$} and {\small $\gtrdot$} reductions
  into {\small $\recurseReduce(F)$}. Moreover 
  {\small $\recurseReduce(F)$} is the unique reduction 
  of {\small $\neg F$}.  
    \label{special_reduction}
\end{proposition}  
\begin{proof} 
  For the uniqueness, observe that only  
  {\small $\neg$} reductions and 
  {\small $\gtrdot$} reduction 5 are used 
  in the reduction of {\small $\neg F$}, and that  
  at any point during the reduction, 
  if there occurs a sub-formula in the form {\small $\neg F_x$}, 
  the sub-formula {\small $F_x$} cannot be reduced by any 
  reduction rules. Then the proof of the uniqueness is 
  straightforward.  
\end{proof}
\hide{
\begin{lemma}[Simple observation]{\ }\\ 
  Let {\small $\psi$} denote 
  {\small $s_1.s_2.\dots.s_{k}$} for 
  some {\small $k \in \mathbb{N}$} ({\small $k = 0$} 
  means that {\small $s_1.s_2.\dots.s_{k} = \epsilon$}). 
  Then 
  it holds that {\small $[\models_{\psi} F] = [\models_{\epsilon} 
  s_1 \gtrdot s_2 \gtrdot \dots \gtrdot s_{k} \gtrdot F]$}. 
  {\ }\\
  \label{simple_observation}
\end{lemma} 
\begin{lemma}[Formula reconstruction]{\ }\\   
  Given any {\small $x \in \mathfrak{X}$}, 
  there exists a formula {\small $F \in \mathfrak{F}$} such that 
  {\small $x = [\models_{\epsilon} F]$} and that 
  all the chains occurring in {\small $x$}\footnote{ 
  In the following sense: 
  for any formulas which occur in 
  the uninterpreted expression {\small $x$}, 
  any chain that occurs in any one of them 
  is said to occur in {\small $x$}.}
  preserve in 
  {\small $F$}. 
  \label{formula_reconstruction}
\end{lemma}
\begin{proof}  
  Use the following recursions 
  to derive a formula {\small $F_a$} 
   from {\small $\underline{x}$}; 
  \begin{itemize}
    \item {\small $\underline{x_1 \oplus x_2} \leadsto  
      \underline{x_1} \wedge \underline{x_2}$}. 
    \item {\small $\underline{x_1 \odot x_2} \leadsto 
      \underline{x_1} \vee \underline{x_2}$}.  
    \item {\small $\underline{[\models_{s_0.s_1\dots.s_{k}} F_b]} 
      \leadsto s_0 \gtrdot s_1 \gtrdot \dots \gtrdot s_{k} \gtrdot 
      F_b$} for {\small $k \in \mathbb{N}$}. 
  \end{itemize}   
  We choose {\small $F_a$} for {\small $F$}, as required.  
\end{proof}  
} 
\subsubsection{Unit chain expansions form a Boolean algebra} 
We make use of disjunctive normal form 
in this sub-section for a simplification of proofs. 
\begin{definition}[Disjunctive/Conjunctive normal form]
  A formula {\small $F \in \mathfrak{F}$} is 
  defined to be in disjunctive normal form only if    
  {\small $\exists i,j,k \in \mathbb{N}\  
      \exists h_{0}, \cdots, h_i \in \mathbb{N}$}\linebreak
  {\small $\exists 
  f_{00}, \dots, f_{kh_k} \in 
  \mathfrak{U} \cup \mathcal{S}.F = \vee_{i =0}^k \wedge_{j = 0}^{h_i} 
  f_{ij}$}.
 Dually, a formula {\small $F \in \mathfrak{F}$} 
  is defined to be in conjunctive normal form 
  only if {\small $\exists i, j, k \in \mathbb{N}\ \exists  
  h_0, \cdots, h_i \in \mathbb{N}\
  \exists 
  f_{00}, \dots, f_{kh_k} \in 
  \mathfrak{U} \cup \mathcal{S}.F = \wedge_{i=0}^k \vee_{j =0}^{h_i} f_{ij}$}.  
\end{definition}      
\noindent Now, for the second objective of ours, we prove 
that {\small $\mathfrak{U} \cup \mathcal{S}$}, 
{\small $\recurseReduce$}, {\small $\vee^{\dagger}$} and {\small $\wedge^{\dagger}$} 
form a Boolean algebra,\footnote{http://en.wikipedia.org/wiki/Boolean\_algebra for the laws 
of Boolean algebra.} from which follows 
the required outcome.          
\begin{proposition}[Annihilation/Identity]
  For any formula {\small $F$}   
  in unit chain expansion and 
  for any valuation frame, 
  it holds (1) that {\small $[\intFrame \models \top \wedge 
  F] = [\intFrame \models F]$}; 
  (2) that {\small $[\intFrame \models \top \vee 
  F] = [\intFrame \models \top]$};
  (3) that 
  {\small $[\intFrame \models \bot \wedge F] =
  [\intFrame \models \bot]$}; 
  and (4) that  
  {\small $[\intFrame \models \bot \vee F] =
  [\intFrame \models F]$}. 
\end{proposition}   
\begin{lemma}[Elementary complementation]  
  For any {\small $s_0 \gtrdot s_1 \gtrdot \dots \gtrdot s_k
  \in \mathfrak{U} \cup \mathcal{S}$} for some 
  {\small $k \in \mathbb{N}$}, 
  if for a given valuation frame 
  it holds that 
  {\small $[\mathfrak{M} \models
  s_0 \gtrdot s_1 \gtrdot \dots \gtrdot s_{k}] = 1$}, 
  then it also holds that 
  {\small $[\intFrame \models \recurseReduce(s_0 
  \gtrdot s_1 \gtrdot \dots 
  \gtrdot s_{k})] = 0$}; or 
  if it holds that 
  {\small $[\intFrame \models s_0 \gtrdot s_1 \gtrdot \dots 
  \gtrdot s_{k}] = 0$}, then 
  it holds that 
  {\small $[\intFrame \models \recurseReduce(s_0 \gtrdot s_1 \gtrdot \dots
  \gtrdot s_{k})] = 1$}. These two events are mutually 
  exclusive.
  \label{unit_chain_excluded_middle}
\end{lemma}
\begin{proof}               
 For the first one,
 {\small $[\intFrame \models s_0 \gtrdot  
     s_1 \gtrdot 
 \dots \gtrdot s_{k}] = 1$} implies that
 {\small $\mathsf{I}(\epsilon, s_0)\! =\! 
 \mathsf{I}(s_0, s_1) \!=\! \dots \!=\! 
 \mathsf{I}(s_0.s_1.\dots.s_{k - 1}, s_{k})\! =\! 1$}.
 So we have; 
 {\small $\mathsf{I}(\epsilon, s_0^c)\!\! =\!\!
     \mathsf{I}(s_0, s_1^c)\!\! =\!\! \dots\!\! =\!\! 
 \mathsf{I}(s_0.s_1\dots.s_{k - 1}, s_{k}^c) 
 \!\!=\!\! 0$} by the 
 definition of {\small $\mathsf{I}$}. 
 Meanwhile, 
 {\small $\recurseReduce(s_0 \gtrdot s_1 \gtrdot \cdots \gtrdot 
 s_k) = s_0^c \vee (s_0 \gtrdot ((s_1^c \vee (s_1 \gtrdot \cdots)))) 
 = s_0^c \vee (s_0 \gtrdot s_1^c) \vee (s \gtrdot s_1 \gtrdot 
 s_2^c) \vee \cdots \vee 
 (s \gtrdot s_1 \gtrdot \cdots \gtrdot s_{k-1} \gtrdot s_k^c)$}. 
 Therefore {\small $[\intFrame \models
 \recurseReduce(s_0 \gtrdot s_1 \gtrdot \cdots \gtrdot s_k)]  
 = 0 \not= 1$} for the given valuation frame. \\
 \indent For the second obligation, 
 {\small $[\intFrame \models
 s_0 \gtrdot s_1 \gtrdot \dots \gtrdot s_{k}] = 0$} 
 implies that   
 {\small $[\mathsf{I}(\epsilon, s_0) 
 = 0] \vee^{\dagger} [\mathsf{I}(s_0, s_1) = 0] 
 \vee^{\dagger} \dots \vee^{\dagger} [\mathsf{I}(s_0.s_1.\dots.
 s_{k -1}, s_{k}) = 0]$}. Again 
 by the definition of {\small $\mathsf{I}$}, 
 we have the required result. That these two events 
 are mutually exclusive is trivial. 
    \hide{
 For the first one,
 {\small $[(\mathsf{I}, \mathsf{J})\models s_0 \gtrdot 
 s_1 \gtrdot \dots \gtrdot s_{k}] = 1$} implies that
 {\small $\mathsf{I}(|\epsilon |, s_0)\! =\! 
 \mathsf{I}(|s_0|, s_1) \!=\! \dots \!=\! 
 \mathsf{I}(| s_0.s_1.\dots.s_{k - 1}|, s_{k})\! =\! 1$}.
 So we have; 
 {\small $\mathsf{I}(| \epsilon |, s_0^c) = 
 \mathsf{I}(| s_0 |, s_1^c) = \dots = 
 \mathsf{I}(| s_0.s_1\dots.s_{k - 1} |, s_{k}^c) = 0$} by the 
 definition of {\small $\mathsf{I}$}. 
 Meanwhile, 
 {\small $\recurseReduce(s_0 \gtrdot s_1 \gtrdot \cdots \gtrdot 
 s_k) = s_0^c \vee (s_0 \gtrdot ((s_1^c \vee (s_1 \gtrdot \cdots)))) 
 = s_0^c \vee (s_0 \gtrdot s_1^c) \vee (s \gtrdot s_1 \gtrdot 
 s_2^c) \vee \cdots \vee 
 (s \gtrdot s_1 \gtrdot \cdots \gtrdot s_{k-1} \gtrdot s_k^c)$}. 
 Therefore {\small $[(\mathsf{I}, \mathsf{J})\models
 \recurseReduce(s_0 \gtrdot s_1 \gtrdot \cdots \gtrdot s_k)]  
 = 0 \not= 1$} for the given interpretation frame. \\
 \indent For the second obligation, 
 {\small $[(\mathsf{I}, \mathsf{J})\models
 s_0 \gtrdot s_1 \gtrdot \dots \gtrdot s_{k}] = 0$} 
 implies that   
 {\small $[\mathsf{I}(|\epsilon|, s_0) 
 = 0] \vee^{\dagger} [\mathsf{I}(|s_0|, s_1) = 0] 
 \vee^{\dagger} \dots \vee^{\dagger} [\mathsf{I}(|s_0.s_1.\dots.
 s_{k -1}|, s_{k}) = 0]$}. Again 
 by the definition of {\small $\mathsf{I}$}, 
 we have the required result. That these two events 
 are mutually exclusive is trivial. \\ 
    }
\end{proof}  
\begin{proposition}[Associativity/Commutativity/Distributivity]
  Given any formulas {\small $F_1, F_2, F_3 \in \mathfrak{F}$} 
  in unit chain expansion and any valuation frame  
  {\small $\mathfrak{M}$}, the following hold:  
  \begin{enumerate}
    \item  {\small $[\intFrame \models F_1] \wedge^{\dagger} 
  ([\intFrame \models F_2] \wedge^{\dagger} 
  [\intFrame \models F_3]) = ([\intFrame \models F_1] \wedge^{\dagger} 
  [\intFrame \models F_2]) \wedge^{\dagger} [\intFrame \models F_3]$} (associativity 1). 
\item {\small $[\intFrame \models F_1] \vee^{\dagger} 
  ([\intFrame \models F_2] \vee^{\dagger} [\intFrame \models F_3]) 
  = ([\intFrame \models F_1] \vee^{\dagger} [\intFrame \models F_2]) \vee^{\dagger} F_3$} 
  (associativity 2). 
\item {\small $[\intFrame \models F_1] \wedge^{\dagger} [\intFrame \models F_2] 
  = [\intFrame \models F_2] \wedge^{\dagger} [\intFrame \models F_1]$} (commutativity 1). 
\item {\small $[\intFrame \models F_1] \vee^{\dagger} [\intFrame \models F_2] 
  = [\intFrame \models F_2] \vee^{\dagger} [\intFrame \models F_1]$} (commutativity 2).  
\item {\small $[\intFrame \models F_1] \wedge^{\dagger} 
  ([\intFrame \models F_2] \vee^{\dagger} [\intFrame \models F_3]) = 
  ([\intFrame \models F_1] \wedge^{\dagger} [\intFrame \models F_2]) \vee^{\dagger} 
  ({[\intFrame \models F_1]} \wedge^{\dagger} [\intFrame \models F_3])$} (distributivity 1).  
\item {\small $[\intFrame \models F_1] \vee^{\dagger} 
  ([\intFrame \models F_2] \wedge^{\dagger} [\intFrame \models F_3]) = 
  ([\intFrame \models F_1] \vee^{\dagger} [\intFrame \models F_2]) \wedge^{\dagger} 
  ({[\intFrame \models F_1]} \vee^{\dagger} [\intFrame \models F_3])$} (distributivity 2).  
  \end{enumerate} 
  \label{associativity_commutativity_distributivity}
\end{proposition}  
\begin{proof}       
    Make use of Lemma \ref{unit_chain_excluded_middle} 
    to note that each 
    {\small $[\mathfrak{M} \models f]$} 
    for {\small $f \in \mathfrak{U} \cup \mathcal{S}$} 
    is assigned one and only one value 
    {\small $v \in \{0,1\}$}. Straightforward 
    with the observation. 
    \hide{
  Let us generate a set of 
  expressions finitely constructed from the following grammar;\\
  \indent {\small $X := [(\mathsf{I}, \mathsf{J})\models f] \ | \ X \wedge^{\dagger} X 
  \ | \ X \vee^{\dagger} X$} where {\small $f \in \mathfrak{U} \cup 
  \mathcal{S}$}. \\
  Then first of all it is straightforward to 
  show that {\small $[(\mathsf{I}, \mathsf{J})\models F_i] = X_i$} 
  for each {\small $i \in \{1,2,3\}$} for some {\small $X_1, X_2, X_3$} 
  that the above grammar recognises. By Lemma \ref{unit_chain_excluded_middle} each expression ({\small $[(\mathsf{I}, \mathsf{J})\models f_x]$} for 
  some {\small $f_x \in \mathfrak{U} \cup \mathcal{S}$}) 
  is assigned one and only one value {\small $v \in \{0,1\}$}. Then 
  since {\small $1 \vee^{\dagger} 1 = 1 \vee^{\dagger} 0 = 
  0 \vee^{\dagger} 1 = 1$}, 
  {\small $0 \wedge^{\dagger} 0 = 0 \wedge^{\dagger} 1 = 
  1 \wedge^{\dagger} 0 = 0$}, and 
  {\small $1 \wedge^{\dagger} 1 = 1$} by definition (given at 
  the beginning 
  of this section), 
  it is also the case that {\small $[(\mathsf{I}, \mathsf{J})\models F_i]$} 
  is assigned one and only one value {\small $v_i \in  \{0,1\}$} 
  for each {\small $i \in \{1,2,3\}$}. Then the proof for the 
  current proposition is straightforward. \\ 
    }
\end{proof} 
\begin{proposition}[Idempotence and Absorption]
  Given any formula\linebreak {\small $F_1, F_2 \in \mathfrak{F}$} 
  in unit chain expansion, 
  for any valuation frame it holds that 
  {\small $[\intFrame \models F_1] \wedge^{\dagger} [\intFrame \models F_1] 
  = {[\intFrame \models F_1]} \vee^{\dagger} [\intFrame \models F_1] = [\intFrame \models F_1]$} 
  (idempotence); and 
  that {\small $[\intFrame \models F_1] \wedge^{\dagger} 
  ([\intFrame \models F_1] \vee^{\dagger} {[\intFrame \models F_2]}) 
  = [\intFrame \models F_1] \vee^{\dagger} 
  ([\intFrame \models F_1] \wedge^{\dagger} [\intFrame \models F_2]) = [\intFrame \models F_1]$} 
  (absorption). 
  \label{idempotence_absorption}
\end{proposition}
\begin{proof}  
  Both {\small $F_1, F_2$} are assigned one and only one value 
  {\small $v \in \{0,1\}$}. 
   Trivial to verify. 
\end{proof} 
\noindent We now prove the laws involving \recurseReduce. 
\begin{lemma}[Elementary double negation]
  Let {\small $F$} denote {\small $s_0 \gtrdot s_1 \gtrdot \cdots \gtrdot 
  s_k \in \mathfrak{U} \cup \mathcal{S}$} for 
  some {\small $k \in \mathbb{N}$}.
  Then 
  for any valuation frame it holds 
  that {\small $[\intFrame \models F] 
  = [\intFrame \models \recurseReduce(\recurseReduce(F))]$}. 
  \label{unit_double_negation}
\end{lemma} 
\begin{proof}  
   {\small $\recurseReduce(\recurseReduce(F)) =  
       \recurseReduce(s^c_0 \vee$}\linebreak
   {\small $(s_0 \gtrdot s^c_1) \vee 
  \cdots \vee 
  (s_0 \gtrdot s_1 \gtrdot \cdots \gtrdot s_{k-1} 
  \gtrdot s^c_{k})) = 
   s_0 \wedge (s^c_0 \vee (s_0 \gtrdot s_1)) 
  \wedge (s_0^c \vee (s_0 \gtrdot s_1^c) 
  \vee (s_0 \gtrdot s_1 \gtrdot s_2)) 
  \wedge \cdots \wedge 
  (s^c_0 \vee (s_0 \gtrdot s_1^c) \vee \cdots 
  \vee (s_0 \gtrdot s_1 \gtrdot \cdots \gtrdot s_{k-2} \gtrdot 
  s^c_{k-1}) 
  \vee (s_0 \gtrdot s_1 \gtrdot \cdots \gtrdot s_{k}))$}.   
  Here, assume that the right hand side of the equation
  which is in conjunctive normal form is ordered, 
  the number of terms, from left to right, strictly increasing 
  from 1 to {\small $k + 1$}. Then as the result of a transformation
  of the conjunctive 
  normal form into disjunctive normal form we will 
  have 1 (the choice from the first conjunctive clause which contains 
  only one term {\small $s_0$}) {\small $\times$} 
  2 (a choice from the second conjunctive clause with 
  2 terms {\small $s_0^c$} and {\small $s_0 \gtrdot s_1$}) 
  {\small $\times$} \ldots {\small $\times$} (k $+$ 1) clauses. But  
  almost all the clauses in 
  {\small $[\intFrame \models (\text{the disjunctive
  normal form})]$}
  will be assigned 0 (trivial; the proof left to readers) so that we gain
  {\small $[\intFrame \models (\text{the disjunctive normal form})] 
  = [\intFrame \models s_0] \wedge^{\dagger} [\intFrame \models
  s_0 \gtrdot s_1] \wedge^{\dagger} \cdots 
  \wedge^{\dagger} [\intFrame \models s_0 \gtrdot s_1 
  \gtrdot \cdots \gtrdot s_k] =
  [\intFrame \models s_0 \gtrdot s_1 
  \gtrdot \cdots \gtrdot s_k]$}. 
\end{proof} 
\begin{proposition}[Complementation/Double negation]{\ }\\
  For any {\small $F$} in unit chain expansion 
  and for any valuation frame, we have  
  {\small $1 = [\intFrame \models F \vee 
  \recurseReduce(F)]$} 
   and that 
  {\small $0 = [\intFrame \models F \wedge \recurseReduce(F)]$} 
  (complementation). 
  Also, for any {\small $F \in \mathfrak{F}$} in unit chain 
  expansion and 
  for any valuation frame 
  we have {\small $[\intFrame \models F] 
  = [\intFrame \models \recurseReduce(\recurseReduce(F))]$} (double negation). 
\label{excluded_middle}
\end{proposition} 
\begin{proof}           
By Proposition \ref{associativity_commutativity_distributivity}, 
  {\small $F$} has a disjunctive normal form: 
  {\small $F = \bigvee_{i = 0}^{k} \bigwedge_{j=0}^{h_i}   
  f_{ij}$} for some {\small $i, j, k \in \mathbb{N}$}, 
  some {\small $h_0, \cdots, h_k \in \mathbb{N}$} 
  and some\linebreak {\small $f_{00}, \cdots, f_{kh_k} \in 
  \mathfrak{U} \cup \mathcal{S}$}.  
  Then we have that;\\\\ \indent{\small $\recurseReduce(F)  
  = \bigwedge_{i=0}^k \bigvee_{j=0}^{h_i} \recurseReduce(f_{ij})$},\\\\
  which, if transformed into a disjunctive normal form, 
  will have\linebreak {\small $(h_0 + 1)$} [a choice from
  {\small $\recurseReduce(f_{00}), \recurseReduce(f_{01}), \dots,\\
      \recurseReduce(f_{0h_0})$}] {\small $\times$}
  {\small $(h_1 + 1)$} [a choice from 
  {\small $\recurseReduce(f_{10}),\\ \recurseReduce(f_{11}), \dots,
  \recurseReduce(f_{1h_1})$}] {\small $\times \dots \times$} 
  {\small $(h_k + 1)$} clauses. Now if 
  {\small $[\intFrame \models F] = 1$}, then we already have the required 
  result. Therefore suppose that {\small $[\intFrame 
      \models F] = 0$}. 
  Then it holds that {\small $\forall i \in \{0, \dots, k\}. 
  \exists j \in \{0, \dots, h_i\}.([\intFrame \models f_{ij}] = 0)$}.  
  By Lemma \ref{unit_chain_excluded_middle}, this is equivalent to 
  saying that {\small $\forall i \in \{0, \dots, k\}. 
  \exists j \in \{0, \dots, h_i\}.([\intFrame \models \recurseReduce(f_{ij})] 
  = 1)$}. But then a clause in disjunctive normal form 
  of {\small $[\intFrame \models \recurseReduce(F)]$} exists,
  which is assigned 1. 
  Dually for {\small $0 = [\intFrame \models F \wedge \recurseReduce(F)]$}. 
  \\
  \indent For {\small $[\intFrame \models F] = 
  [\intFrame \models \recurseReduce(\recurseReduce(F))]$},  
  by Proposition \ref{associativity_commutativity_distributivity}, 
  {\small $F$} has a disjunctive normal form: 
  {\small $F = \bigvee_{i = 0}^k \bigwedge_{j=0}^{h_i} f_{ij}$} 
  for some {\small $i, j, k \in \mathbb{N}$}, 
  some {\small $h_0, \dots, h_k \in \mathbb{N}$} and 
  some {\small $f_{00}, \dots, f_{kh_k} \in \mathfrak{U} \cup 
  \mathcal{S}$}.  Then;\\\\\indent   {\small $\recurseReduce(\recurseReduce(F)) 
  =$}\\\indent\indent
  {\small $\bigvee_{i = 0}^{k} \bigwedge_{j=0}^{h_i} \recurseReduce(
  \recurseReduce(f_{ij}))$}.\\\\ But by Lemma \ref{unit_double_negation}
  {\small $[\intFrame \models \recurseReduce(\recurseReduce(f_{ij}))] = 
      [\intFrame \models f_{ij}]$} for each appropriate {\small $i$} and 
  {\small $j$}. Straightforward. 
\end{proof}      
\begin{theorem} 
  Denote by {\small $X$} 
  the set of the expressions comprising all 
  {\small $[\intFrame \models f_x]$} for 
  {\small $f_x \in \mathfrak{U} \cup \mathcal{S}$}. 
  Then for every valuation frame, it holds that\linebreak {\small $(X, \recurseReduce, \wedge^{\dagger}, \vee^{\dagger})$} 
  defines a Boolean algebra. 
  \label{theorem_1}
\end{theorem}
\begin{proof}
  Follows from earlier propositions and lemmas. 
\end{proof} 
\subsubsection{Gradual classical logic is neither 
para-consistent nor inconsistent}  
To achieve the last objective we assume several notations.  
\begin{definition}[Sub-formula notation] 
    Given a formula {\small $F \in \mathfrak{F}$}, 
    we denote by {\small $F[F_a]$} the fact that 
    {\small $F_a$} occurs as a sub-formula 
    in {\small $F$}. Here the definition 
    of a sub-formula of a formula 
    follows that which is found in standard textbooks 
    on mathematical logic \cite{Kleene52}. {\small $F$} 
    itself is a sub-formula of {\small $F$}. 
\end{definition}   
\hide{ 
\begin{proposition}[Reduction on the induction measure] 
  Given a formula {\small $F$}, 
  any reduction on a sub-formula of {\small $F$} 
  either does not alter, or otherwise reduces the 
  size of formula structure of any sub-formula. 
  \label{important_observation} 
\end{proposition}
\begin{proof} 
  Consider cases. But this does not work because 
  according to the current induction measure, 
  if a reduction takes place above, 
  it may happen that the occurrence of 
  an attributed object increase. So instead 
  of simply counting the number, 
  it should use the maximum. 
  First, if a reduction 
\end{proof}  
}
\begin{definition}[Small step reductions] 
    By {\small $F_1 \leadsto F_2$} 
    for some formulas {\small $F_1$} and {\small $F_2$} 
    we denote that {\small $F_1$} reduces in one 
    reduction step into {\small $F_2$}. By 
  {\small $F_1 \leadsto_{r} F_2$} we denote that 
  the reduction holds explicitly by 
  a reduction rule {\small $r$} (which is either of the 
  7 rules). By {\small $F_1 \leadsto^* F_2$} we denote 
  that {\small $F_1$} reduces 
  into {\small $F_2$} in a finite number of steps including 
  0 step in which case {\small $F_1$} is said to be 
  irreducible. By {\small $F_1 \leadsto^k F_2$} we denote 
  that the reduction is in exactly {\small $k$} steps. 
  By 
  {\small $F_1 \leadsto^*_{\{r_1, r_2, \cdots\}} F_2$} or 
  {\small $F_1 \leadsto^k_{\{r_1, r_2, \cdots\}} F_2$} we denote 
  that the reduction is via those specified rules 
  {\small $r_1, r_2, \cdots$} only. 
\end{definition}   
\begin{definition}[Formula size] 
    Let us define
    a function that outputs a positive rational number, 
    as follows. The A: B notation derives from 
    programming practice, but simply says that 
    $A$ is a member of $B$. \\
    \textbf{Description of} 
    {\small $\lformulasize(d : \mathbb{N},l: \mathbb{N}, bool:  
        \textsf{Boolean}, 
         F: \mathfrak{F})$} 
    \textbf{outputting a positive rational number} 
    \begin{enumerate}   
        \item If {\small $F = s$} for some 
            {\small $s \in \mathcal{S}$}, 
            then return {\small $1/4^l$}. 
        \item If {\small $F = \neg F_1$} 
            for {\small $F_1 \in \mathfrak{F}$}, then   
            return  {\small $(1/4^{d}) + \lformulasize(d, l, 
                bool, 
                F_1)$}.     
\item If {\small $F = F_1 \gtrdot F_2$}, 
            then return {\small $\lformulasize(d+1, l, 
                bool,F_1) 
                + \lformulasize(d+1, l,bool, F_2)$}.  
        \item If {\small $F = F_1 \wedge F_2$} 
            or {\small $F = F_1 \vee F_2$}, then 
            \begin{enumerate} 
                \item If {\small $bool$} is true,
                    return {\small $\textsf{max}(\lformulasize(d+1,l+1,\textsf{false},F_1), \lformulasize(d+1,l+1,\textsf{false},F_2))$}.
                \item Otherwise, return 
                    {\small $\textsf{max}(\lformulasize(d,l,\textsf{false},F_1),\lformulasize(d,l,\textsf{false},F_2))$}. 
            \end{enumerate}
      \end{enumerate}
    Then we define the size of {\small $F$} to be
    {\small $\lformulasize(0,0, \textsf{true}, 
        F)$}.  
\end{definition}   
\noindent The purpose of the last definition, 
including the choice of {\small $1/4^l$}, 
is just so that the formula size at each 
formula reduction does not increase. There is a 
one-to-one mapping between 
all the numbers as may be returned by this function and 
a subset of {\small $\mathbb{N}$}. 
Other than for the stated  purpose,
there is no rationale behind the particular decisions 
in the definition. 
Readers should not try to figure 
any deeper intuition, for there is none. 
\begin{proposition}[Preliminary observation]   
    \label{preliminary_observation}      
    The following results hold. {\small $\textsf{b} 
        \in \{\textsf{true}, \textsf{false}\}$}. 
    \begin{enumerate}[leftmargin=0.7cm]
        \item   {\small $\lformulasize(d,l, \textsf{b}, 
                \neg s) \ge 
                \lformulasize(d,l, \textsf{b}, 
                s^c)$}. 
        \item {\small $\lformulasize(d,l, \textsf{b}, 
                \neg (F_1 \wedge F_2)) 
                \ge \lformulasize(d,l, \textsf{b}, 
                \neg F_1 \vee \neg F_2)$}. 
        \item {\small $\lformulasize(d,l, \textsf{b}, 
                \neg (F_1 \vee F_2)) 
                \ge \lformulasize(d,l, \textsf{b}, 
                \neg F_1 \wedge \neg F_2)$}. 
        \item {\small $\lformulasize(d,l, \textsf{b}, 
                \neg (s \gtrdot F_2)) 
                \ge \lformulasize(d,l, \textsf{b}, 
                s^c \vee (s \gtrdot \neg F_2))$}.   
        \item {\small $\lformulasize(d,l, \textsf{b},
                (F_1 \gtrdot F_2) \gtrdot 
                F_3) \ge$}\\
            {\small $\lformulasize(d,l, \textsf{b}, 
                (F_1 \gtrdot F_3) 
                \wedge ((F_1 \gtrdot F_2) \vee (F_1 \gtrdot F_2 
                \gtrdot F_3)))$}. 
        \item {\small $\lformulasize(d,l, \textsf{b}, F_1 \wedge F_2 \gtrdot F_3) 
                \ge \lformulasize(d,l, \textsf{b}, (F_1 \gtrdot F_3) \wedge 
                (F_2 \gtrdot F_3))$}. 
        \item {\small $\lformulasize(d,l, \textsf{b}, 
                F_1 \vee F_2 \gtrdot F_3) 
                \ge \lformulasize(d,l, \textsf{b}, (F_1 \gtrdot F_3) \vee 
                (F_2 \gtrdot F_3))$}. 
        \item {\small $\lformulasize(d,l, \textsf{b}, F_1 \gtrdot F_2 \wedge F_3) 
                \ge \lformulasize(d,l, \textsf{b}, (F_1 \gtrdot F_2) \wedge 
                (F_1 \gtrdot F_3))$}. 
        \item {\small $\lformulasize(d,l, \textsf{b}, F_1 \gtrdot F_2 \vee F_3) 
                \ge \lformulasize(d,l, \textsf{b}, (F_1 \gtrdot F_2) \vee 
                (F_1 \gtrdot F_3))$}.    
        \item {\small $\lformulasize(d,l, \textsf{b},
                F_1 \wedge F_2) = 
                \lformulasize(d,l, \textsf{b}, 
                F_2 \wedge F_1)$}. 
        \item {\small $\lformulasize(d,l, \textsf{b}, 
                F_1 \vee F_2) = 
                \lformulasize(d,l, \textsf{b}, 
                F_2 \vee F_1)$}. 
        \item {\small $\lformulasize(d,l, \textsf{b},
                (F_1 \wedge F_2) \wedge F_3) 
                = \lformulasize(d,l, \textsf{b},  
                F_1 \wedge (F_2 \wedge F_3))$}. 
        \item {\small $\lformulasize(d,l, \textsf{b}, 
                (F_1 \vee F_2) \vee F_3) = 
                \lformulasize(d,l, \textsf{b}, 
                F_1 \vee (F_2 \vee F_3))$}.   
     \end{enumerate}
\end{proposition} 
\begin{proof}    
    Shown with an assistance of a Java program. 
    The source code and the test cases are found in Appendix A. 
\end{proof} 
\noindent Along with the above notations, we also enforce that  
{\small $\mathcal{F}(F)$} denote the set of 
formulas in unit chain expansion that {\small $F \in \mathfrak{F}$}  
can reduce into. A stronger result than Lemma \ref{normalisation_without_negation} 
follows. 
\begin{theorem}[Bisimulation]   
\label{bisimulation}
  Assumed below are pairs of formulas. 
  {\small $F'$} differs from {\small $F$} only by
  the shown sub-formulas, \emph{i.e.} {\small $F'$} 
  derives from {\small $F$} by replacing the shown sub-formula 
  for {\small $F'$} 
  with the shown sub-formula for {\small $F$} and vice versa. 
  Then for each pair  
  {\small $(F, F')$} 
  below, it holds for every valuation frame 
  that 
  {\small $[\intFrame \models F_1] = [\intFrame \models F_2]$} for 
  all {\small $F_1 \in \mathcal{F}(F)$} and 
  for all {\small $F_2 \in \mathcal{F}(F')$}.   
  {\small 
  \begin{eqnarray}\nonumber 
    F[F_a \wedge F_b \gtrdot F_c] &,& F'[(F_a \gtrdot F_c) 
    \wedge (F_b \gtrdot F_c)]\\\nonumber
    F[F_a \vee F_b \gtrdot F_c] &,& F'[(F_a \gtrdot F_c) 
    \vee (F_b \gtrdot F_c)]\\\nonumber 
    F[F_a \gtrdot F_b \wedge F_c] &,& F'[(F_a \gtrdot F_b)
    \wedge (F_a \gtrdot F_c)]\\\nonumber 
    F[F_a \gtrdot F_b \vee F_c] &,& F'[(F_a \gtrdot F_b)
    \vee (F_a \gtrdot F_c)]\\\nonumber 
    F[(F_a \gtrdot F_b) \gtrdot F_c] &,&
    F'[(F_a \gtrdot F_c) \wedge ((F_a \gtrdot F_b) \vee (F_a \gtrdot 
    F_b \gtrdot F_c))]\\\nonumber
    F[\neg s] &,& F'[s^c]\\\nonumber
    F[\neg(F_1 \wedge F_2)] &,&
    F'[\neg F_1 \vee \neg F_2]\\\nonumber
    F[\neg(F_1 \vee F_2)] &,&
    F'[\neg F_1 \wedge \neg F_2]\\\nonumber
    F[\neg(s \gtrdot F_2)] &,& 
    F'[s^c \vee (s \gtrdot \neg F_2)]\\\nonumber
    F[F_a \vee F_a] &,& 
    F'[F_a]\\\nonumber
    F[F_a \wedge F_a] &,&
    F'[F_a]\\\nonumber
    F[F_a \wedge F_b] &,& 
    F'[F_b \wedge F_a]\\\nonumber
    F[F_a \vee F_b] &,&
    F'[F_b \vee F_a]\\\nonumber
    F[F_a \wedge (F_b \wedge F_c)] &,& 
    F'[(F_a \wedge F_b) \wedge F_c]\\\nonumber
    F[F_a \vee (F_b \vee F_c)] &,&
    F'[(F_a \vee F_b) \vee F_c]
  \end{eqnarray}  
  }
  \end{theorem} 
\begin{proof} 
    By simultaneous induction on  
    the size of the formula that is not 
    a strict sub-formula of any other formulas\footnote{That is,
        if {\small $F \leadsto F_a \leadsto F_b \leadsto \dots$},
        then we get {\small $\lformulasize(0, 0,\textsf{true},F),
            \lformulasize(0, 0,\textsf{true},F_a),\lformulasize(0, 0,\textsf{true},F_b)...$}}, 
    a sub-induction on the inverse of (the number of 
    occurrences of {\small $\neg$} + 1)\footnote{If {\small $\neg$} 
        occurs once, then we get {\small $1/2$}. If 
        it occurs twice, then we get {\small $1/3$}.} and 
    a sub-sub-induction on the inverse of (the number of 
    occurrences of {\small $\gtrdot$} + 1). None of these 
    are generally an intger; but there is a mapping 
    into {\small $\mathbb{N}$}, so that 
    a larger number maps into a larger natural number. 
    The composite induction measure strictly decreases
    at each reduction (Cf. Appendix A). 
    We 
   first establish that {\small $
       \mathcal{F}(F_1) = \mathcal{F}(F_2)$} (by bisimulation). 
      One way to show that to each 
   reduction on {\small $F'$} corresponds 
   reduction(s) on {\small $F$} is straightforward, 
   for we can choose to reduce {\small $F$} into {\small $F'$}, 
   thereafter synchronizing both of the reductions. Into the 
   other way to show that to each 
   reduction on {\small $F$} corresponds 
   reduction(s) on {\small $F'$}; 
   \begin{enumerate}
     \item The first pair.  
       \begin{enumerate} 
	 \item 
       If a reduction takes place on a sub-formula which 
       neither is a sub-formula of the shown sub-formula 
       nor takes as its sub-formula the shown sub-formula, 
       then we reduce the same sub-formula in {\small $F'$}.   
       Induction hypothesis on the pair of 
       the reduced formulas. (The 
       formula size of the stated 
       formulas is that of {\small $F$} in this direction 
       of the proof). 
     \item If it takes place on a sub-formula 
       of {\small $F_a$} or {\small $F_b$} 
       then we reduce the same sub-formula of 
       {\small $F_a$} or {\small $F_b$} in {\small $F'$}. Induction 
hypothesis. 
     \item If it takes place on a sub-formula 
       of {\small $F_c$} then we reduce the same sub-formula 
       of both occurrences of {\small $F_c$} in {\small $F'$}.  
Induction hypothesis. 
     \item If {\small $\gtrdot$} reduction 2 takes place on 
       {\small $F$} such that we have; 
       {\small $F[(F_a \wedge F_b) \gtrdot F_c] \leadsto 
       F_x[(F_a \gtrdot F_c) \wedge (F_b \gtrdot F_c)]$} where 
       {\small $F$} and {\small $F_x$} differ only by 
       the shown sub-formulas,\footnote{This note `where \dots' is assumed in the 
       remaining.} then do nothing on {\small $F'$}. And {\small $F_x = 
       F'$}. Vacuous thereafter.  
     \item If a reduction takes place on a sub-formula {\small $F_p$} of 
       {\small $F$} in which the shown sub-formula of 
       {\small $F$} occurs as a strict sub-formula 
       ({\small $F[(F_a \wedge F_b) \gtrdot F_c] 
       = F[F_p[(F_a \wedge F_b) \gtrdot F_c]]$}), then 
       we have {\small $F[F_p[(F_a \wedge F_b) \gtrdot F_c]] 
       \leadsto F_x[F_q[(F_a \wedge F_b) \gtrdot F_c]]$}. 
       But we have 
       {\small $F' = F'[F_p'[(F_a \gtrdot F_c) \wedge (F_b \gtrdot 
       F_c)]]$}. Therefore we apply the same reduction on 
       {\small $F_p'$} to gain; 
       {\small $F'[F_p'[(F_a \gtrdot F_c) \wedge (F_b \gtrdot 
       F_c)]] \leadsto F'_x[F_{q}'[(F_a \gtrdot F_c) \wedge (F_b 
       \gtrdot F_c)]]$}. Induction hypothesis. 
   \end{enumerate} 
 \item The second pair: Similar. 
 \item The third pair: Similar, except when  
     {\small $\neg$} reduction 4 applies such that we have; 
     {\small $F[\neg (s \gtrdot F_b \wedge F_c)] \leadsto 
         F_p[s^c \vee (s \gtrdot \neg (F_b \wedge F_c))]$}.  
     By the simultaneous induction 
     and by Proposition \ref{preliminary_observation}, 
     it does not cost generality 
     if we replace it with 
     {\small $F_q[s^c \vee (s \gtrdot \neg F_b 
         \vee \neg F_c)]$} 
         that 
     differs from {\small $F_p$} only 
     by the shown sub-formulas, which 
     we then replace with 
     {\small $F_r[(s^c \vee s^c) \vee ((s \gtrdot \neg F_b) 
         \vee (s \gtrdot \neg F_c))]$}. 
     Since {\small $\lformulasize(0,0,\textsf{true},F_r) < \lformulasize(0,0,\textsf{true},F)$},   
     we again replace it with 
     {\small $F_u[s^c \vee (s^c \vee ((s \gtrdot \neg F_b) 
         \vee (s \gtrdot \neg F_c)))]$}, 
     and so on and so forth, to eventually 
     arrive at {\small $F_v[(s^c \vee (s \gtrdot \neg F_b))
         \vee (s^c \vee (s \gtrdot \neg F_c))]$}, 
     without loss of generality. 
     Meanwhile, we can reduce {\small $F'$} as follows.
     {\small $F'[\neg ((s \gtrdot F_b) \wedge (s \gtrdot F_c))] \leadsto 
         F'_x[\neg (s \gtrdot F_b) \vee \neg (s \gtrdot F_c)]
         \leadsto F'_y[(s^c \vee (s \gtrdot \neg F_b)) 
         \vee (s^c \vee (s \gtrdot \neg F_c))]$}.   
     Induction hypothesis. The other cases 
     are straightforward.  
 \item The fourth pair: Similar. 
 \item The fifth pair: 
   \begin{enumerate}
     \item If a reduction takes place on a sub-formula 
       which neither is a sub-formula of the shown 
       sub-formula nor takes as its sub-formula the shown sub-formula, 
       then we reduce the same sub-formula in {\small $F'$}.  
Induction hypothesis. 
     \item If it takes place on a sub-formula of {\small $F_a$}, 
       {\small $F_b$} or {\small $F_c$}, then 
       we reduce the same sub-formula of all the occurrences
       of the shown {\small $F_a$}, {\small $F_b$} or {\small $F_c$} 
       in {\small $F'$}. Induction hypothesis. 
     \item If {\small $\gtrdot$} reduction 4 takes place on 
       {\small $F$} such that we have; 
       {\small $F[(F_a \gtrdot F_b) \gtrdot F_c]
       \leadsto 
       F_x[(F_a \gtrdot F_c) \wedge ((F_a \gtrdot F_b) \vee 
       (F_a \gtrdot F_b \gtrdot F_c))]$}, then do nothing on 
       {\small $F'$}. And {\small $F_x = F'$}. Vacuous thereafter. 
     \item If a reduction takes place on a sub-formula 
       {\small $F_p$} of {\small $F$} in which the shown 
       sub-formula of {\small $F$} occurs 
       as a strict sub-formula, then similar to the case 1) e).    
   \end{enumerate} 
   \item The sixth pair: Straightforward.  
   \item The seventh and the eighth pairs: Similar. 
   \item The ninth pair: Similar except when 
       either {\small $\gtrdot$} reduction 4 or 5 
       takes place, which we have already covered 
       (for the third pair). 
        \item The 10th pair: Mostly straightforward. 
            Suppose {\small $\gtrdot$} reduction 2 applies 
            such that we have; 
            {\small $F[F_a \wedge F_a \gtrdot F_b]
                \leadsto F_p[(F_a \gtrdot F_b) \wedge 
                (F_a \gtrdot F_b)]$}, 
            then because we have 
            {\small $F'[F_a \gtrdot F_b]$}, 
            we apply induction hypothesis for a conclusion.  
            Or, suppose that a reduction takes 
            place on a sub-formula of an occurrence 
            of {\small $F_a$} such that we have;
            {\small $F \leadsto 
                F_x[F_u \wedge F_a \gtrdot F_b]$}, then 
            by the simultaneous induction, 
            it does not cost generality 
            if we replace it with 
            {\small $F_y[F_u \wedge F_u \gtrdot F_b]$} 
            that differs from {\small $F_x$} 
            only by the shown sub-formulas. 
            Meanwhile, we apply the same 
            reduction rule on the occurrence of 
            {\small $F_a$} in {\small $F'$} such that 
            we have; {\small $F' \leadsto F'_y[F_u \gtrdot F_b]$}.
            Induction hypothesis.   
            Likewise for the others. 
        \item  The 11th pairs: Similar. 
        \item The 12th and the 13th pairs: Straightforward. 
        \item The 14th and the 15th pairs:  
            Cf. the approach for the third pair.
   \end{enumerate} 
   By the result of the above bisimulation, we now have  
   {\small $\mathcal{F}(F) = \mathcal{F}(F')$}. However,
   it takes only those 5 {\small $\gtrdot$} reductions 
   and 4 {\small $\neg$} reductions 
   to derive a formula in unit chain expansion; 
   hence we in fact have
   {\small $\mathcal{F}(F) = \mathcal{F}(F_x)$} for some 
   formula {\small $F_x$} in unit chain expansion. But 
   then by Theorem \ref{theorem_1}, there could be 
   only one value out of {\small $\{0,1\}$} assigned to {\small $[\intFrame \models F_x]$}, as required. 
\end{proof}  
  \begin{corollary}[Normalisation]
    Given a formula {\small $F \in \mathfrak{F}$}, 
    denote the set of formulas in unit chain expansion 
    that it can reduce into by {\small $\mathcal{F}_1$}. 
    Then it holds 
    for every valuation frame 
    either that {\small $[\intFrame \models F_a] = 1$} 
    for all {\small $F_a \in \mathcal{F}_1$} or else 
    that 
    {\small $[\intFrame \models F_a] = 0$} for all {\small $F_a \in 
    \mathcal{F}_1$}.
    \label{theorem_normalisation}
  \end{corollary} 
  By Theorem \ref{theorem_1} and 
  Corollary \ref{theorem_normalisation}, we may define 
  implication: {\small $F_1 \supset F_2$} to be an abbreviation 
  of {\small $\neg F_1 \vee F_2$} - {\it exactly 
      the same} - as in classical logic. 
 \section{Decidability}  
We show a decision procedure {\small $\oint$} for 
universal validity 
of some input formula {\small $F$}. 
   Also assume a terminology of `object level', which is defined inductively. 
Given {\small $F$} in unit chain expansion, (A) if
   {\small $s \in \mathcal{S}$} in {\small $F$} occurs 
   as a non-chain or 
   as a head of a unit chain, then 
   it is said to be at the 0-th object level. 
   (B) if it 
   occurs in a unit chain 
   as {\small $s_0 \gtrdot \dots \gtrdot s_k \gtrdot s$} or 
   as {\small $s_0 \gtrdot \dots \gtrdot s_k  \gtrdot 
       s \gtrdot ...$}
   for some {\small $k \in \mathbb{N}$} 
   and some {\small $s_0, \dots, s_k \in \mathcal{S}$}, 
   then it is said to be at the (k+1)-th object level. 
   Further, assume a function {\small $\textsf{toSeq}: 
       \mathbb{N} \rightarrow \mathcal{S}^*$} satisfying
   {\small $\textsf{toSeq}(0) = \epsilon$} 
   and {\small $\textsf{toSeq}(k+1) = \underbrace{\top.
           \dots.\top}_{k+1}$}. 
\begin{description}[leftmargin=0.3cm]
    \item[{\small $\oint(F: \mathfrak{F},   
            \textsf{object}\_\textsf{level} : \mathbb{N} 
            )$}]  \textbf{returning 
            either 0 or 1}\\
    $\backslash\backslash$ This pseudo-textsf uses 
    {\small $n, o:\mathbb{N}$}, {\small $F_a, 
        F_b:\mathfrak{F}$}. \\
    \textbf{L0: } Duplicate {\small $F$} and 
    assign the copy to {\small $F_a$}. 
    If {\small $F_a$} is not already
    in unit chain expansion, then reduce it into  
    a formula in unit chain expansion. \\ 
\textbf{L1: } {\small $F_b := \squash(F_a, \textsf{object}\_\textsf{level})$}. \\  
\textbf{L2: }  {\small $n := \textsf{COUNT}\_\textsf{DISTINCT}(F_b)$}.\\   
\textbf{L3$_0$: } For each {\small $\mathsf{I}: 
    \textsf{toSeq}(\textsf{object}\_\textsf{level}) \times \mathcal{S}$} distinct
for the {\small $n$} elements of {\small $\mathcal{S}$} 
at the given object level, 
    Do:  \\  
    \textbf{L3$_1$: } If {\small $\sat(F_b, 
        \mathsf{I})$}, 
    then go to \textbf{L5}.\\   
    \textbf{L3$_2$: } Else if no 
    unit chains occur in {\small $F_a$}, 
    go to \textbf{L3}$_5$.\\ 
    \textbf{L3$_3$: } 
    {\small $o := \oint(\textsf{REWRITE}(F_a, \mathsf{I}, 
        \textsf{object}\_\textsf{level}), 
        \textsf{object}\_\textsf{level} + 1)$}. \\ 
    \textbf{L3$_4$: } 
    If {\small $o = 0$}, go to \textbf{L5}. \\
    \textbf{L3$_5$: } End of For Loop. \\
    \textbf{L4: } return 1. $\backslash\backslash$ Yes.\\
    \textbf{L5: } return 0. $\backslash\backslash$ No.\\
\end{description}    
\begin{description}[leftmargin=0.3cm]
    \item[\squash({\small $F: \mathfrak{F}, \textsf{object}\_\textsf{level}: \mathbb{N}$}) returning 
    {\small $F': \mathfrak{F}$}]
    {\ }\\ 
    \textbf{L0}: {\small $F' := F$}.  \\
    \textbf{L1}: For every 
    {\small $s_0 \gtrdot s_1 \gtrdot \dots \gtrdot s_{k}$} 
    for some {\small $k \in \mathbb{N}$} greater than 
    or equal to \textsf{object}\_\textsf{level} 
    and 
    some {\small $s_0, s_1, \dots, s_{k} \in \mathcal{S}$} 
    occurring 
    in {\small $F'$}, replace it with 
    {\small $s_0 \gtrdot \dots \gtrdot s_{\textsf{object}\_
            \textsf{level}}$}. \\
    \textbf{L2}: return {\small $F'$}. 
\end{description} 
\begin{description}[leftmargin=0.3cm]
  \item[$\textsf{COUNT}\_\textsf{DISTINCT}(F : \mathfrak{F})$ returning  
    {\small $n : \mathbb{N}$}]  
    {\ } \\
    \textbf{L0}: 
    return {\small $n:=$} (number of distinct members 
    of {\small $\mathcal{A}$} in {\small $F$}
     ). 
\end{description}
\begin{description}[leftmargin=0.3cm]
  \item[\sat({\small $F: \mathfrak{F}, \mathsf{I}: \mathsf{I}$}) returning 
      \textsf{true} or \textsf{false}]{\ }\\ 
      \textbf{L0}: return \textsf{true} if, 
    for the given interpretation {\small $\mathsf{I}$},\linebreak
    {\small $[(\mathsf{I}, \mathsf{J}) \models F] = 0$}. 
    Otherwise, return \textsf{false}. 
\end{description} 
\begin{description}[leftmargin=0.3cm]
    \item[\rewrite({\small $F: \mathfrak{F}, \mathsf{I}: \mathsf{I}, \textsf{object}\_\textsf{level}: \mathbb{N}$}) returning 
    {\small $F': \mathfrak{F}$}]{\ }\\    
    \textbf{L0}: {\small $F' := F$}.  \\
    \textbf{L1}: remove
    all the non-unit-chains and 
    unit chains shorter than or equal to 
    \textsf{object}\_\textsf{level} from {\small $F'$}. The 
    removal is in the following sense: 
      if {\small $f_x \wedge F_x$}, {\small $F_x \wedge f_x$}, 
	{\small $f_x \vee F_x$} or {\small $F_x \vee f_x$} 
	occurs as a sub-formula in {\small $F'$} for  
    {\small $f_x$} those just specified, then 
	replace them not simultaneously but one at a time 
	to {\small $F_x$} until no more reductions are possible. \\
    \textbf{L2$_0$}: For each unit chain 
    {\small $f$} in {\small $F'$}, Do:\\
    \textbf{L2$_1$}: if the head of {\small $f$} is 0 under 
    {\small $\mathsf{I}$}, then remove the unit chain 
    from {\small $F'$}; else replace the head of {\small $f$} with {\small $\top$}.  \\
    \textbf{L2$_2$}: End of For Loop.  \\
    \textbf{L3}: return {\small $F'$}. 
\end{description}  
The intuition of the procedure is found within the 
proof below. 
\begin{proposition}[Decidability of gradual classical logic] 
Complexity of\linebreak {\small $\oint(F, 0)$} is at most \exptime.  
\end{proposition}       
\begin{proof}     
   We show that it is a decision procedure.  
   That the complexity bound cannot be worse 
   than {\exptime} is clear from 
   the semantics (for \textbf{L0}) and from 
   the procedure itself. 
   Consider \textbf{L0} of the main procedure. 
   This reduces a given formula into 
   a formula in unit chain expansion.       
   In \textbf{L1} of the main procedure, 
   we get a snapshot of the input formula. We
   extract from it components of the 0-th object level, 
   and check if it is (un)satisfiable. The motivation  
   for this operation 
   is as follows: if the input formula is contradictory 
   at the 0th-object level, the input formula is 
   contradictory by the definition of {\small $\mathsf{J}$}. 
   Since we are considering validity of a formula, 
   we need to check all the possible valuation frames. 
   The number is determined by distinct {\small $\mathcal{A}$}
   elements. \textbf{L2} gets the number (n). 
   The For loop starting at \textbf{L3}$_0$ iterates through the {\small $2^n$} 
   distinct interpretations. If the snapshot 
   is unsatisfiable for any such valuation frame, 
   it cannot be valid, which in turn implies 
   that the input formula cannot be valid (\textbf{L3}$_1$). 
   If the snapshot is satisfiable and if 
   the maximum object-level in the input formula 
   is the 0th, \emph{i.e.} the snapshot is the input 
   formula, then the input formula 
   is satisfiable for this particular valuation frame, 
   and so we check the remaining valuation frames
   (\textbf{L3}$_2$). Otherwise, if 
   it is satisfiable and if the maximum object-level
   in the input formula is not the 0th, then 
   we need to check that snapshots in all the other 
   object-levels of the input formula are satisfiable 
   by all the valuation frames. We do this check 
   by recursion (\textbf{L3}$_3$). Notice the first
   parameter {\small $\textsf{REWRITE}(F_a, \mathsf{I},
       \textsf{object}\_\textsf{level})$} here. 
   This returns some formula {\small $F'$}. At the 
   beginning of the sub-procedure, {\small $F'$} is 
   a duplicated copy of {\small $F_a$} (not 
   {\small $F_b$}). Now, 
   under the particular 0-th object level interpretation 
   {\small $\mathsf{I}$}, some unit chain in 
   {\small $F_a$} may be 
   already evaluated to 0. Then 
   we do not need consider them 
   at any deeper object-level. 
   So we remove
   them from {\small $F'$}. Otherwise, 
   in all the remaining unit chains, the 0-th object 
   gets local interpretation of 1. So we replace 
   the {\small $\mathcal{S}$} element at the 0-th object 
   level with {\small $\top$} which always gets 1.\footnote{Such replacement does not preserve equivalence; equisatisfiability 
       is preserved, however.}
   Finally, all the non-chain 
   {\small $\mathcal{S}$} constituents 
   and all the chains shorter than or equal to 
   \textsf{object}\_\textsf{level} in {\small $F_a$} 
   are irrelevant at a higher object-level. So we 
   also remove them (from {\small $F'$}). We pass this 
   {\small $F'$} and an incremented \textsf{object}\_\textsf{level} to the main procedure for 
   the recursion. \\
   \indent The recursive process continues either until  
   a sub-formula passed to the main procedure 
   turns out to be invalid, in which case  
   the recursive call returns 0 
   (\textbf{L2}$_2$ and \textbf{L4} in the
   main procedure) 
   to the caller who assigns 0 to $o$ 
   (\textbf{L2}$_4$) and again returns 0, and so on 
   until the first recursive caller. 
   The caller receives 0 once again to conclude that 
   {\small $F$} is invalid, as expected. Otherwise, 
   we have that {\small $F$} is valid, for we 
   considered all the valuation frames. 
   The number of recursive calls cannot be infinite. 
\end{proof}   
\section{Analysis}       
In this section I will present an advanced 
observation about the principle of gradual logic. 
I will also highlight
an alternative 
interpretation of the object-attribute relation, which 
is hoped to cement the idea of gradual logic. 
A moderate comprehension of the gists
of Section 1, Section 2 and Section 3 is a pre-requisite 
for the first sub-section. The sub-section 5.2 assumes 
a full understanding of Section 3. 
\subsection{The notion of recognition cut-off} 
There are many that can be seen in 
the object-attribute relation. 
According to Postulate 1,  all the (attributed)
objects, so long as they remain recognisable as an object, 
have an extension (recall the relation between 
{\small $\textsf{Hat}$} and {\small $\textsf{Hat} \gtrdot \top$}), 
{\it i.e.} they are not atomic; and because 
an attribute is also an object, the implication 
is that no matter how deep the ladder of attributed 
objects formed in a sequence of {\small $\gtrdot$} goes, 
there is no possibility that we arrive at the most 
precise description of the object. Observed from 
the other side, it means that we can always refine 
any given (attributed) object with more 
attributes as we notice ambiguities in them. 
For illustration, 
if we have {\small $\textsf{Hat}$}, 
then 
{\small $\textsf{Hat} \gtrdot \textsf{Brooch}$}, 
{\small $\textsf{Hat} \gtrdot \textsf{Brooch} \gtrdot 
    \textsf{Green}$}, \dots, are, provided that 
they are not contradictory, guaranteed to be 
a specific instance of {\small $\textsf{Hat}$}.
    At the same time,
none of them is a fully explicated atomic object,  
since every one of them has extension. Descriptions 
in gradual classical logic reflect 
our intuition about concepts (Cf. Section 1) in this manner. \\
\indent One issue that must be touched upon, however, 
is that the arbitrary ambiguity 
has been the reason 
why natural languages are generally considered 
unsuited (Cf. \cite{Frege19,Tarski56}) for a rigorous 
treatment of concepts. 
As can be inferred from 
the analysis 
in Section 1 but also found already in Transcendental 
Logic, the truth in formal logic can be only this thing 
or that thing that we define as the truth. But, then, 
the arbitrary ambiguity in entities do appear to  
encumber the construction of the definition of this and that things, 
for we seem to be hitting upon an impossibility 
of knowing what they are. 
It is then reasonable to presuppose 
indivisible entities that act as the building stones  
of the truth, which is the measure typically taken 
in formal/symbolic logic. \\
\indent A rather different perspective about the ambiguity and atomicity of an entity was taken in this work, which 
is embodied in the following reasoning. - 
If the arbitrary ambiguity 
in natural descriptions must be felt universally, 
it would not be possible for us to 
coherently speak on a topic (which naturally 
concerns concepts referring to 
objects), for whatever that comes into our mind cannot 
be fully disambiguated. As a matter of fact, 
however, we have little issue in drawing 
a comprehensible conclusion, be it agreeable 
to us or not, from a discussion with our acquaintances. 
Concerning it, 
it appears that what is at stake is not the 
coaxing paradox: although nothing 
    should be comprehensible, we have 
    nonetheless comprehended the quintessence of something; but, 
    rather, 
our recognition of a divisible entity as an indivisible one. 
As one descriptive -  albeit rather impromptu  - example, 
suppose we have an apparatus, which has 
almost no utility, save that 
it could tell if a book is in 210 mm $\times$ 297 mm. 
The actual judgement mechanism of the apparatus  
is concealed from 
the eyes of the users. 
Now, let us say that we have applied it to several books,
as the result of which 
some of them have turned out to be in 
210 mm $\times$ 297 mm. 
But then it could happen that a book is 
in 209.993 mm $\times$ 297.001 mm,
another in 210.01 mm $\times$ 297.002 mm, 
if we are to measure the size more accurately 
by another method. But the point is that, 
so long as the measure of judging 
the size is as facilitated by the apparatus, 
we see a given book either in 210 mm $\times$ 297 mm, or
otherwise. We will not  know of the variance in 
210 mm $\times$ 297 mm unless the measure 
itself is changed. Here, a measure 
defining a threshold
and adopting which we become 
indifferent to all the remaining details outside it, is what may be called
a recognition cut-off, which in the above example 
was enforced by the apparatus. \\ 
\indent This principle of the recognition cut-off is 
prominently applied in concept-manoeuvring in general, where 
the apparatus is nothing but a state of our mind as conditioned 
such as by knowledge  and pre-suppositions. 
We understand things, and the things are understood. 
Yet, according to Postulate 1, 
it looks that they cannot be understood. 
Generally,  this does not 
indicate a misuse on our part of the term: to understand. 
What it does indicate, on the contrary, is the susceptibility of the existence of the things 
to our perception and cognition which define recognition 
cut-offs, by the merit of which, if 
for instance `I understand things' is given, 
it finds a cogent interpretation that 
`I' referring to I  under a recognition cut-off  
(Cf. Transcendental Logic also for Kant's observation 
about `I')
understands `things' referring to the things under 
a recognition cut-off, so that it becomes indeed 
possible for one to understand things without 
him/her, in supporting the very possibility, 
being forced
to accede that he/she is erring.  \\
\indent With the notion of the recognition cut-off, we can 
at last give a satisfactory justice 
to the subject of the arbitrary ambiguity and atomicity 
in gradual classical logic. 
In the expression 
{\small $\textsf{Hat} \gtrdot \textsf{Brooch} 
    \gtrdot \textsf{Green} \gtrdot 
    \textsf{Lamination}$}, with the  intended 
reading of the existence of ``hat ornamented with a brooch in laminated 
green'', {\small $\textsf{Hat}$}, at least some 
form of it, must exist before 
anything, for it would be absurd to state that it be possible 
to reason about 
the attributes in the absence of an entity to which they 
are allegedly an attribute. 
At the moment of the judgement of the existence, the judgement 
measure cannot favour one specific hat to other hats, 
so long as it is hat,
     which 
therefore exerts its influence only over what is found 
in the specific domain of discourse in which 
{\small $\textsf{Hat}$} is found. Even if 
{\small $\textsf{Hat}$} is non-atomic, 
it is still judged as if atomically under the judgement 
setting forth a recognition cut-off. 
It is only 
when deeper attributes are cogitated 
that it comes to light that 
it was not atomic. Meanwhile the attributes involving 
{\small $\textsf{Brooch}$} 
are again judged as if atomically under the judgement 
measure reigning over the domain of discourse conditioned 
(at least) by the existence of {\small $\textsf{Hat}$}. 
In this manner, gradual classical logic materialises 
the observation that we cannot tell apart 
whether a so-regarded atomic entity is atomic or is just 
atomic enough not to be considered non-atomic. \\
\hide{\indent The notion of the recognition cut-off, along with 
the above-stressed interplay of extension and intension, 
in a way addresses the issue of co-predication, which 
is mentioned for example in \cite{ }. As was stated in 
the work, what for instance makes it difficult 
to model the natural phenomena is the fact that  
entities in formal logic have come with certain fixed 
forms (or type if spoken in type theoretical terms): 
an entity if atomic is atomic. An entity if a predicate 
is a predicate. As much as it compactly models 
features of a natural language, it also omits 
an essential feature of a natural language where AND 
CONTINUES.} 
\subsubsection{Homonyms under recognition cut-off} 
More often than not, literature stresses that there is something particularly interesting 
about homonyms since, unless a sufficient 
    context is given, what they denote 
    cannot be determined. 
    By virtue of the recognition cut-off, however, 
they are almost as ambiguous a description as any other 
common descriptions, since a description 
involving concepts referring to objects, according 
to Postulate 1, possesses 
the same degree of ambiguity as a homonym does - 
the same degree insofar as they 
are arbitrarily ambiguous. 
A `book' identifies that what the concept 
points to shall be a book, but nothing more  
can be asserted. 
But this then allows us to infer that `bow' with little to no contexts 
provided still identifies that what 
the concept refers to shall be bow. The following
criticism is amply expected at this point: such an answer, 
for the reason that it by no means addresses the fundamental 
problem that `bow' with no context does not determine 
which (definition of) `bow' it is, is absurd. 
However, if it were absurd, then in order to 
avoid the same absurdity it must be explicated what 
bow, supposing that enough context has been given 
to identify it as a violin bow, it is. An answer would 
burgeon the criticism of the same kind, and we would 
never get out of the cycle. If we are permitted 
to fluctuate the point of the recognition cut-off freely, 
it holds that 
a word with a context no more determines what it is than 
that with no context does. One description concerning 
concept(s) referring to objects is only comparatively
less ambiguous than others. 
\subsection{Another interpretations of the object-attribute relation} 
Semantic interpretation of {\small $\gtrdot$}  
is not restricted to the one that we saw in Section 2, 
which was formalised in Section 3. Just like 
in modal logic, there are other interpretations that 
could have a linguistic meaningfulness.  
In one variant, we may remove 
the synchronization condition on {\small $\mathsf{I}$} 
interpretation. The  motivation 
is that, suppose {\small $X \gtrdot 
    Y$}, 
it may be that we like to say that the attribute {\small $Y$} 
varies according to what it is an attribute to: 
{\small $X$} in this case. Then, 
if we have {\small $\textsf{Hat} \gtrdot \textsf{Green}$} 
and {\small $\textsf{Brooch} \gtrdot \textsf{Green}$}, 
we do not know if the same greenness is talked about 
for {\small $\textsf{Hat}$} and {\small $\textsf{Brooch}$}.  
In such an interpretation, 
we do not have the following distributivity: 
{\small $\textsf{Hat} \wedge \textsf{Brooch}
    \gtrdot \textsf{Green} 
    \mapsto (\textsf{Hat} \gtrdot \textsf{Green})
    \wedge (\textsf{Brooch} \gtrdot \textsf{Green})$}. 
Another distributivity of the sort: 
{\small $(\textsf{Hat} \vee \textsf{Brooch}) 
    \gtrdot \textsf{Green} 
    \mapsto (\textsf{Hat} \gtrdot \textsf{Green})
    \vee (\textsf{Brooch} \gtrdot \textsf{Green})$} 
would also need altered to: 
{\small $(\textsf{Hat} \gtrdot \textsf{Green}) 
    \vee (\textsf{Brooch} \gtrdot \textsf{Green}) 
    \vee (\textsf{Hat} \wedge \textsf{Brooch} 
    \gtrdot \textsf{Green})$}, covering each
possibility of the existence of the objects. 
On the other hand, we may or may not have the rule of the 
sort: 
{\small $(\textsf{Hat} \gtrdot \textsf{Green}) 
    \gtrdot \textsf{Brooch} 
    \mapsto 
    (\textsf{Hat} \gtrdot \textsf{Brooch}) 
    \wedge ((\textsf{Hat} \gtrdot \textsf{Green}) 
    \vee (\textsf{Hat} \gtrdot \textsf{Green} \gtrdot 
    \textsf{Brooch}))$}. In our demonstration, we choose not to 
include this rule 
for brevity. We omit {\small $\top$} 
and {\small $\bot$}, for {\small $\top$} 
does not behave well in 
{\small $(\top \gtrdot F_1) \wedge (\top \gtrdot F_1)$} 
under the specified interpretation.\footnote{It is also 
    recommendable that the number 
    of elements of each domain of discourse be 
    at least countably infinite in this interpretation.} 
    Let us 
formalise this logic, beginning with
peripheral definitions. 
\begin{definition}[Unit graph chain/unit graph expansion]  
    Given any {\small $F$} (no\linebreak occurrences 
    of {\small $\top$} and {\small $\bot$}), 
    we say that {\small $F$} is a unit graph chain 
    if and only if it is recognised 
    in the following rules. 
    \begin{itemize}  
        \item A unit chain is a unit graph chain.  
\item If {\small $F_1$} is a unit graph chain and 
    {\small $a$} is a literal, 
             then {\small $F_1 \gtrdot a$} is a unit 
             graph chain.
         \item If {\small $a$} is a literal, and {\small $F_1$} and 
            {\small $F_2$} are either a literal
             or a unit graph chain, then 
             {\small $F_1 \wedge F_2 \gtrdot a$} 
             is a unit graph chain.
    \end{itemize} 
    We say that a given formula is in
    unit graph expansion if and only if 
    all the chains that occur 
    in the formula are a unit graph chain. 
    
\end{definition} 
\noindent By {\small $G$} with or without a sub-/super-script 
we denote a formula that is either 
a literal or a unit graph chain. \\
\indent The semantics is as follows. 
Assume that {\small $\mathcal{I}(k)$} 
for {\small $k \in \mathbb{N}$} is 
the power set of 
{\small $\{0, 1, \ldots, k\}$} minus the empty set. 
\begin{itemize} 
    \item {\small $\forall a \in \mathcal{A}. 
            \neg a \mapsto a^c$} 
        ({\small $\neg$} reduction 1). 
    \item {\small $\neg (F_1 \wedge F_2) 
            \mapsto \neg F_1 \vee \neg F_2$} ({\small $\neg$} 
        reduction 2).  
    \item {\small $\neg (F_1 \vee F_2) 
            \mapsto \neg F_1 \wedge \neg F_2$} ({\small $\neg$} 
        reduction 3).   
    \item {\small $\neg (G_0 \wedge \dots 
            \wedge G_k \gtrdot F_2)  
            \mapsto \neg G_0 \vee \dots 
            \vee \neg G_k \vee (G_0 \wedge \dots 
            \wedge G_k \gtrdot \neg F_2)$} 
        ({\small $\neg$} reduction 4). 
    \item {\small $G_0 \vee G_1 \vee \dots 
            \vee G_{k+1} 
            \gtrdot F 
            \mapsto 
            \bigvee_{I \in \mathcal{I}(k+1)} (\bigwedge_{j \in I} G_{j}
            \gtrdot F)$} 
        ({\small $\gtrdot$} reduction 3).\footnote{This should not be confused with {\small $\bigvee_{I \in \mathcal{I}(k+1)} (\bigwedge_{j \in I} (G_{j}
                \gtrdot F))$}.}
    \item {\small $F_1 \gtrdot F_2 \wedge F_3 
            \mapsto (F_1 \gtrdot F_2) 
            \wedge (F_1 \gtrdot F_3)$} 
        ({\small $\gtrdot$}  reduction 4). 
    \item {\small $F_1 \gtrdot F_2 \vee F_3
            \mapsto (F_1 \gtrdot F_2) 
            \vee (F_1 \gtrdot F_3)$} 
        ({\small $\gtrdot$} reduction 5).  
    \item {\small $F_1 \wedge (F_2 \vee F_3) \gtrdot F_4 
            \mapsto (F_1 \wedge F_2) \vee (F_1 
            \wedge F_3) \gtrdot F_4$} 
        (obj distribution 1). 
    \item {\small $(F_1 \vee F_2) \wedge F_3 \gtrdot F_4 
            \mapsto (F_1 \wedge F_3) \vee 
            (F_2 \wedge F_3) \gtrdot F_4$} 
        (obj distribution 2).  
\end{itemize}      
We assume that the {\small $\gtrdot$} reduction 3 applies 
to any {\small $\vee$}-connected unit graph chains with 
no regard to a particular association among the unit graph chains: it applies just 
as likely to 
{\small $(G_0 \vee G_1) \vee (G_2 \vee G_3) \gtrdot F$} 
as to {\small $G_0 \vee (G_2 \vee (G_3 \vee G_1)) \gtrdot F$}. 
Similarly for the {\small $\wedge$}-connected
formulas in {\small $\neg$} reduction 4. 
\begin{proposition}[Reduction of induction measure]  
    \label{strict_reduction_induction} 
    Let induction measure be the formula size (the main 
    induction), the inverse of (the number of 
    {\small $\neg + 1$}) (a sub-induction), 
    the inverse of (the number of 
    {\small $\gtrdot + 1$}) (a sub-sub-induction), 
    and the inverse of (the number of 
    {\small $\wedge + 1$}) (a sub-sub-sub-induction). 
    Then the induction measure strictly decreases 
    at each reduction on a given formula. 
   Additionally, associativity and commutativity 
   of {\small $\wedge$} and {\small $\vee$} 
   do not alter the induction measure.
   \label{reduction_size2} 
\end{proposition} 
\begin{proof} 
    Checked with a Java program, whose 
    source code is as found in Appendix A. The 
    test cases are found in Appendix B. 
\end{proof} 
\begin{definition}[Domains and valuations]
    Let {\small $\mathcal{T}$} denote 
    a non-empty set that has all the 
    elements that match 
    the following
    inductive rules.  
        \begin{itemize}  
        \item {\small $\unorders{a}$} for {\small $a \in \mathcal{A}$} 
            is an element of {\small $\mathcal{T}$}.  
        \item {\small $\seqs{a}$} for {\small $a \in \mathcal{A}$} 
            is an element of {\small $\mathcal{T}$}. 
        \item if {\small $t$} is an element of 
            {\small $\mathcal{T}$}, 
            then so are both {\small $\unorders{t}$} and 
            {\small $\seqs{t}$}. 
        \item if {\small $t_1, t_2$} are elements of 
            {\small $\mathcal{T}$}, 
            then so are {\small $\unorders{t_1, t_2}$} 
            and {\small $\seqs{t_1}.\seqs{t_2}$}. 
    \end{itemize}  
    We assume that {\small $\unorders{\ }$} defines 
    an unordered set: {\small $\unorders{t_1, t_2} = \unorders{t_2, t_1}$}; {\small $\unorders{t, t} = \unorders{t}$}. 
We also assume the following congruence relations
among the elements: {\small $\unorders{\unorders{t}} \doteq  
    \unorders{t}$}, {\small $\unorders{t_1, t_2} 
    \doteq \unorders{t_2, t_1}$} and 
{\small $\seqs{\seqs{t}} \doteq \seqs{t}$}. 
Then by {\small $\dot{\mathcal{T}}$} we denote 
a sub-set of {\small $\mathcal{T}$} 
which contains only the least elements in each 
congruence class.\footnote{Here, an element in one 
    congruence class is smaller than another 
    if it contains a fewer number of symbols.} 
    Now, let {\small $\mathcal{T}^*$} denote 
    the set of all the 
    finite sequences of elements of {\small $\dot{\mathcal{T}}$}, 
    {\it e.g.} {\small $t_0.t_1.\ldots.t_k$} for 
    some {\small $k \in \mathbb{N}$}, plus 
    an empty sequence which we denote by 
    {\small $\{\epsilon\}$}. An element of 
    {\small $\mathcal{T}^*$} is referred to 
    by {\small $t^*$} with or without a sub-script. 
    Then, we define a domain function 
    {\small $D: 
        \mathcal{T}^* \rightarrow 2^{\mathcal{A}} \backslash 
        \emptyset$}, and 
    a valuation frame 
    as a 2-tuple: {\small $(\mathsf{I}, \mathsf{J})$}, 
    where {\small $\mathsf{I}: \mathcal{T}^* \times 
        \mathcal{A} \rightarrow \{0,1 \}$} is what 
    we call local interpretation and 
    {\small $\mathsf{J}:\mathcal{T}^*\backslash \{\epsilon\}
        \rightarrow \{0,1\}$} is what we call 
    global interpretation. The following are 
    defined to satisfy for all {\small $k \in \mathbb{N}$}, 
    for all {\small $t^* \in \mathcal{T}^*$}, 
    and for all {\small $t_0, \ldots, t_k \in \dot{\mathcal{T}}$}. 
   \begin{description}  
           \item[Regarding domains of discourse]{\ } 
               \begin{itemize}[leftmargin=-0.3cm] 
                   \item For 
                       all {\small $t^* \in \mathcal{T}^*$},
                       {\small $D(t^*)$} is closed 
                       under complementation and is 
                       non-empty. 
                \end{itemize}
           \item[Regarding local interpretations]{\ }
      \begin{itemize}[leftmargin=-0.3cm] 
      \item {\small $\forall a \in  
              D(t^*). 
              [\mathsf{I}(t^*, a)
              = 0] \vee^{\dagger} 
              [\mathsf{I}(t^*, a)
= 1]$}. 
\item {\small $\forall a \in  
              D(t^*). 
[\mathsf{I}(t^*,a) = 0]
        \leftrightarrow^{\dagger} 
        [\mathsf{I}(t^*,a^c) = 1]$}.
 \end{itemize}
 \item[Regarding global interpretations]{\ }  
     \begin{itemize}[leftmargin=-0.3cm]
         \item {\small $\mathsf{J}(\seqs{t_0}.\seqs{t_1}.\ldots.\seqs{t_{k}}) 
                 =  \bigwedge_{i = 0}^{k}
                 \mathsf{J}(\seqs{t_0}.\ldots .\seqs{t_{i-1}}.t_i)$}. 
         \item {\small $\mathsf{J}(t^*.\unorders{t_0, t_1, \ldots, 
                     t_k}) = \bigwedge_{i=0}^k\mathsf{J}(t^*.t_i)$}. 
         \item {\small $\forall a \in D(t^*).\mathsf{J}(t^*.a) = 
                 \mathsf{I}(t^*, a)$}. 
  \end{itemize} 
  \end{description}  
  \label{second_interpretations}
\end{definition}   
To briefly explain the {\small $\mathcal{T}^*$}, 
it provides a semantic mapping for 
every formula in unit graph expansion. Compared to 
the corresponding definition of domain functions 
and valuation frames back in Section 3, here 
the domain function cannot be determined  
by a sequence of literals. What was then 
a literal must be generalised to possibly conjunctively
connected unit graph chains and literals. Note the implicit 
presumption of 
the associativity and commutativity of the classical 
{\small $\wedge$} in 
the definition of {\small $\dot{\mathcal{T}}$}. 
\begin{definition}[Valuation] 
    Suppose a valuation frame {\small $\tintFrame
        = (\mathsf{I}, \mathsf{J})$}. The following
 are defined to hold. 
\begin{itemize} 
    \item {\small $[\tintFrame \models 
            G] = 
            \mathsf{J}(
            \textsf{compress} \circ \textsf{map}(G))$}. 
    \item {\small $[\tintFrame \models
            F_1 \wedge F_2] = 
            [\mathfrak{M} \models
            F_1] \andMeta [\mathfrak{M} \models
            F_2]$}. 
    \item {\small $[\tintFrame \models
            F_1 \vee F_2] = 
            [\mathfrak{M} \models
            F_1] \orMeta  [\mathfrak{M} 
            \models
            F_2]$}.  
\end{itemize} 
where  {\small $\textsf{map}$} 
is defined by: 
\begin{itemize}  
    \item {\small $\textsf{map}(G_0 \wedge G_1) = \unorders{\textsf{map}(
            G_0), \textsf{map}(
            G_1)}$}. 
\item {\small $\textsf{map}(G_0 \gtrdot G_1) = \seqs{\textsf{map}(
        G_0)}.\seqs{\textsf{map}(
        G_1)}$}.
\item {\small $\textsf{map}(a) = 
        \seqs{a}$}. 
\end{itemize}     
and {\small $\textsf{compress}$}, 
given an input, 
returns the least element in the same congruence 
class as the input. 
\end{definition}    
\begin{definition}[Validity and satisfiability] 
    A formula {\small $F$} with no occurrences 
    of {\small $\top$} and {\small $\bot$} is 
    said to be satisfiable in 
    a valuation frame {\small $\tintFrame$} 
    iff {\small $[\tintFrame \models 
        F] = 1$}; it is said to be valid 
    iff it is satisfiable in all the valuation frames; 
    it is said to be invalid iff it is not valid; 
    and it is said to be unsatisfiable iff it is 
    not satisfiable. 
\end{definition} 
\noindent I state the main results. Many details 
will be omitted, the proof approaches being similar 
to those that we saw in Section 3. 
\begin{definition}[Procedure \recurseReducet]  
   The procedure given below takes as an input 
   a formula {\small $F$} in unit graph 
   expansion. \\
\textbf{Description of {\small $\recurseReducet(F)$}}
\begin{enumerate}[leftmargin=0.5cm] 
    \item Replace {\small $\wedge$} and {\small $\vee$} in 
        {\small $F$} 
        which is not in a chain  with 
        {\small $\vee$} and respectively 
        with {\small $\wedge$}. 
        These 
        two operations are simultaneous. 
    \item Replace all the non-chains {\small $a \in 
            \mathcal{A}$} in {\small  $F$} 
        simultaneously with {\small $a^c$}. 
    \item For every chain {\small $F_a$} in 
        {\small $F$} which is not a strict 
        a strict sub-chain of another chain, 
        with its head
           {\small $F_h$}
        and its tail 
        {\small $F_{t}$}, 
        replace {\small $F_a$} with\linebreak
        {\small $\recurseReducet(F_h) \vee  
            (F_h \gtrdot \recurseReducet(F_{t}))$}. 
    \item Reduce {\small $F$} via {\small $\gtrdot$} reductions 
        4 and 5 in unit graph expansion. 
  \end{enumerate}
\end{definition}  
\begin{proposition}[Reduction of negated 
    unit graph expansion]  
    \label{special_reduction_2} 
  Let {\small $F$} be a formula in unit graph 
  expansion. Then {\small $\neg F$} 
  reduces via the {\small $\neg$} and 
  {\small $\gtrdot$} reductions into 
  {\small $\recurseReducet(F)$} which is 
  in unit graph expansion. The reduction 
  is unique. 
\end{proposition}   
\begin{lemma}[Elementary complementation] 
  For any {\small $G_0 \gtrdot G_1 \gtrdot \dots 
      \gtrdot G_k$} 
  for {\small $k \in \mathbb{N}$} and 
  {\small $G_k \in \mathcal{A}$},\footnote{
      {\small $G_k$} is always
       an element of {\small $\mathcal{A}$} by the definition 
       of a unit graph chain.}
  if for a given valuation frame it holds that 
  {\small $[\tintFrame \models G_0 \gtrdot G_1 
      \gtrdot \dots \gtrdot G_k] = 1$}, 
  then it also holds that {\small 
      $[\tintFrame \models \recurseReducet(G_0 \gtrdot G_1 
      \gtrdot \dots \gtrdot G_k)] = 0$}; 
  or if it holds that  {\small $[\tintFrame \models \recurseReducet(G_0 \gtrdot G_1 
      \gtrdot \dots \gtrdot G_k)] = 1$}, 
  then it holds that {\small $[\tintFrame \models G_0 \gtrdot G_1 
      \gtrdot \dots \gtrdot G_k] = 0$}. These 
  two events are mutually exclusive.  
  \label{unit_chain_excluded_middle_2}
\end{lemma}  
\begin{proof}     
    Let us abbreviate {\small $\recurseReducet$} by 
    {\small $\abbR$}. What we need to show for the first 
    obligation is 
    {\small $[\tintFrame \models 
        \abbR(G_0)] =  
        [\tintFrame \models G_0 \gtrdot \abbR(G_1)]
     = \dots = 
     [\tintFrame \models G_0 \gtrdot \dots \gtrdot 
     \abbR(G_k)] = 0$}. 
    The reasoning process is recursive on 
    each {\small $G_i$}, {\small $0 \le i \le k$} 
    within {\small $\abbR$}.    
    Since the formula size stays finite 
    and since each reduction incurs 
    finite branching, 
    there is an end to each recursion. In the end, 
    we will be showing that 
    {\small $[\tintFrame \models G'_0 \gtrdot 
        \dots G'_{j-1} \gtrdot \abbR(G'_j)] = 0$} 
    for  {\small $j \in \mathbb{N}$} and {\small $G'_j 
        \in \mathcal{A}$}, 
    whenever the pattern is encountered during the recursion. 
    For each such pattern, 
    we will have that {\small $[\tintFrame \models G'_0 \gtrdot 
        \dots G'_{j-1} \gtrdot G'_j] = 1$} (the co-induction 
    is left to readers; 
    note the property of a formula in unit graph expansion). 
    Then the result follows. 
\end{proof}  
\begin{lemma}[Elementary double negation] 
    Let {\small $G$} denote 
    {\small $G_0 \gtrdot G_1 \gtrdot \dots \gtrdot G_k$} 
    for {\small $k \in \mathbb{N}$} and 
    {\small $G_k \in \mathcal{A}$}. Then for 
    any valuation frame it holds that 
    {\small $[\tintFrame \models G] = 
        [\tintFrame \models \recurseReducet(\recurseReducet(G))]$}. 
    \label{double_negation_2} 
\end{lemma} 
\begin{proof}  
    Let us use an abbreviation 
    {\small $\abbR$} for 
    {\small $\recurseReducet$} for space. \\
    {\small $\abbR(\abbR(G))
        = \abbR(\abbR(G_0) \vee 
        (G_0 \gtrdot \abbR(G_1)) 
        \vee \dots \vee (G_0 \gtrdot G_1 \gtrdot \dots \gtrdot 
        G_{k-1} \gtrdot \abbR(G_k))) 
        = \abbR(\abbR(G_0)) \wedge 
        \abbR(G_0 \gtrdot \abbR(G_1)) 
        \wedge \dots \wedge \abbR(G_0 \gtrdot G_1 \gtrdot 
        \dots \gtrdot G_{k-1} \gtrdot \abbR(G_k))$}.   
    Since translation to disjunctive normal form 
    is tedius, let us solve the problem 
    directly here. The strategy is that we first show
    {\small $[\tintFrame \models \abbR(\abbR(G_0))] = 
        [\tintFrame \models G_0]$}, which 
    reduces (via Lemma \ref{unit_chain_excluded_middle_2}) 
    the right hand side of 
    the equation into 
    {\small $[\tintFrame \models G_0 \wedge (G_0 \gtrdot 
        \abbR(\abbR(G_1))) \wedge 
        (G_0 \gtrdot \abbR(G_1 \gtrdot \abbR(G_2))) 
        \wedge \dots \wedge 
        (G_0 \gtrdot \abbR(G_1 \gtrdot \dots \gtrdot 
        G_{k-1} \gtrdot \abbR(G_k)))] = [\tintFrame \models (G_0 \gtrdot 
        \abbR(\abbR(G_1))) \wedge 
        (G_0 \gtrdot \abbR(G_1 \gtrdot \abbR(G_2))) 
        \wedge \dots \wedge 
        (G_0 \gtrdot \abbR(G_1 \gtrdot \dots \gtrdot 
        G_{k-1} \gtrdot \abbR(G_k)))]$}; 
    we then show {\small $[\tintFrame \models 
        G_0 \gtrdot \abbR(\abbR(G_1))] = [\tintFrame \models 
        G_0 \gtrdot G_1]$} to reduce again; 
    and so on and so forth. In the end, we arrive 
    at the required result. 
    Therefore it suffices to show that 
    {\small $[\tintFrame \models 
        G_0 \gtrdot \dots \gtrdot G_{i-1} 
        \gtrdot \abbR(\abbR(G_i))] = [\tintFrame \models 
        G_0 \gtrdot \dots \gtrdot G_{i-1} 
        \gtrdot G_i], 0 \le i \le k$}.   
    But each {\small $G_i$} in {\small $\abbR(\abbR(G_i))$} 
    is strictly smaller in the number of 
    symbols appearing within 
    than {\small $G$}. So the reasoning is recursive. Because 
    every formula is of a finite size and 
    the reduction rules induce only finite branchings, 
    it follows that every recursion is also finite, 
    reaching at the obligation pattern of 
    {\small $[\tintFrame \models 
        G'_0 \gtrdot \dots \gtrdot \abbR(\abbR(G'_j))]
        = [\tintFrame \models  G'_0 \gtrdot \dots \gtrdot G'_j]$} 
    for {\small $G'_j \in \mathcal{A}$}. 
    But these equations hold by the way the local/global 
    interpretations 
    are defined. 
\end{proof} 
\begin{theorem} 
    \label{boolean_algebra_2} 
  Denote by {\small $X$} the set of 
  the expressions comprising 
  all {\small $[\tintFrame \models G]$} 
  for a formula {\small $G$} in unit graph expansion.  
  Then for every valuation frame, 
  {\small $(X, \recurseReducet, \dot{\top}, 
      \dot{\bot}, \andMeta, \orMeta)$} 
  with suppositional
  nullary connectives: {\small $\dot{\top}$} 
  and {\small $\dot{\bot}$} 
defines a Boolean algebra. 
  \end{theorem}    
\begin{proof}   
  It suffices to show anihilation, identity, associativity, 
  commutativity, distributivity, idempotence, absorption, 
  complementation and double negation. Straightforward with 
  Lemma \ref{unit_chain_excluded_middle_2}, 
  Lemma \ref{double_negation_2}, and by following 
  the approaches taken in Section 3. 
\end{proof}  
The insertion of the suppositional connectives 
into the theorem is inessential, since we 
could take it for granted that we are considering {\small $\mathfrak{F}$} 
minus {\small $\top$} and {\small $\bot$} 
plus {\small $\dot{\top}$} and {\small $\dot{\bot}$}, 
except that we never make use of
{\small $\dot{\top}$} or {\small $\dot{\bot}$} 
in an expression. For the following results, 
let us enforce that {\small $\mathcal{G}(F)$} denote 
the set of formulas in unit graph expansion 
that {\small $F$} without the occurrences of {\small $\top$} 
and {\small $\bot$} reduce into. 
\begin{theorem}[Bisimulation]  
    \label{bisimulation2}
  Assumed below are pairs of formulas 
  in which {\small $\top$} 
  and {\small $\bot$} do not occur. 
  {\small $F'$} differs from {\small $F$} only by
  the shown sub-formulas, \emph{i.e.} {\small $F'$} 
  derives from {\small $F$} by replacing the shown sub-formula 
  for {\small $F'$} 
  with the shown sub-formula for {\small $F$} and vice versa. 
    Then for each pair  
  {\small $(F, F')$} 
  below, it holds for every valuation frame 
  that 
  {\small $[\intFrame \models F_1] = [\intFrame \models F_2]$} for 
  all {\small $F_1 \in \mathcal{G}(F)$} and 
  for all {\small $F_2 \in \mathcal{G}(F')$}. 
  {\small 
  \begin{eqnarray}\nonumber 
      F[G_0 \vee \dots \vee G_{k+1} \gtrdot F_c] 
    &,& F'[\bigvee_{I \in \mathcal{I}(k+1)}(\bigwedge_{j \in I}
    G_j \gtrdot F_c)]\\\nonumber 
    F[F_a \gtrdot F_b \wedge F_c] &,& F'[(F_a \gtrdot F_b)
    \wedge (F_a \gtrdot F_c)]\\\nonumber 
    F[F_a \gtrdot F_b \vee F_c] &,& F'[(F_a \gtrdot F_b)
    \vee (F_a \gtrdot F_c)]\\\nonumber  
    F[F_a \wedge (F_b \vee F_c) \gtrdot F_d] &,&
    F'[(F_a \wedge F_b) \vee (F_a \wedge F_c) \gtrdot F_d 
    ]\\\nonumber 
    F[(F_a \vee F_b) \wedge F_c \gtrdot F_d] &,& 
    F'[(F_a \wedge F_c) \vee (F_b \wedge F_c) \gtrdot F_d]  
    \\\nonumber  
F[\neg G] 
      &,& F'[\recurseReducet(G)]\\\nonumber  
F[\neg (F_a \wedge F_b)] &,& F'[\neg F_a \vee \neg F_b]\\\nonumber
    F[\neg (F_a \vee F_b)] &,& F'[\neg F_a \wedge \neg F_b]\\\nonumber 
    F[\neg (G_0 \wedge \dots G_k \gtrdot F_a)] 
    &,& 
    F'[\neg G_0 \vee \dots \vee \neg G_k \vee 
    (G_0 \wedge \dots \wedge G_k \gtrdot \neg F_a)]\\\nonumber
    F[F_a \vee F_a] &,& F'[F_a]\\\nonumber
    F[F_a \wedge F_a] &,& F'[F_a]\\\nonumber
    F[F_a \vee F_b] &,& 
    F'[F_b \vee F_a]\\\nonumber
    F[F_a \vee (F_b \vee F_c)] 
    &,& 
    F'[(F_a \vee F_b) \vee F_c]\\\nonumber
    F[F_a \wedge F_b] &,& F'[F_b \wedge F_a]\\\nonumber
    F[F_a \wedge (F_b \wedge F_c)] &,& 
    F'[(F_a \wedge F_b) \wedge F_c]
  \end{eqnarray}  
  }
  \end{theorem} 
\begin{proof}  
    Similar in approach to the proof of 
    Theorem 2, the proof following by 
    simultaneous composite induction 
    by the induction measure 
    as in Proposition 
   \ref{strict_reduction_induction}, which strictly 
   decreases at each reduction. 
   We first establish that {\small $
       \mathcal{G}(F_1) = \mathcal{G}(F_2)$}. 
   One way, to show that to each 
   reduction on {\small $F'$} corresponds 
   reduction(s) on {\small $F$}, is straightforward, 
   for we can choose to 
   reduce {\small $F$} into {\small $F'$} in most of the 
   cases, 
   thereafter synchronizing both of the reductions. 
   We show one sub-proof for the 10th pair, however, to 
   be very safe. If {\small $\gtrdot$} reduction 3 
           applies ({\small $F_a = G$} for some unit graph 
           chain {\small $G$}), then we have; 
           {\small $F[G_0 \vee G_1 \vee \dots \vee 
               (G \vee G) \vee \dots \vee
               G_{k+1} \gtrdot F_a] 
               \mapsto 
               F_p[\bigwedge_{I \in \mathcal{I}(k+1)}(
               \bigvee_{j \in I} G_j \gtrdot F_a)]$}. 
           Since association by {\small $\vee$} 
           on the constituting unit graph chains 
           is freely chosen, 
           let us assume that 
{\small $\bigwedge_{I \in \mathcal{I}(k+1)}(
    \bigvee_{j \in I} G_j \gtrdot F_a)$} 
is ordered such that, 
from the left to the right, 
the number of the occurrences of {\small $G$} 
either stays the same, or else strictly increases. 
Then, generally speaking, there are three 
groups of sub-formulas: those that 
do not have the occurrences of the {\small $G$}; 
those in which the {\small $G$} occurs once; 
and those in which there are two occurrecnes of {\small $G$}. 
Now, for {\small $F'$}, applying the same reduction rule, 
we gain: 
           {\small $F'[G_0 \vee G_1 \vee \dots \vee 
               G\vee \dots \vee
               G_{k} \gtrdot F_a] 
               \mapsto 
               F'_p[\bigwedge_{I \in \mathcal{I}(k)}(
               \bigvee_{j \in I} G_j \gtrdot F_a)]$}. 
           It is straightforward to see that 
           {\small $F'_p$} involves 
           every constituent from the first group 
           of {\small $F_p$} (if the group 
           has any constituent at all);
          and half of the constituents from the second 
          group. So we can sequentially 
          apply induction hypothesis on {\small $F_p$} 
          to match up with {\small $F'_p$}. 
          Induction hypothesis. \\
          \indent Into the 
   other way to show that to each 
   reduction on {\small $F$} corresponds 
   reduction(s) on {\small $F'$}: 
   \begin{enumerate}
     \item The first pair:  
       \begin{enumerate} 
	 \item 
       If a reduction takes place on a sub-formula which 
       neither is a sub-formula of the shown sub-formula 
       nor takes as its sub-formula the shown sub-formula, 
       then we reduce the same sub-formula in {\small $F'$}.   
       Induction hypothesis (note that the number of 
reduction steps is that of {\small $F$} into this 
direction).  
     \item If it takes place on a sub-formula 
       of {\small $F_c$} then we reduce the same sub-formula 
       of both occurrences of {\small $F_c$} in {\small $F'$}.  
Induction hypothesis. 
     \item If {\small $\gtrdot$} reduction 3 takes place on 
       {\small $F$} such that we have; 
       {\small $F[G_0 \vee \dots \vee G_{k+1} \gtrdot F_c] \leadsto 
           F_x[\bigvee_{I \in \mathcal{I}(k+1)}(\bigwedge_{j \in 
               I} G_j \gtrdot F_c)]$}, 
       {\small $F$} and {\small $F_x$} differ only by 
       the shown sub-formulas. 
       Do nothing on {\small $F'$}, and {\small $F_x = 
       F'$}. Vacuous thereafter. 
     \item If a reduction takes place on a sub-formula {\small $F_p$} of 
       {\small $F$} in which the shown sub-formula of 
       {\small $F$} occurs as a strict sub-formula 
       ({\small $F[G_0 \vee \dots \vee G_{k+1} \gtrdot F_c] 
           = F[F_p[G_0 \vee \dots \vee G_{k+1}
           \gtrdot F_c]]$}), then 
       we have {\small $F[F_p[G_0 \vee \dots \vee G_{k+1} 
           \gtrdot F_c]] 
           \leadsto F_x[F_q[G_0 \vee \dots \vee G_{k+1} 
           \gtrdot F_c ]]$}. 
       But we have 
       {\small $F' = F'[F_p'[\bigvee_{I \in \mathcal{I}(k)}(\bigwedge_{j \in I} G_j \gtrdot F_c)]] 
       $}. Therefore we apply the same reduction on 
       {\small $F_p'$} to gain; \\
       {\small $F'[F_p'[\bigvee_{I \in \mathcal{I}(k)}(\bigwedge_{j \in I} G_j \gtrdot F_c) 
       ]] \leadsto F'_x[F_{p'}'[\bigvee_{I \in \mathcal{I}(k)}(\bigwedge_{j \in I} G_j \gtrdot F_c) 
       ]]$}. Induction hypothesis.     
   \hide{ 
   \item If obj associativity or obj commutativity 
       applies on a sub-formula of 
       the shown {\small $G_0 \vee \dots \vee G_{k+1}$}, 
       then arrange formulas in {\small $F'$} accordingly. 
       This change does not alter the fact that 
       {\small $F'$} is in the form:\linebreak
       {\small $F'[\bigwedge_{I \in \mathcal{I}(k+1)}(\bigvee_{j \in 
                   I} G_j \gtrdot F_c)]$}.  
       }
   \end{enumerate}  
 \item The second and the third: Straightforward. 
 \item The fourth pair:  
     \begin{enumerate} 
         \item If obj distribution 1 takes place 
             on {\small $F$} such that we have; 
             {\small $F[F_a \wedge (F_b \vee F_c) 
                 \gtrdot F_d] \leadsto 
                 F_x[(F_a \wedge F_b) \vee (F_a \wedge F_c)
                 \gtrdot F_d] 
                 $}, 
             then do nothing on {\small $F'$}; 
             and we have {\small $F_x = F'$}. 
             Vacuous thereafter.  
             \hide{ 
         \item If obj distribution 1 takes place 
             on {\small $F$} such that we have; 
             {\small $F[F_a \wedge (F_b \vee F_{\beta} \vee 
                 F_{\gamma}) \gtrdot F_d] 
                 \leadsto 
                 F_x[(F_a \wedge (F_b \vee F_{\beta})
                 \gtrdot F_d) \vee 
                 (F_a \wedge F_{\gamma} \gtrdot 
                 F_d)]$} for {\small $F_c = F_{\beta} \vee F_{\gamma}$}, then apply obj distribution 1 on 
             {\small $F'$} such that we have; 
             {\small $F'[(F_a \wedge F_b \gtrdot F_d) 
                 \vee (F_a \wedge (F_{\beta} \vee F_{\gamma}) \gtrdot F_d)]
                 \leadsto 
                 F'_x[(F_a \wedge F_b \gtrdot F_d) 
                 \vee (F_a \wedge F_{\beta} \gtrdot F_d) 
                 \vee (F_a \wedge F_{\gamma} \gtrdot F_d)]$}. 
             Now, we have 
             {\small $F_x = 
                 F_p[F_a \wedge (F_b \vee F_{\beta}) 
                 \gtrdot F_d]$}; 
             and {\small $F'_x = F'_p[(F_a \wedge F_b \gtrdot F_d)
                 \vee (F_a \wedge F_{\beta} \gtrdot F_d)]$} 
             such that {\small $F_p$} and {\small $F'_p$} 
             differ only in the shown sub-formulas. Apply 
             induction hypothesis on them.   
         }
         \item If obj distribution 2 takes place 
             on {\small $F$} such that we have; 
             {\small $F[(F_{\beta} \vee F_{\gamma}) 
                 \wedge (F_b \vee F_c) \gtrdot F_d] 
                 \leadsto 
                 F_x[(F_{\beta} \wedge (F_b \vee F_c)) 
                 \vee (F_{\gamma} \wedge (F_b \vee F_c)) 
                 \gtrdot 
                 F_d]$} for 
             {\small $F_a = F_{\beta} \vee F_{\gamma}$}, then 
             by induction hypothesis,  
             it does not cost generality 
             if we replace it 
             with {\small $F_p[((F_{\beta} \wedge F_b) \vee  
                 (F_{\beta} \wedge F_c)) \vee 
                 ((F_{\gamma} \wedge F_b) \vee 
                 (F_{\gamma} \wedge F_c)) \gtrdot F_d]$} 
            that differs from {\small $F_x$} 
            only by the shown sub-formulas.   
            Meanwhile, we derive; 
            {\small $F'[
                ((F_{\beta} \vee F_{\gamma}) \wedge 
                F_b) \vee ((F_{\beta} \vee F_{\gamma}) 
                \wedge F_c) \gtrdot F_d] \leadsto 
                F'_x[((F_{\beta} \wedge F_b) \vee 
                (F_{\gamma} \wedge F_b)) \vee 
                ((F_{\beta} \wedge F_c) \vee (F_{\gamma} 
                \wedge F_c)) \gtrdot F_d]$}. 
            Apply induction hypothesis on 
            {\small $F_x$} to arrive at {\small $F'_x$}. 
            Vacuous thereafter. 
\item The other cases: Straightforward. 
     \end{enumerate}
 \item The fifth pair: Similar.  
 \item The sixth pair: Straightforward, since 
     a unit graph chain cannot be further reduced, 
     since the reduction of {\small $\neg G$} 
     is unique, and since, by the definition 
     of {\small $\gtrdot$} reduction 3, 
     it cannot apply unless {\small $\neg G$}  
     has been fully reduced to {\small $\recurseReducet(G)$}.  
 \item The rest: Cf. the proof of Theorem 2. 
 \end{enumerate} 
 By the result of the above bisimulation, 
 we now have {\small $\mathcal{G}(F) = \mathcal{G}(F')$}. 
 However, 
 it takes only those 4 {\small $\neg$} reductions, 
 3 {\small $\gtrdot$} reductions, 
 obj distribution 1 and obj distribution 2 to derive 
 a formula in unit graph expansion; hence 
 we in fact have {\small $\mathcal{G}(F) = \mathcal{G}(F_x)$} 
 for some formula {\small $F_x$} in unit graph 
 expansion. But then by Theorem \ref{boolean_algebra_2}, there could be 
 only one value out of {\small $\{0,1\}$} assigned 
 to {\small $[\tintFrame \models F_x]$}, as required. 
\end{proof}  
\begin{corollary}[Normalisation]
    Given a formula {\small $F$} 
    with no occurrences of {\small $\top$} and {\small $\bot$},
    denote the set of formulas in unit graph expansion 
    that it can reduce into by {\small $\mathcal{G}_1$}. 
    Then it holds 
    for every valuation frame 
    either that {\small $[\tintFrame \models F_a] = 1$} 
    for all {\small $F_a \in \mathcal{G}_1$} or else 
    that 
    {\small $[\tintFrame \models F_a] = 0$} for all {\small $F_a \in 
    \mathcal{G}_1$}.
    \label{normalisation2}
  \end{corollary}  
By construction, this logic is also decidable, 
but for any unrestrained expressions 
the complexity will be high, owing 
to {\small $\gtrdot$} reduction 3. 
\section{Conclusion}             
This work analysed phenomena that arise around concepts and 
their attrbutes, and called 
attention the positioning of atomic entities
in formal logic, based on the notion of recognition cut-off. 
Both the philosophical and the mathematical foundations 
of gradual logic were laid down. 
For the object-attribute relation, there may be many 
linguistically reasonable interpretations. \\
\indent To conclude, I state connections of gradual logic to 
Aristotle's logic and others, along with prospects. 
\subsection{Connection to Aristotle's logic}      
Natural expressions require more than 
one types of negation. Given an expression X, if 
it is contradictorily negated into another expression Y, 
X is true iff Y is false. A contrary expression to X, however, 
only demands that it be false if X is true. A sub-contrary 
expression to X, on the other hand, demands that 
it be true if X is false. The distinction, which is 
as good as defunct in the post-Fregean modern logic, despite 
sporadic recurrences of the theme here and then \cite{Russell05,
    Wright59,McCall67}, 
has been nonetheless known - already since the era of Aristotle's.   
An extensive discussion on contrarieties 
is found in Categories for qualities and in 
Prior Analytics for categorical sentences. Meanwhile, only 
external, contradictory negations remain proper in the modern logic, 
contrarieties dismissed. 
Some such as Lukasiewicz
\cite{Lukasiewicz34}
contend that the modern logic on sentences, 
founded by the Stoics and axiomatized by Frege, is 
logically 
prior to Aristotelian term logic in that, according 
to his judgement, logic on propositions underlies the term logic. 
    Others
are not so convinced. Some such as Horn, Sommers and Englebretsen defend the term logic 
in the respective works of theirs \cite{Horn89,Sommers70,Englebretsen81}, countenancing that there are 
features that have been lost in the mass-scale migration from 
the term logic to the modern logic. With 
analysis in Section 1, I have in part concurred with proponents 
of the term logic. One that 
is of particular interest in 
the Aristotelian term logic is the use of 
indefinite\footnote{These terminologies
    are taken from 
    The Internet Classics Archive (classics.mit.edu/Aristotle).} nouns as result from 
prefixing `not' to a noun, {\it e.g.} from man 
to not-man; and of indefinite verbs, from 
walks to not-walks. For sentences, they can be 
either affirmative or negative. ``Man walks'', 
``Not-Man walks'', ``Man Not-walks'', and ``Not-Man Not-walks'' 
are affirmative; 
``Man does not Not-walk'', ``Not-Man does not walk'', 
``Man does not Not-walk'', and ``Not-Man does not 
Not-walk'' are negative corresponding to them.\footnote{
    Judging from Aristotle's texts (in translation), affirmative sentences 
    in the form: X is Y, are primary to Aristotle. 
    X and Y can be in the form Not-X' or Not-Y', for 
    they are in any case affirmed in the sentence. 
    However, it is not the case for Aristotle that 
    sentences of the form: X is not Y, denying 
    Y of X, bears a truth value primarily. This 
    is clear from an example (Cf. On Interpretation) about the 
    truth value of `Socrates is ill/Socrates is not ill'. 
    In case Socrates does not denote anything, 
    then the non-being (and non-being is not a being) cannot be ill in an ordinary sense, and so `Socrates 
    is ill' is false. Aristotle then judges that 
    `Socrates is not ill' is true. But this would 
    remain debatable if it were the case that negative 
    statements primarily bore a truth value. For, then, 
    it, by exactly the same reasoning, could be simply false in the 
    absence of Socrates. 
    Therefore, while in the Fregean logic  
    whatever sentences may be a proposition,  
    with no regard whether it is affirmative or negative, 
    the differentiation is important in Aristotelian term logic.}  
Hence to each singular sentence\footnote{A sentence 
    that does not specify ``all'' or ``some'' 
    are singular.},
say ``Man walks'', there are seven corresponding sentences 
that are either negative or having an indefinite term. 
There is no such denying particle 
on terms in Fregean modern logic; and 
diversification
of negation by scopal distinction \cite{Russell05} or by adopting  
more than one external negation operators \cite{Wright59} on 
sentences 
cannot make amends for the limitation that arises directly
from sentential atomicity. 
One crucial point 
about a proposition in the modern logic in fact is that, even if it is an atomic proposition, 
one never knows how complex it already is, and consequently 
how many contraries it ought to have (Cf. also Geach \cite{Geach72}). However, concerning this matter, 
the assumption of atomicity of entities in 
formal/symbolic logic may be more fundamental. 
Although 
to Aristotle, too, there existed indivisible entities, 
it is unlikely that such 
entities, if they are to exist, are cognizable, so long as the entities that we deal with 
are concepts that refer to objects; and, about those,  
we cannot reason. One may 
also apply Postulate 1 to propositions in general, whereby 
a proposition becomes an object, which 
will then be divisible. Then there will be 
propositions about the proposition as its attributes.  
By explaining the indivisible 
in terms of recognition cut-offs, one can also appreciate that 
any proposition, if treated 
as an object, 
will have attributive propositions
about it, and they, too, as internal structures of the proposition, 
can be structured in $\gtrdot$s. 
    Then not just one, not just two, but arbitrarily 
many internal contradictions 
can be brought to light, 
which are externally  
contraries, for those propositions that 
have been hitherto atomic. \\
\indent 
There is another relevant remark of Aristotle's 
found in Prior/Posterior Analytics. It is combinability of predicates. If bird 
is for example both beautiful and singing, then 
that could be expressed in the first-order logic as; 
{\small $\textsf{IsBeautiful}(bird) \wedge 
    \textsf{IsSinging}(bird)$}, where 
{\small $bird$} is a term. But in so expressing, 
it goes no further. When 
we attempt a proximity mapping of `Bird is beautiful, and 
it is singing' in gradual logic with 
`Bird is (judged under a domain of discourse); and it has the attribute of being 
beautiful, and that it has the capacity of singing (judged 
under another domain of discourse),' we gain; 
{\small $(\textsf{Bird} \gtrdot \textsf{Beautiful})
    \wedge (\textsf{Bird} \gtrdot \textsf{Singing})$}; 
or, equivalently 
{\small $\textsf{Bird} \gtrdot \textsf{Beautiful} \wedge 
    \textsf{Singing}$}, the two attributes 
conjoining into a unified attribute. 
Similar may also hold for the other side of {\small $\gtrdot$}. 
Exactly how combination occurs depends on a given linguistic interpretation 
on the object-attribute relation. Aristotle mentions 
of such combination in one part. Now,  why a similar process 
does not occur in the above-given first-order expression
is because, not only of the terms but also of 
all the predicates, form is pre-defined. Fregean terms 
are indivisible and fixed; and it must be 
known how many Fregean terms each Fregean predicate will take. \\
\indent To conclude this sub-section, 
we saw that  
first-order logic does not share the same logical foundation 
of the term logic. Aristotle's logic treats terms 
both as subjects and predicates, whereas 
Fregean terms are not Fregean predicates, nor vice versa. 
The object-attribute relation in gradual logic is closer 
in the respect to Aristotle's subject-predicate relation 
than the relation that holds between Fregean terms and Fregean predicates.
As stated in 5.1, however, in gradual logic {\small $X \gtrdot Y$} 
may itself act as an object, as in {\small $(X \gtrdot Y) \gtrdot Z$}, distinct from Aristotelian subject-predicate relation which 
does not produce a subject. \\
\indent As one research interest out of gradual logic, it should 
be fruitful to conduct  
cross-studies against Aristotle's logic, and to see how well 
Aristotle's syllogism can be explained within. 
Instead of 
embedding Aristotle's logic in first-order logic \cite{Edgar07} or 
first-order logic in Aristotle's logic \cite{Sommers70}, 
the strengths that the two have may be mutually extended. 
gradual logic 
may pave a way for realising the possibility. 
It must 
be pointed out, however, that, 
in order to  
attempt modelling the universal/particular sentences, the three figures and syllogism in Prior Analytics, 
it is necessary that we develop 
gradual predicate logic, as to be stated shortly. 

 \subsection{Gradual X Logic}    
 Meta-framework of some existing framework(s) 
 offers a way of deriving 
 new results without destroying the properties 
 of the original framework(s). As it retains 
 principles in the original 
 framework, it is also highly reusable. For example,  
 assuming 
 that all the (Fregean) terms and quantifications are 
 contained within a domain of discourse, replacement 
 of the underlying propositional logic in this work 
 with first-order logic, or, in general, another Boolean logic X, gives us 
 gradual X logic, and all the main results that we saw 
 go through, apart from the decidability result which 
 depends on the decidability of the underlying logic. 
 The reason that we can simply swap 
 the underlying in this manner is because the meta-framework 
 considered in this work acts only on the 0/1 (Cf. 
 the given semantics). How the 0/1 is generated is irrelevant 
 to the applicability of the meta-constructs 
 that $\gtrdot$s generate. \\
 \indent However, from a theoretical perspective about the 
 use of predicates within gradual logic, the use of intra-domain-of-discourse
 predicates/quantifications is conservative, since the terms 
 so introduced  
 will be atomic, which is not in harmony with the  
 philosophical standpoint that was taken in this work. Real 
 theoretical extensions 
 will be by introducing predicates that range over 
 attributed objects themselves. For instance, 
 suppose that we have two expressions 
 {\small $\textsf{Adjective} \gtrdot \textsf{Sheep}$} 
 and {\small $\textsf{Ovine}$}, then we may say 
 {\small $\textsf{IsEqualTo}(\textsf{Adjective} \gtrdot 
     \textsf{Sheep}, \textsf{Ovine})$} 
 in some (but not necessarily all) domains of discourse. 
 This type of extension may be called active predicate 
 extension. In this direction, there are both 
 philosophical/linguistic and mathematical challenges, 
 and it will be important to adequately capture interactions 
 between the active predicates and the reduction rules. 
\subsection{On tacit agreement} 
The incremental shift in domain of discourse 
models tacit agreement, which is 
otherwise understood as a context. 
Within formal logic considered in artificial intelligence, 
a line of studies since 
McCarthy \cite{McCarthy93,Buvac93,Ghidini01,Nayak94} 
have set a touchstone for logics handling
contextual reasonings. In those context logics,\footnote{There are 
other logics termed context logics which treat 
a context as an implication. But these, 
by explicitly stating what follow 
from what in the same domain of discourse, do not truly express the tacitness of a tacit agreement.}
all the propositions are judged under a context 
depending on which their truth values are determined. 
The question of what a context is, nonetheless, 
has not been pursued in 
the context logics 
any farther than that it is a rich object that is 
only partially explained. 
But because they treat 
a proposition as, in comparison, something 
that is known, there emerges a distinction 
between a context and a proposition whereby 
the former becomes a meta-term like a nominal in hybrid logic 
\cite{Brauner14}
that conditions the latter.  
As much as the consideration appears natural, 
it may be also useful to think what truly 
makes a context differ from a proposition, for, suppose
a proposition that Holmes LS is a detective in the context 
of Sherlock Holmes stories \cite{McCarthy93}, 
it appears on a reasonable ground that 
that the scenes (under which the proposition 
falls) are the stories of Sherlock Holmes is 
indeed a proposition. 
And if a proposition itself is a rich object that 
can be only partially explained,  
then the fundamental gap between the two domains 
closes in. 
\bibliographystyle{asl} 
\bibliography{references} 
\section*{Appendix A - a Java file, and test cases for Proposition 7}   
The Java code (version 1.6) that is used for the tests of Proposition 7 and 
Proposition 9 is listed below. 
{\tiny 
\begin{lstlisting}
import java.util.Stack;
/**
 * Written just for proving two results on "Logic on 
 * Recognition Cut-Off: Objects, Attributes and Atomicity". 
 * This program does not perform any value comparisons, which must be manually done. 
 * Intended for a personal use, there are hardly any exception handling. In any case, 
 * the source is in public to see where errors are thrown. 
 * An argument should be provided. If it is 1, then the program checks the results 
 * for Proposition 7(Preliminary observation); otherwise, it checks the results 
 * for Proposition 9. 
 * The outputs look like this: <p>
 * "======TEST n, X ======== first expression === second expression"
 * on the first line where the 'n' indicates which test procedure that is being 
 * called; X is either true or false corresponding to the Boolean parameter of 
 * f_size function on the paper; the 'first expression' and the 'second 
 * expression' a formula whose value is being calculated. <p>
 * On the second line appears the calculated result of the 'first expression'.<p>
 * And on the third line that of the 'second expression'.<p>
 * On the fourth line are found for the 'first expression' the inverse of 
 * ((the number of occurrences of !) + 1), that of % (+ 1 on the denumerator), 
 * and, in case of the second test, also that of * (+ 1 on the denumerator). <p> 
 * ON the fifth line, same but for the 'second expression'. <p>
 * This basic structure repeats as many as the number of the test cases. <p>
 * Briefly remarking on the syntax, an expression is written in prefix form. 
 * A small alphabet which must be of length 1 denotes a literal. A capital 
 * alphabet of the length 1 denotes a general formula, on the other hand. 
 * Grammar (let us denote an expression by EXP): <p> 
 * 1. A formula is EXP.<p>
 * 2. !EXP is EXP. ! means not. <p>
 * 3. *(EXP)(EXP) is EXP. * means and. <p>
 * 4. +(EXP)(EXP) is EXP. + means or. <p>
 * 5. %(EXP)(EXP) is EXP. % means the object-attribute relation. <p>
 * Also remarking on the results, ^n is exponent to immediately preceding number. 
 * n/m denotes n divided by m. + denotes addition. 
 * (l=n) after a capital alphabet indicates the value to the second argument of 
 * the function f_size.  
 * @author 
 *
 */
public class Calculator {
	
	//========DATA========
        private static enum FType{ 
		LITERAL,    //a literal. 
		NONLITERAL, //not a literal.
		GFORMULA    //Sub-formula still to be processed.
	
	}	
	private static String f_left;
	private static String f_right;
	
	private static Stack<Integer> stack;
	
	//these three are used to get the number of occurrences of !,% and *.
	private static int neg_counter; 
	private static int obat_counter; 
	private static int and_counter; 
	//====================
	/**
	 * If args = "1", this procedure checks the results for Proposition 7. At the 
	 * same time, it tells the inverse of the number of occurrences of negations 
	 * (!) and the inverse of the number of occurrences of .> (%). 
	 * Otherwise, it checks the results for Proposition 9. At the same time, it
	 * tells the inverse of the number of the occurrences of negations (!), the 
	 * inverse of the number of the occurrences of .> (%),
	 * and the inverse of the number of the occurrences of conjunctions (*). 
	 * @param args
	 */
	public static void main(String[] args){
		
	   if(args[0].equals("1"))
	   {
		test1(); test2();  test3(); test4();test5();test6();test7();
		test8();test9();test10();test11();test12();test13();
	   }
	   else
	   {
		  test1B();
		  test2B();test3B();test4B();test5B();test6B();test7B();test8B();
		  test9B();test10B();test11B();test12B();test13B();
	   }
	}
	/**
	 * Test neg reduction 1. 
	 */
	private static void test1(){ 
		f_left = new String("!s"); 
		f_right = new String("s");  
		printOut(f_left,f_right,1);
		
   	}
	/**
	 * Test neg reduction 2. 
	 */
	private static void test2(){
		f_left = new String("!*(A)(B)"); 
		f_right = new String("+(!A)(!B)"); 
		printOut(f_left, f_right,2);
	}
	/**
	 * Test neg reduction 3. 
	 */
	private static void test3(){
		f_left = new String("!+(A)(B)");
		f_right = new String("*(!A)(!B)");
		printOut(f_left, f_right,3);
	}
	/**
	 * Test neg reduction 4. 
	 */
	private static void test4(){
		f_left = new String("!%(s)(A)");
		f_right = new String("+(s)(%(s)(!A))");
		printOut(f_left, f_right, 4);
	}
	/**
	 * Test neg reduction 5. 
	 */
	private static void test5(){
		f_left = new String("%(%(A)(B))(C)");
		f_right = new String("*(%(A)(C))(+(%(A)(B))(%(A)(%(B)(C))))");
		//(f_left, f_right, 5);
		printOut(f_left,f_right,5);
		
	}
	/**
	 * Test .> reduction 1. 
	 */
	private static void test6(){
		f_left = new String("%(*(A)(B))(C)");
		f_right = new String("*(%(A)(C))(%(B)(C))");
		printOut(f_left,f_right,6);
	}
	/**
	 * Test .> reduction 2. 
	 */
	private static void test7(){
		f_left = new String("%(+(A)(B))(C)");
		f_right = new String("+(%(A)(C))(%(B)(C))");
		printOut(f_left,f_right,7);
		
	}
	/**
	 * Test .> reduction 3. 
	 */
	private static void test8(){
		
		f_left = new String("%(A)(*(B)(C))");
		f_right = new String("*(%(A)(B))(%(A)(C))");
		printOut(f_left,f_right,8);
	}
	/**
	 * Test .> reduction 4.
	 */
	private static void test9(){
		f_left = new String("%(A)(+(B)(C))");
		f_right = new String("+(%(A)(B))(%(A)(C))");
		printOut(f_left,f_right,9);
	}
	/**
	 * Test * commutativity. 
	 */
	private static void test10(){
		f_left = new String("*(A)(B)");
		f_right = new String("*(B)(A)");
		printOut(f_left,f_right,10);
	}
	/**
	 * Test + commutativity. 
	 */
	private static void test11(){
		f_left = new String("+(A)(B)");
		f_right = new String("+(B)(A)");
		printOut(f_left,f_right,11);
	}
	/**
	 * Test * associativity. 
	 */
	private static void test12(){
		f_left = new String("*(*(A)(B))(C)");
		f_right = new String("*(A)(*(B)(C))");
		printOut(f_left,f_right,12);
	}
	/**
	 * Test + associativity.
	 */
	private static void test13(){
		f_left = new String("+(+(A)(B))(C)");
		f_right = new String("+(A)(+(B)(C))");
		printOut(f_left,f_right,13);
	} 
	/**
	 * Tests for Proposition 9. ! reduction 1. 
	 */
	private static void test1B(){ 
		test1(); 		
	}
	/**
	 * For ! reduction 2. 
	 */
	private static void test2B(){
		test2();
	}
	/**
	 * For ! reduction 3. 
	 */
	private static void test3B(){
		test3();
	}
	/**
	 * For ! reduction 4, for k =2. Association does not matter 
	 * due to other cases in the same proposition. 
	 */
	private static void test4B(){
		f_left = new String("!%(*(A)(*(B)(C)))(D)");
		f_right = new String("+(+(!A)(+(!B)(!C)))(%(*(A)(*(B)(C)))(!D))");
		printOut(f_left, f_right, 400);
	}
	/**
	 * For .> reduction 3, for k = 2. 
	 */
	private static void test5B(){
		f_left = new String("%(+(A)(+(B)(C)))(D)");
		f_right = new String("+(%(A)(D))(+(%(B)(D))(+(%(C)(D))(+(%(*(A)(B))(D))"+
		"(+(%(*(A)(C))(D))(+(%(*(B)(C))(D))(%(*(A)(*(B)(C)))(D)))))))");
		printOut(f_left,f_right,500);
	}
	/**
	 * For .> reduction 4.
	 */
	private static void test6B(){
		test8();
	}
	/**
	 * For .> reduction 5. 
	 */
	private static void test7B(){
		test9();
	}
	/**
	 * For obj distribution  1. 
	 */
	private static void test8B(){
		f_left = new String("%(*(A)(+(B)(C)))(D)");
		f_right = new String("%(+(*(A)(B))(*(A)(C)))(D)");
		printOut(f_left,f_right,800);
	}
	/**
	 * For obj distribution 2. 
	 */
	private static void test9B(){
		f_left = new String("%(*(+(A)(B))(C))(D)");
		f_right = new String("%(+(*(A)(C))(*(B)(C)))(D)");
		printOut(f_left,f_right,900);
	}
	/**
	 * For * commutativity. 
	 */
	private static void test10B(){
		test10();
	}
	/**
	 * For + commutativity. 
	 */
	private static void test11B(){
		test11();
	}
	/**
	 * For * associativity. 
	 */
	private static void test12B(){
		test12();
	}
	/**
	 * For + associativity. 
	 */
	private static void test13B(){
		test13();
	}
	/**
	 * As on the paper, save in the prefix form. XX indicates that the string is
	 *  not formulated according to the grammar. 
         * @param neg_depth
	 * @param l
	 * @param bool 
	 * @param f_str
	 * @return
	 */
	private static String f_size(int neg_depth, int l, boolean bool, String f_str){
		 String prefix; 
		String suffix; 
		if(getFType(f_str) == FType.LITERAL)
			return "1/4^" + (new Integer(l).toString()); 
		else if(getFType(f_str) == FType.NONLITERAL)
			return f_str+ "(l=" + new Integer(l).toString()+ ")"; 
		else if(f_str.charAt(0) == '!') 
		{
			negCounterIncrement();
			return "1/4^" + (new Integer(neg_depth).toString()) +  " + (" + 
					f_size(neg_depth,l, bool, f_str.substring(1))+ ")"; 
		}
		else if(f_str.charAt(0) == '%')
		{
			obatCounterIncrement();
			int index = getIndexParenthesis(f_str);
			prefix = f_str.substring(2, index);
			suffix = f_str.substring(index+2,f_str.length()-1);	
			return "(" + 
				f_size(neg_depth+1,l, bool, prefix) + " + " 
					+ f_size(neg_depth+1,l,bool, suffix) + ")";
			
		}

		else if(f_str.charAt(0) == '*')
		{  
			andCounterIncrement();
			int index = getIndexParenthesis(f_str);
			prefix = new String(f_str.substring(2, index));
			suffix = new String(f_str.substring(index+2,f_str.length()-1));
			if(bool) return "max(" + f_size(neg_depth +1,l+1,false,prefix) + 
					" , " + f_size(neg_depth+1,l+1,false,suffix) + ")"; 
			else return "max(" + f_size(neg_depth,l,false,prefix) + 
					" , " + f_size(neg_depth,l,false,suffix) + ")"; 
			
		}		
		else if(f_str.charAt(0)=='+'){
			int index = getIndexParenthesis(f_str);
			prefix = f_str.substring(2, index);
			suffix = f_str.substring(index+2,f_str.length()-1);
			if(bool) return "max(" + f_size(neg_depth+1,l+1,false,prefix) + 
		            " , " + f_size(neg_depth+1,l+1,false,suffix) + ")"; 
			else return "max(" + f_size(neg_depth,l,false,prefix) + 
					" , " + f_size(neg_depth,l,false,suffix) + ")"; 
		}

		return "XX"; 
	}
	/**
	 * Given a string, it tells if it is a literal or a general formula, or 
	 * otherwise.
	 * @param in_str
	 * @return
	 */
	private static FType getFType(String in_str)
	{
		if(in_str.length() >= 2)
		{
			//System.out.println("Getting GFORMULA which is: " + in_str);
			return FType.GFORMULA;
		}
		else if(in_str.charAt(0) >= 'a' && 
				in_str.charAt(0) <= 'z')
			return FType.LITERAL;
		else
		{
		//	System.out.println(in_str + " Is NONLITERAL.");
			return FType.NONLITERAL;
		}
			
	}
	/**
	 * Parser. -100 is an error.
	 * @param in_str
	 * @return
	 */
	private static Integer getIndexParenthesis(String in_str)
	{
		//with stack. 
		stack = new Stack<Integer>();
		String curStr = new String(in_str); 
		char curChar; 
		for(int i =0; i <curStr.length(); i++)
		{
			curChar = curStr.charAt(i); 
		if (curChar == '(')
				stack.push(i);
			else if (curChar == ')')
			{
				stack.pop();
				if(stack.isEmpty())
					return i;
			}
			
		}
		
		return -100;
		
	}
	/**
	 * Console printing. 
	 * @param left_str
	 * @param right_str
	 * @param n
	 */
	private static void printOut(String left_str, String right_str, int n)
	{
		int lNegCounter,lObatCounter,lAndCounter;
		counterReset();
		System.out.println("=====TEST" + new Integer(n).toString() +
                    ", false=======" + left_str + "===" + right_str); 
		System.out.println(f_size(1,1, false,f_left)); 
		//store the counter values for the first expression.
		lNegCounter = neg_counter; lObatCounter = obat_counter; lAndCounter = 
				and_counter;
		//and reset the counters.
		counterReset();
		System.out.println(f_size(1,1,false,f_right));  
		//print the counter values. 
		System.out.println("1/" + lNegCounter + ",1/" + lObatCounter + ",1/"
		+ lAndCounter);
		System.out.println("1/" + neg_counter + ",1/" + obat_counter + ",1/" 
		+ and_counter);

		counterReset();
		System.out.println("=====TEST" +  new Integer(n).toString() + ", " +
				"true======="); 
		System.out.println(f_size(1,1,true,f_left)); 
		lNegCounter = neg_counter; lObatCounter = obat_counter; 
		lAndCounter = and_counter;
		counterReset();
		System.out.println(f_size(1,1,true,f_right));  
		System.out.println("1/" + lNegCounter + ",1/" + lObatCounter + ",1/" + 
		lAndCounter);
		System.out.println("1/" + neg_counter + ",1/" + obat_counter + ",1/" +
		and_counter);
	}
	/**
	 * Reset the counters to 1. 1 means basically 0, but as mentioned in the 
	 * class description, 1/0 is bad. So the minimum is 1. 
	 */
	private static void counterReset()
	{
		obat_counter = 1; neg_counter = 1; and_counter = 1;
	}
	/**
	 * Increases the counter counting the occurrences of %. 
	 */
	private static void obatCounterIncrement()
	{
		obat_counter++; 
	}
	/**
	 * Increases the counter counting the occurrences of !. 
	 */
	private static void negCounterIncrement()
	{
		neg_counter++; 
	}
	/**
	 * Increases the counter counting the occurrences of *.
	 */
	private static void andCounterIncrement()
	{
		and_counter++;
	}
}
\end{lstlisting}  
       }  
\noindent And the test cases for Proposition 7 below. 
Please refer 
to the class description of the Java source code for the 
format. Test 1 tests {\small $\neg$} reduction 1, 
Test 2 tests {\small $\neg$} reduction 2, 
and so on until Test 4. 
Test 5 tests {\small $\gtrdot$} reduction 1, 
Test 6 does {\small $\gtrdot$} reduction 2, 
and so on until Test 9. Test 10 and 11 test 
commutativity of {\small $\wedge$} and {\small $\vee$}. 
Test 12 and Test 13 associativity. 
{\tiny 
\begin{lstlisting} 
=====TEST1, false=======!s===s
1/4^1 + (1/4^1)
1/4^1
1/2,1/1,1/1
1/1,1/1,1/1
=====TEST1, true=======
1/4^1 + (1/4^1)
1/4^1
1/2,1/1,1/1
1/1,1/1,1/1
=====TEST2, false=======!*(A)(B)===+(!A)(!B)
1/4^1 + (max(A(l=1) , B(l=1)))
max(1/4^1 + (A(l=1)) , 1/4^1 + (B(l=1)))
1/2,1/1,1/2
1/3,1/1,1/1
=====TEST2, true=======
1/4^1 + (max(A(l=2) , B(l=2)))
max(1/4^2 + (A(l=2)) , 1/4^2 + (B(l=2)))
1/2,1/1,1/2
1/3,1/1,1/1
=====TEST3, false=======!+(A)(B)===*(!A)(!B)
1/4^1 + (max(A(l=1) , B(l=1)))
max(1/4^1 + (A(l=1)) , 1/4^1 + (B(l=1)))
1/2,1/1,1/1
1/3,1/1,1/2
=====TEST3, true=======
1/4^1 + (max(A(l=2) , B(l=2)))
max(1/4^2 + (A(l=2)) , 1/4^2 + (B(l=2)))
1/2,1/1,1/1
1/3,1/1,1/2
=====TEST4, false=======!%(s)(A)===+(s)(%(s)(!A))
1/4^1 + ((1/4^1 + A(l=1)))
max(1/4^1 , (1/4^1 + 1/4^2 + (A(l=1))))
1/2,1/2,1/1
1/2,1/2,1/1
=====TEST4, true=======
1/4^1 + ((1/4^1 + A(l=1)))
max(1/4^2 , (1/4^2 + 1/4^3 + (A(l=2))))
1/2,1/2,1/1
1/2,1/2,1/1
=====TEST5, false=======%(%(A)(B))(C)===*(%(A)(C))(+(%(A)(B))(%(A)(%(B)(C))))
((A(l=1) + B(l=1)) + C(l=1))
max((A(l=1) + C(l=1)) , max((A(l=1) + B(l=1)) , (A(l=1) + (B(l=1) + C(l=1)))))
1/1,1/3,1/1
1/1,1/5,1/2
=====TEST5, true=======
((A(l=1) + B(l=1)) + C(l=1))
max((A(l=2) + C(l=2)) , max((A(l=2) + B(l=2)) , (A(l=2) + (B(l=2) + C(l=2)))))
1/1,1/3,1/1
1/1,1/5,1/2
=====TEST6, false=======%(*(A)(B))(C)===*(%(A)(C))(%(B)(C))
(max(A(l=1) , B(l=1)) + C(l=1))
max((A(l=1) + C(l=1)) , (B(l=1) + C(l=1)))
1/1,1/2,1/2
1/1,1/3,1/2
=====TEST6, true=======
(max(A(l=2) , B(l=2)) + C(l=1))
max((A(l=2) + C(l=2)) , (B(l=2) + C(l=2)))
1/1,1/2,1/2
1/1,1/3,1/2
=====TEST7, false=======%(+(A)(B))(C)===+(%(A)(C))(%(B)(C))
(max(A(l=1) , B(l=1)) + C(l=1))
max((A(l=1) + C(l=1)) , (B(l=1) + C(l=1)))
1/1,1/2,1/1
1/1,1/3,1/1
=====TEST7, true=======
(max(A(l=2) , B(l=2)) + C(l=1))
max((A(l=2) + C(l=2)) , (B(l=2) + C(l=2)))
1/1,1/2,1/1
1/1,1/3,1/1
=====TEST8, false=======%(A)(*(B)(C))===*(%(A)(B))(%(A)(C))
(A(l=1) + max(B(l=1) , C(l=1)))
max((A(l=1) + B(l=1)) , (A(l=1) + C(l=1)))
1/1,1/2,1/2
1/1,1/3,1/2
=====TEST8, true=======
(A(l=1) + max(B(l=2) , C(l=2)))
max((A(l=2) + B(l=2)) , (A(l=2) + C(l=2)))
1/1,1/2,1/2
1/1,1/3,1/2
=====TEST9, false=======%(A)(+(B)(C))===+(%(A)(B))(%(A)(C))
(A(l=1) + max(B(l=1) , C(l=1)))
max((A(l=1) + B(l=1)) , (A(l=1) + C(l=1)))
1/1,1/2,1/1
1/1,1/3,1/1
=====TEST9, true=======
(A(l=1) + max(B(l=2) , C(l=2)))
max((A(l=2) + B(l=2)) , (A(l=2) + C(l=2)))
1/1,1/2,1/1
1/1,1/3,1/1
=====TEST10, false=======*(A)(B)===*(B)(A)
max(A(l=1) , B(l=1))
max(B(l=1) , A(l=1))
1/1,1/1,1/2
1/1,1/1,1/2
=====TEST10, true=======
max(A(l=2) , B(l=2))
max(B(l=2) , A(l=2))
1/1,1/1,1/2
1/1,1/1,1/2
=====TEST11, false=======+(A)(B)===+(B)(A)
max(A(l=1) , B(l=1))
max(B(l=1) , A(l=1))
1/1,1/1,1/1
1/1,1/1,1/1
=====TEST11, true=======
max(A(l=2) , B(l=2))
max(B(l=2) , A(l=2))
1/1,1/1,1/1
1/1,1/1,1/1
=====TEST12, false=======*(*(A)(B))(C)===*(A)(*(B)(C))
max(max(A(l=1) , B(l=1)) , C(l=1))
max(A(l=1) , max(B(l=1) , C(l=1)))
1/1,1/1,1/3
1/1,1/1,1/3
=====TEST12, true=======
max(max(A(l=2) , B(l=2)) , C(l=2))
max(A(l=2) , max(B(l=2) , C(l=2)))
1/1,1/1,1/3
1/1,1/1,1/3
=====TEST13, false=======+(+(A)(B))(C)===+(A)(+(B)(C))
max(max(A(l=1) , B(l=1)) , C(l=1))
max(A(l=1) , max(B(l=1) , C(l=1)))
1/1,1/1,1/1
1/1,1/1,1/1
=====TEST13, true=======
max(max(A(l=2) , B(l=2)) , C(l=2))
max(A(l=2) , max(B(l=2) , C(l=2)))
1/1,1/1,1/1
1/1,1/1,1/1
\end{lstlisting} 
       }
\section*{Appendix B - test cases for Proposition 9}    
Test cases for Proposition 9, and associativity and 
commutativity cases of {\small $\wedge$} (* in the code) and 
{\small $\vee$} (+ in the code). Some of the lines are 
very long, and are split in two lines, which is indicated 
by SP. Test 1 tests {\small $\neg$} reduction 1, 
Test 2 {\small $\neg$} reduction 2, and so on 
until Test 400. Test 500 tests 
{\small $\gtrdot$} reduction 3, Test 8 
{\small $\gtrdot$} reduction 4, and Test 9
{\small $\gtrdot$} reduction 5.  
Test 800 tests obj distribution 1, Test 900 obj distribution 2. 
Test 10 - 13 test associativity and commutativity of 
{\small $\wedge$} and {\small $\vee$}, which are the same 
as for Proposition 7. 
{\tiny 
\begin{lstlisting}
=====TEST1, false=======!s===s
1/4^1+(1/4^1)
1/4^1
1/2,1/1,1/1
1/1,1/1,1/1
=====TEST1, true=======
1/4^1+(1/4^1)
1/4^1
1/2,1/1,1/1
1/1,1/1,1/1
=====TEST2, false=======!*(A)(B)===+(!A)(!B)
1/4^1+(max(A(l=1),B(l=1)))
max(1/4^1+(A(l=1)),1/4^1+(B(l=1)))
1/2,1/1,1/2
1/3,1/1,1/1
=====TEST2, true=======
1/4^1+(max(A(l=2),B(l=2)))
max(1/4^2+(A(l=2)),1/4^2+(B(l=2)))
1/2,1/1,1/2
1/3,1/1,1/1
=====TEST3, false=======!+(A)(B)===*(!A)(!B)
1/4^1+(max(A(l=1),B(l=1)))
max(1/4^1+(A(l=1)),1/4^1+(B(l=1)))
1/2,1/1,1/1
1/3,1/1,1/2
=====TEST3, true=======
1/4^1+(max(A(l=2),B(l=2)))
max(1/4^2+(A(l=2)),1/4^2+(B(l=2)))
1/2,1/1,1/1
1/3,1/1,1/2
=====TEST400, false=======SP
!%(*(A)(*(B)(C)))(D)===+(+(!A)(+(!B)(!C)))(%(*(A)(*(B)(C)))(!D))
1/4^1+((max(A(l=1),max(B(l=1),C(l=1)))+D(l=1)))
max(max(1/4^1+(A(l=1)),max(1/4^1+(B(l=1)),1/4^1+(C(l=1)))),SP
(max(A(l=1),max(B(l=1),C(l=1)))+1/4^2+(D(l=1))))
1/2,1/2,1/3
1/5,1/2,1/3
=====TEST400, true=======
1/4^1+((max(A(l=2),max(B(l=2),C(l=2)))+D(l=1)))
max(max(1/4^2+(A(l=2)),max(1/4^2+(B(l=2)),1/4^2+(C(l=2)))),(max(A(l=2),SP
max(B(l=2),C(l=2)))+1/4^3+(D(l=2))))
1/2,1/2,1/3
1/5,1/2,1/3
=====TEST500, false=======%(+(A)(+(B)(C)))(D)===+(%(A)(D))(+(%(B)(D))(+(%(C)(D))SP
(+(%(*(A)(B))(D))(+(%(*(A)(C))(D))(+(%(*(B)(C))(D))(%(*(A)(*(B)(C)))(D)))))))
(max(A(l=1),max(B(l=1),C(l=1)))+D(l=1))
max((A(l=1)+D(l=1)),max((B(l=1)+D(l=1)),max((C(l=1)+D(l=1)),SP
max((max(A(l=1),B(l=1))+D(l=1)),max((max(A(l=1),C(l=1))+D(l=1)),SP
max((max(B(l=1),C(l=1))+D(l=1)),(max(A(l=1),max(B(l=1),C(l=1)))+D(l=1))))))))
1/1,1/2,1/1
1/1,1/8,1/6
=====TEST500, true=======
(max(A(l=2),max(B(l=2),C(l=2)))+D(l=1))
max((A(l=2)+D(l=2)),max((B(l=2)+D(l=2)),max((C(l=2)+D(l=2)),SP
max((max(A(l=2),B(l=2))+D(l=2)),max((max(A(l=2),C(l=2))+D(l=2)),SP
max((max(B(l=2),C(l=2))+D(l=2)),(max(A(l=2),max(B(l=2),C(l=2)))+D(l=2))))))))
1/1,1/2,1/1
1/1,1/8,1/6
=====TEST8, false=======%(A)(*(B)(C))===*(%(A)(B))(%(A)(C))
(A(l=1)+max(B(l=1),C(l=1)))
max((A(l=1)+B(l=1)),(A(l=1)+C(l=1)))
1/1,1/2,1/2
1/1,1/3,1/2
=====TEST8, true=======
(A(l=1)+max(B(l=2),C(l=2)))
max((A(l=2)+B(l=2)),(A(l=2)+C(l=2)))
1/1,1/2,1/2
1/1,1/3,1/2
=====TEST9, false=======%(A)(+(B)(C))===+(%(A)(B))(%(A)(C))
(A(l=1)+max(B(l=1),C(l=1)))
max((A(l=1)+B(l=1)),(A(l=1)+C(l=1)))
1/1,1/2,1/1
1/1,1/3,1/1
=====TEST9, true=======
(A(l=1)+max(B(l=2),C(l=2)))
max((A(l=2)+B(l=2)),(A(l=2)+C(l=2)))
1/1,1/2,1/1
1/1,1/3,1/1
=====TEST800, false=======%(*(A)(+(B)(C)))(D)===%(+(*(A)(B))(*(A)(C)))(D)
(max(A(l=1),max(B(l=1),C(l=1)))+D(l=1))
(max(max(A(l=1),B(l=1)),max(A(l=1),C(l=1)))+D(l=1))
1/1,1/2,1/2
1/1,1/2,1/3
=====TEST800, true=======
(max(A(l=2),max(B(l=2),C(l=2)))+D(l=1))
(max(max(A(l=2),B(l=2)),max(A(l=2),C(l=2)))+D(l=1))
1/1,1/2,1/2
1/1,1/2,1/3
=====TEST900, false=======%(*(+(A)(B))(C))(D)===%(+(*(A)(C))(*(B)(C)))(D)
(max(max(A(l=1),B(l=1)),C(l=1))+D(l=1))
(max(max(A(l=1),C(l=1)),max(B(l=1),C(l=1)))+D(l=1))
1/1,1/2,1/2
1/1,1/2,1/3
=====TEST900, true=======
(max(max(A(l=2),B(l=2)),C(l=2))+D(l=1))
(max(max(A(l=2),C(l=2)),max(B(l=2),C(l=2)))+D(l=1))
1/1,1/2,1/2
1/1,1/2,1/3
=====TEST10, false=======*(A)(B)===*(B)(A)
max(A(l=1),B(l=1))
max(B(l=1),A(l=1))
1/1,1/1,1/2
1/1,1/1,1/2
=====TEST10, true=======
max(A(l=2),B(l=2))
max(B(l=2),A(l=2))
1/1,1/1,1/2
1/1,1/1,1/2
=====TEST11, false=======+(A)(B)===+(B)(A)
max(A(l=1),B(l=1))
max(B(l=1),A(l=1))
1/1,1/1,1/1
1/1,1/1,1/1
=====TEST11, true=======
max(A(l=2),B(l=2))
max(B(l=2),A(l=2))
1/1,1/1,1/1
1/1,1/1,1/1
=====TEST12, false=======*(*(A)(B))(C)===*(A)(*(B)(C))
max(max(A(l=1),B(l=1)),C(l=1))
max(A(l=1),max(B(l=1),C(l=1)))
1/1,1/1,1/3
1/1,1/1,1/3
=====TEST12, true=======
max(max(A(l=2),B(l=2)),C(l=2))
max(A(l=2),max(B(l=2),C(l=2)))
1/1,1/1,1/3
1/1,1/1,1/3
=====TEST13, false=======+(+(A)(B))(C)===+(A)(+(B)(C))
max(max(A(l=1),B(l=1)),C(l=1))
max(A(l=1),max(B(l=1),C(l=1)))
1/1,1/1,1/1
1/1,1/1,1/1
=====TEST13, true=======
max(max(A(l=2),B(l=2)),C(l=2))
max(A(l=2),max(B(l=2),C(l=2)))
1/1,1/1,1/1
1/1,1/1,1/1
\end{lstlisting}
}       
\end{document}

%% file: first_figure.tex
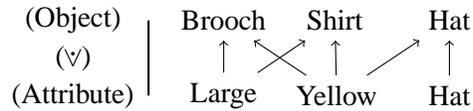
\begin{figure}[!h]
  \begin{center}  
    \begin{tikzpicture}     
      \draw[thick] (-2, 0) -- (-2, -1);   
      \node (a) at (-1, 0) {$\text{Brooch}$}; 
      \node (b) at (0.48, 0) {$\text{Shirt}$}; 
      \node (c) at (2, 0) {$\text{Hat}$};  
      \node (d) at (-1, -1) {$\text{Large}$}; 
      \node (e) at (0.5, -1) {$\text{Yellow}$};  
      \node (f) at (2, -1) {$\text{Hat}$};  
      \draw [->] (d) -- (a); 
      \draw [->] (d) -- (b); 
      \draw [->] (e) -- (a); 
      \draw [->] (e) -- (b); 
      \draw [->] (e) -- (c); 
      \draw [->] (f) -- (c); 
      \node (g) at (-3, 0) {$\text{(Object)}$};   
      \node (h) at (-3, -1) {$\text{(Attribute)}$};  
      \node (i) at (-3, -0.5) {(\reflectbox{\rotatebox[origin=c]{-90}{$\gtrdot$}})};  
    \end{tikzpicture} 
  \end{center}
  \caption{Illustration of an expression 
  {\small $((\text{Brooch} \wedge 
  \text{Shirt}) \gtrdot  
  \text{Large}) \wedge 
  ((\text{Brooch} \wedge 
  \text{Shirt} \wedge \text{Hat}) \gtrdot 
  \text{Yellow}) \wedge 
  (\text{Hat} \gtrdot \text{Hat})$}: the existential 
  fact of the attribute large 
  depends on the existential facts 
  of brooch and shirt; the existential 
  fact of the attribute of being yellow 
  depends on the existential facts 
  of brooch, shirt and hat; and 
  the existential fact of the attribute 
  hat depends on the existential fact of 
  hat to which it is an attribute.}
  \label{first_figure}
\end{figure}